\documentclass{article} \usepackage{iclr2027_conference,times}

\usepackage[utf8]{inputenc}

\usepackage[english]{babel}

\let\cite\citep

\usepackage{graphicx}
 \usepackage{subcaption}

\usepackage{xcolor}

\usepackage[colorinlistoftodos,prependcaption,textsize=small]{todonotes} \usepackage{comment}

\usepackage{soul}
\usepackage[normalem]{ulem}

\usepackage{amsmath} \usepackage{amssymb} \usepackage{mathtools} \usepackage{amsfonts}
\usepackage{amsthm} \usepackage{bm} \usepackage{bbm}

\usepackage{enumitem} \setlist[itemize]{leftmargin=2em}

\usepackage{booktabs} \usepackage{multirow} \usepackage{makecell} \usepackage{colortbl} \usepackage{xltabular} \usepackage{rotating}

\usepackage{csquotes}

\usepackage{hyperref}
\usepackage[capitalize,noabbrev]{cleveref}
\usepackage{url}

\usepackage{tikz}
\usepackage{tikz-cd} \usetikzlibrary{positioning}

\usepackage{xspace}

\usepackage{algorithm} \usepackage{algpseudocode}

\definecolor{cadmiumgreen}{rgb}{0.0, 0.40, 0.25}
\definecolor{lightcadmiumgreen}{rgb}{0.0, 0.60, 0.30}
\definecolor{cadmiumorange}{rgb}{0.93, 0.53, 0.18}
\definecolor{burgundy}{rgb}{0.50, 0.0, 0.13}
\definecolor{airforceblue}{rgb}{0.36, 0.54, 0.66}
\definecolor{niceblue}{rgb}{0.13, 0.33, 0.83}
\definecolor{nicegreen}{rgb}{0.13, 0.55, 0.17}
\definecolor{darkred}{rgb}{0.70, 0.14, 0.0}

\definecolor{lightgrey}{rgb}{0.79, 0.79, 0.79}
\definecolor{lightred}{rgb}{0.85, 0.70, 0.70}
\definecolor{lightgreen}{rgb}{0.70, 0.85, 0.70}
\definecolor{lightblue}{rgb}{0.70, 0.70, 0.90}
\definecolor{lightpurple}{rgb}{0.90, 0.70, 0.90}
\colorlet{grey}{gray}

    \newcommand{\del}[1]{}
    
    \newcommand{\itodo}[1]{}
    \newcommand{\todox}[1]{}
    \newcommand{\notex}[1]{}
    \newcommand{\scratchx}[1]{}
    \excludecomment{draft}

    \newcommand{\EA}[1]{{#1}\xspace}
    \newcommand{\EAR}[1]{\ignorespaces}
    \newcommand{\EAX}[1]{{#1}\xspace}

    \newcommand{\FLR}[1]{\ignorespaces}

\makeatletter
\newcommand{\customlabelnref}[2]{\protected@write \@auxout {}{\string \newlabel {#1}{{#2}{\thepage}{#2}{#1}{}}}\hypertarget{#1}{#2}}
\newcommand{\customlabel}[2]{\protected@write \@auxout {}{\string \newlabel {#1}{{#2}{\thepage}{#2}{#1}{}}}}
\makeatother

\makeatletter
\newcommand{\spaceornewline}[1]{\if@twocolumn
$$
$$
\else
#1
\fi
}\makeatother

\def\eqref#1{Equation~\ref{#1}}

\def\1{\mathbbm{1}}

\def\eps{{\epsilon}}

\def\vzero{{\bm{0}}}
\def\vone{{\bm{1}}}

\def\vtheta{{\bm{\theta}}}
\def\vell{{\bm{\ell}}}
\def\va{{\bm{a}}}
\def\vb{{\bm{b}}}
\def\vc{{\bm{c}}}

\def\ve{{\bm{e}}}

\def\vr{{\bm{r}}}

\def\vu{{\bm{u}}}
\def\vv{{\bm{v}}}
\def\vw{{\bm{w}}}

\def\vy{{\bm{y}}}
\def\vz{{\bm{z}}}

\def\mE{{\bm{E}}}

\def\mY{{\bm{Y}}}

\DeclareMathAlphabet{\mathsfit}{\encodingdefault}{\sfdefault}{m}{sl}
\SetMathAlphabet{\mathsfit}{bold}{\encodingdefault}{\sfdefault}{bx}{n}

\def\gB{{\mathcal{B}}}

\def\gF{{\mathcal{F}}}

\def\gR{{\mathcal{R}}}

\def\gT{{\mathcal{T}}}

\newcommand{\E}{\mathbb{E}}
\newcommand{\Ls}{\mathcal{L}}
\newcommand{\Lstar}{\Ls^{\star}} \newcommand{\dt}{{d}} \newcommand{\nv}[1]{{}^{\natural}\mkern-0.5mu #1} \newcommand{\R}{\mathbb{R}}

\DeclareMathOperator*{\argmax}{arg\,max}
\DeclareMathOperator*{\argmin}{arg\,min}

\DeclareMathOperator{\sign}{sign}

\usepackage{aliascnt}
\theoremstyle{plain}

\crefname{theorem}{Theorem}{Theorems}
\Crefname{theorem}{Theorem}{Theorems}
\newaliascnt{proposition}{theorem}
\newtheorem{proposition}[proposition]{Proposition}
\aliascntresetthe{proposition}
\crefname{proposition}{Proposition}{Propositions}
\Crefname{proposition}{Proposition}{Propositions}
\newaliascnt{lemma}{theorem}

\aliascntresetthe{lemma}
\crefname{lemma}{Lemma}{Lemmas}
\Crefname{lemma}{Lemma}{Lemmas}
\newaliascnt{corollary}{theorem}
\newtheorem{corollary}[corollary]{Corollary}
\aliascntresetthe{corollary}
\crefname{corollary}{Corollary}{Corollaries}
\Crefname{corollary}{Corollary}{Corollaries}
\theoremstyle{definition}
\newaliascnt{definition}{theorem}
\newtheorem{definition}[definition]{Definition}
\aliascntresetthe{definition}
\crefname{definition}{Definition}{Definitions}
\Crefname{definition}{Definition}{Definitions}
\newaliascnt{assumption}{theorem}

\aliascntresetthe{assumption}
\crefname{assumption}{Assumption}{Assumptions}
\Crefname{assumption}{Assumption}{Assumptions}
\theoremstyle{remark}
\newaliascnt{remark}{theorem}
\newtheorem{remark}[remark]{Remark}
\aliascntresetthe{remark}
\crefname{remark}{Remark}{Remarks}
\Crefname{remark}{Remark}{Remarks}

\newcommand{\egR}{\widehat{\gR}_{\eps}}

\usepackage{thmtools}

\title{Temporal Predictive Multiplicity: \\ Equally Accurate Time Series Models Yield Different Forecast Trajectories}
\author{\makebox[\dimexpr0.5\textwidth-2\tabcolsep\relax][l]{Emanuele Albini \thanks{Corresponding Author}}\\
J.P. Morgan AI Research\\
Imperial College London\\ 
London, UK \\
\texttt{emanuele.albini@jpmorgan.com} \\
\And
\makebox[\dimexpr0.5\textwidth-2\tabcolsep\relax][l]{Francesca Toni}\\
Imperial College London\\
London, UK \\
\texttt{ft@imperial.ac.uk} \\
\AND
\makebox[\dimexpr0.5\textwidth-2\tabcolsep\relax][l]{Saumitra Mishra}\\
J.P. Morgan AI Research \\
London, UK \\
\texttt{saumitra.mishra@jpmorgan.com} \\
\And
\makebox[\dimexpr0.5\textwidth-2\tabcolsep\relax][l]{Francesco Leofante}\\
Imperial College London \\
London, UK \\
\texttt{f.leofante@imperial.ac.uk} \\
}

\iclrfinalcopy
\begin{document}

\maketitle
\lhead{Technical Report}

\begin{abstract}
Models with near-identical predictive performance can yield substantially different predictions, a phenomenon known as \emph{predictive multiplicity}. Prior work has mostly studied this at the level of individual scalar outputs. In time-series forecasting, however, predictions across horizons jointly define a trajectory, and horizon-wise comparisons can hide important differences in predictive behavior.

To address this problem, we introduce \emph{temporal predictive multiplicity}, a framework that characterizes disagreement over complete forecast trajectories among models with near-identical predictive performance. We show that constraining predictive performance alone can still admit a broad range of different trajectories. We further show that constraining multiplicity at individual horizons partially reduces, but does not eliminate, trajectory-level multiplicity.

Experiments with 19 neural forecasting architectures on 11 datasets confirm that near-optimal models can exhibit substantial variability in the forecast trajectories they produce, and trajectory-level disagreement is largely unrelated to horizon-wise disagreement.
Our framework, therefore, exposes a gap in existing multiplicity studies: models with indistinguishable predictive performance imply fundamentally different temporal trajectories, with consequential downstream effects.

\end{abstract}

\section{Introduction}

In machine learning, multiple models can often achieve nearly indistinguishable predictive performance while producing different predictions. This phenomenon, known as the \emph{Rashomon Effect} \cite{Breiman2001}, gives rise to \emph{predictive multiplicity}: the fact that models with comparable performance can make different predictions for the same instance. A growing literature has studied the existence and consequences of competing models within the \emph{Rashomon set}, the collection of models whose predictive performance is close to optimal \citep{Fisher2019,Semenova2022}, showing that model choice can affect individual predictions \cite{Marx2020a}, feature attributions \cite{Dong2020}, and downstream decisions \cite{Du2025}, even when aggregate performance remains essentially unchanged.

Time-series forecasting is a particularly natural setting in which to study this phenomenon. Despite the proliferation of deep forecasting architectures \cite{Lim2021,Benidis2022}, no model class consistently dominates: simple linear models remain competitive with more complex neural architectures \cite{Zeng2023,Toner2024}, while model rankings shift considerably with hyper-parameter choices and evaluation protocols \cite{Brigato2026}. Forecasting practitioners therefore routinely train a heterogeneous pool of models with near-identical aggregate performance, yet must select a single model for deployment. In the absence of a principled framework that characterizes how these models differ temporally, this choice is effectively arbitrary.

\begin{figure}[t]
    \centering
    \newcommand{\trajaxis}{\draw[->,gray] (0,-2.9) -- (0,3.3);
        \node[left,font=\scriptsize] at (0,2.6) {$\hat{y}$};
        \draw[->,gray] (0,-2.9) -- (5.15,-2.9);
        \node[right,font=\scriptsize] at (5.15,-2.9) {$h$};
        \foreach \h/\tone in {1/8,2/15,3/22,4/29}{
            \fill[gray!\tone] (1.25*\h-0.894,-2.6) rectangle (1.25*\h-0.306,2.6);
            \draw[gray] (1.25*\h-0.6,-2.9) -- (1.25*\h-0.6,-3.2);
            \node[below,font=\scriptsize] at (1.25*\h-0.6,-3.2) {$\h$};
        }}
    \newcommand{\trajpair}[3]{\fill[niceblue] (#1,#2) circle (1.35pt);
        \fill[darkred] (#1,#3) circle (1.35pt);}
    \newcommand{\trajorigin}[5]{\draw[niceblue,line width=0.8pt,#1] (0.65,#2) -- (1.9,#3) -- (3.15,#4) -- (4.4,#5);
        \draw[darkred,line width=0.8pt,#1] (0.65,-0.7*#2) -- (1.9,-0.7*#3) -- (3.15,-0.7*#4) -- (4.4,-0.7*#5);
        \trajpair{0.65}{#2}{-0.7*#2}
        \trajpair{1.9}{#3}{-0.7*#3}
        \trajpair{3.15}{#4}{-0.7*#4}
        \trajpair{4.4}{#5}{-0.7*#5}}
    \begin{tikzpicture}[x=0.735cm,y=0.26cm]
        \useasboundingbox (-0.4,-4.2) rectangle (5.6,3.8);
        \trajaxis
        \trajorigin{solid}{2}{2}{2}{2}
        \trajorigin{dashed}{-2}{-2}{-2}{-2}
        \trajorigin{dotted}{0.4}{0.4}{0.4}{0.4}
    \end{tikzpicture}
    \hfill
    \begin{tikzpicture}[x=0.735cm,y=0.26cm]
        \useasboundingbox (-0.4,-4.2) rectangle (5.6,3.8);
        \trajaxis
        \trajorigin{solid}{2}{2}{2}{2}
        \trajorigin{dashed}{-2}{0.4}{-2}{0.4}
        \trajorigin{dotted}{0.4}{-2}{0.4}{-2}
    \end{tikzpicture}
    \hfill
    \begin{tikzpicture}[x=0.735cm,y=0.26cm]
        \useasboundingbox (-0.4,-4.2) rectangle (5.6,3.8);
        \trajaxis
        \trajorigin{solid}{2}{-2}{2}{-2}
        \trajorigin{dashed}{-2}{2}{-2}{2}
        \trajorigin{dotted}{0.4}{0.4}{0.4}{0.4}
    \end{tikzpicture}
    \\[2pt]
    \begin{tikzpicture}[baseline=-0.5ex,font=\scriptsize]
        \draw[niceblue,line width=0.8pt] (0,0) -- (0.6,0); \fill[niceblue] (0.3,0) circle (1.35pt); \node[right,inner sep=2pt] at (0.6,0) {Model 1};
        \draw[darkred,line width=0.8pt] (1.9,0) -- (2.5,0); \fill[darkred] (2.2,0) circle (1.35pt); \node[right,inner sep=2pt] at (2.5,0) {Model 2};
        \draw[gray!50,line width=0.4pt] (4.1,-0.18) -- (4.1,0.18);
        \draw[black!70,line width=0.8pt,solid] (4.4,0) -- (5.0,0); \node[right,inner sep=2pt] at (5.0,0) {Origin 1};
        \draw[black!70,line width=0.8pt,dashed] (6.3,0) -- (6.9,0); \node[right,inner sep=2pt] at (6.9,0) {Origin 2};
        \draw[black!70,line width=0.8pt,dotted] (8.2,0) -- (8.8,0); \node[right,inner sep=2pt] at (8.8,0) {Origin 3};
    \end{tikzpicture}
    \caption{\protect\EA{Examples of forecast trajectories that are indistinguishable under current multiplicity frameworks. At every horizon, the disagreement among forecasts is identical across the three panels. Yet, the panels exhibit markedly different patterns of temporal disagreement at the trajectory level.}}
    \label{fig:trajectories}
\end{figure}
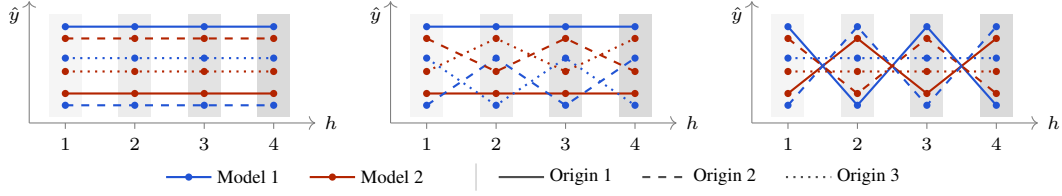

Current multiplicity frameworks offer little guidance for selecting among competing forecasting models. 
Existing work has largely focused on \emph{individual predictive multiplicity}, in which models produce a single \EAX{scalar} prediction for each instance, and multiple \EAX{outputs} are treated as independent.
Forecasts, however, are \emph{temporal objects} that require a different treatment: they are sequences of predictions over multiple time points whose joint behavior defines a forecast trajectory, as illustrated in \cref{fig:trajectories}. \EAX{Ignoring this structure would render markedly different trajectories indistinguishable.}

Motivated by these observations, we introduce a new theoretical framework for predictive multiplicity in time-series forecasting. At its core, our framework organizes multiplicity into a three-level hierarchy: \emph{pointwise}, \emph{horizon-wise}, and \EA{\emph{temporal}}, making explicit which aspects of forecast structure each level can capture. We then use \EA{this} hierarchy to reveal two fundamental limitations of conventional notions of multiplicity. First, near-optimal predictive performance provides weak control over forecast trajectories: constraining aggregate performance alone allows models to differ substantially over an entire trajectory, with this freedom increasing at longer horizons. Second, even when disagreement is tightly controlled at individual horizons, \EA{at pointwise and horizon-wise levels,} substantial temporal differences can remain, including the direction of predicted changes, the timing of events, and the shape and frequency structure of forecast trajectories.

We complement these theoretical results with experiments on 19 neural forecasting architectures across 11 public datasets. We find that near-optimal models frequently produce substantially different forecast trajectories, including predictions of changes in opposite directions. Moreover, trajectory-level disagreement is largely unrelated to disagreement measured at individual horizons, confirming that temporal multiplicity cannot be recovered from conventional pointwise or horizon-wise diagnostics. Finally, through a case study, we show that these differences can translate into materially different downstream decisions. Taken together, these results show that model choice is consequential; our framework makes the underlying temporal differences explicit and measurable, giving practitioners a principled basis for auditing near-optimal candidates along application-relevant dimensions and for justifying the model they ultimately deploy.

\section{Related Work}
\label{sec:related}

\textbf{Model Multiplicity in Machine Learning.}
That many models may explain the same data approximately equally well was articulated by Breiman as the \emph{Rashomon Effect} \cite{Breiman2001}, and has resurfaced in modern machine learning through studies of underspecification and training instability \cite{Madhyastha2019,DAmour2022,Teney2022}.
Models within a Rashomon set have been shown to differ in variable importance and feature attributions \cite{Fisher2019,Dong2020,Donnelly2023,Laberge2023} and, at the prediction level, to assign conflicting predictions to the same instance, a phenomenon formalized by \citet{Marx2020a} as \emph{predictive multiplicity}.
A growing literature distinguishes predictive multiplicity from its procedural, explanatory, and dataset-induced forms \cite{Black2022b,Heljakka2023,Ganesh2025}, investigates its root causes \cite{Parikh2026}, weighs its opportunities and risks \cite{Rudin2024,Hsu2026}
and reconcile its implications for downstream decisions \cite{Du2025}.
All of these works treat each instance as an independent prediction problem: 
their notions of multiplicity live at the \emph{pointwise} level of our hierarchy, and are provably blind to the complex structure of time series, as we will later show.

\textbf{Model Equivalence and Multiplicity in Forecasting.}
Forecasting research has long recognized that several models may be statistically indistinguishable in predictive ability: the Model Confidence Set of \citet{Hansen2011} retains the models whose differences cannot be statistically established, \citet{Fosten2021} let this set vary with the horizon through \emph{horizon confidence sets}, while \citet{Quaedvlieg2021} and \citet{Monschang2026} select models with one joint constraint on all horizons. Such statistical tests compare models through their losses, either globally, per horizon, or jointly, and thus determine \emph{which} models are performance-equivalent; they do not, however, characterize how the forecasts of equivalent models differ, which is precisely the gap that this paper addresses.
Recent work has begun to bring a Rashomon-set perspective to forecasting. This includes studies of predictive and procedural multiplicity in recovery-rate forecasting \cite{Hibbeln2026}, functional differences among volatility forecasters with tied loss \cite{Cortesi2026}, and variation in feature reliance among time-series forecasting models \cite{Berry2024}. Closest to our setting, \citet{Kale2026} introduce horizon-constrained Rashomon sets to track predictive multiplicity across forecast horizons in chaotic dynamical systems. All these works sit at the \emph{horizon-wise} level of our hierarchy, which again cannot fully characterize trajectory-level multiplicity.

\section{A Hierarchy of Predictive Multiplicity Notions}
\label{sec:preliminaries}

This section develops our hierarchy of predictive multiplicity notions for forecasting. We begin by introducing the forecasting setup and notation. We then define three kinds of multiplicity: \emph{pointwise}, \emph{horizon-wise}, and \emph{temporal}, which differ in whether they compare isolated predictions, horizon-specific collections of predictions, or complete forecast trajectories. Finally, we show that these levels are not merely alternative aggregations of the same disagreement: each retains a different amount of information about the forecasts.

\textbf{Setup and Notation.} Consider a class of models $\gF$. 
Given a model $f\in\gF$,
let $\Ls(f)$ be the loss of $f$ evaluated over a collection of origins $\gT$.
For a tolerance $\eps\geq 0$, the corresponding \emph{Rashomon set} is \cite{Fisher2019}:
\begin{equation}
    \gR_{\eps}
    =
    \left\{
        f\in\gF:
        \Ls(f)\leq\Lstar+\eps
    \right\}
    \qquad
    \text{where}
    \quad
    \Lstar=\Ls(f^\star),
    \quad
    f^\star\in\argmin_{f\in\gF}\Ls(f)
    .
    \label{eq:rashomon-set}
\end{equation}
\EA{Given a model $f\in\gR_\eps$, let $\widehat{\mY}^{f}\in\R^{T\times H}$ be its \emph{forecast matrix}, $\mY\in\R^{T\times H}$ the \emph{target matrix}, and $\mE^{f}=\mY-\widehat{\mY}^{f}$ its \emph{error matrix}, \EA{with rows} indexed by \emph{origins} $t\in\gT$ and columns by forecast \emph{horizons} $h\in\{1,\ldots,H\}$. We consider three \EA{temporal} \EA{trajectories} extracted from these matrices:
\begin{itemize}[leftmargin=*,nosep,topsep=0pt]
    \item the \emph{forecast trajectory} $\widehat{\vy}^{f}_{t}=(\hat{y}^{f}_{t+1},\ldots,\hat{y}^{f}_{t+H})^{\top}$, the $t$-th row of $\widehat{\mY}^{f}$: forecasts across horizons;
    \item the \emph{revision trajectory} $\widehat{\vr}^{f}_{t}$, the $t$-th anti-diagonal of $\widehat{\mY}^{f}$: successive forecasts of the same target; \item the \emph{error trajectory} $\ve^{f}_{t}=\vy_t-\widehat{\vy}^{f}_{t}$, the $t$-th row of $\mE^{f}$: forecast errors across horizons. \end{itemize}
}
Forecast and revision trajectories also arise in the forecast-stability literature, where their \EA{fluctuations underlie the definitions of} \emph{horizontal} and \emph{vertical} (in)stability, respectively \cite{VanBelle2023,Godahewa2025}. That literature regularizes such fluctuations within the forecasts of a \emph{single} model; here, instead, we use the same geometry to characterize disagreement among \EA{the members of the Rashomon set:} \emph{distinct models with equivalent loss}.

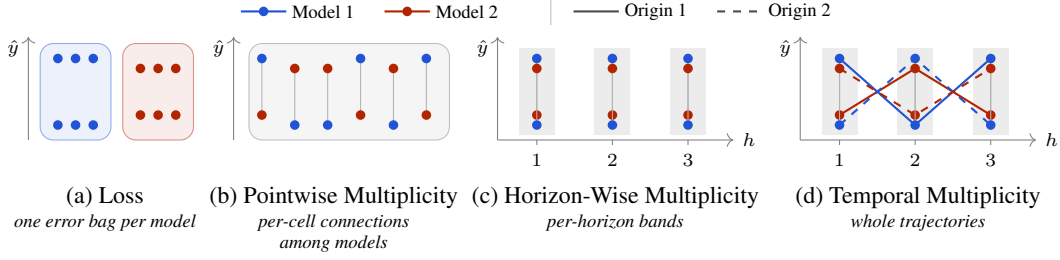
\begin{figure}[t]
    \centering
    \captionsetup[sub]{justification=centering,font=footnotesize}
    \newcommand{\hierpair}[3]{\draw[gray!60,line width=0.4pt] (#1,#2) -- (#1,#3);
        \fill[niceblue] (#1,#2) circle (1.8pt);
        \fill[darkred] (#1,#3) circle (1.8pt);}
    \newcommand{\hieraxis}{\draw[->,gray] (0,-2.9) -- (0,3.3);
        \node[left,font=\scriptsize] at (0,2.6) {$\hat{y}$};}
    \newcommand{\hierbands}{\draw[->,gray] (0,-2.9) -- (3.9,-2.9);
        \node[right,font=\scriptsize] at (3.9,-2.9) {$h$};
        \foreach \h in {1,2,3}{
            \fill[gray!15] (1.25*\h-0.894,-2.6) rectangle (1.25*\h-0.306,2.6);
            \draw[gray] (1.25*\h-0.6,-2.9) -- (1.25*\h-0.6,-3.2);
            \node[below,font=\scriptsize] at (1.25*\h-0.6,-3.2) {$\h$};
        }
}
    \newcommand{\hierdots}{\foreach \h in {1,2,3}{
            \hierpair{1.25*\h-0.6}{2}{-1.4}
            \hierpair{1.25*\h-0.6}{-2}{1.4}
        }}
    \begin{tikzpicture}[baseline=-0.5ex,font=\scriptsize]
        \draw[niceblue,line width=0.8pt] (0,0) -- (0.6,0); \fill[niceblue] (0.3,0) circle (1.8pt); \node[right,inner sep=2pt] at (0.6,0) {Model 1};
        \draw[darkred,line width=0.8pt] (1.9,0) -- (2.5,0); \fill[darkred] (2.2,0) circle (1.8pt); \node[right,inner sep=2pt] at (2.5,0) {Model 2};
        \draw[gray!50,line width=0.4pt] (4.1,-0.18) -- (4.1,0.18);
        \draw[black!70,line width=0.8pt,solid] (4.4,0) -- (5.0,0); \node[right,inner sep=2pt] at (5.0,0) {Origin 1};
        \draw[black!70,line width=0.8pt,dashed] (6.3,0) -- (6.9,0); \node[right,inner sep=2pt] at (6.9,0) {Origin 2};
    \end{tikzpicture}
    \\[-3pt]
    \begin{subfigure}[t]{0.18\textwidth}
        \centering
        \begin{tikzpicture}[x=0.52cm,y=0.22cm]
            \useasboundingbox (-0.4,-4.6) rectangle (4.3,4.3);
            \hieraxis
            \draw[rounded corners=5pt,fill=niceblue!8,draw=niceblue!50] (0.3,-2.9) rectangle (2.1,2.9);
            \draw[rounded corners=5pt,fill=darkred!8,draw=darkred!50] (2.4,-2.9) rectangle (4.2,2.9);
            \foreach \x/\y in {0.75/2, 1.2/2, 1.65/2, 0.75/-2, 1.2/-2, 1.65/-2}
                \fill[niceblue] (\x,\y) circle (1.8pt);
            \foreach \x/\y in {2.85/-1.4, 3.3/-1.4, 3.75/-1.4, 2.85/1.4, 3.3/1.4, 3.75/1.4}
                \fill[darkred] (\x,\y) circle (1.8pt);
        \end{tikzpicture}
        \caption{{Loss}\\[-1pt]{\scriptsize \emph{one error bag per model}}} 
        \label{fig:hierarchy-loss}
    \end{subfigure}
    \hfill
    \begin{subfigure}[t]{0.24\textwidth}
        \centering
        \begin{tikzpicture}[x=0.68cm,y=0.22cm]
            \useasboundingbox (-0.4,-4.6) rectangle (4.3,4.3);
            \hieraxis
            \draw[rounded corners=5pt,fill=gray!8,draw=gray!50] (0.3,-2.9) rectangle (4.2,2.9);
            \foreach \x/\yb/\yr in {0.55/2/-1.4, 1.19/-2/1.4, 1.83/-2/1.4, 2.47/2/-1.4, 3.11/-2/1.4, 3.75/2/-1.4}{
                \hierpair{\x}{\yb}{\yr}
            }
        \end{tikzpicture}
        \caption{{Pointwise Multiplicity}\\[-1pt]{\scriptsize\begin{tabular}[t]{@{}c@{}}\emph{per-cell connections}\\ \emph{among models}\end{tabular}}}
        \label{fig:hierarchy-pointwise} 
    \end{subfigure}
    \hfill
    \begin{subfigure}[t]{0.28\textwidth}
        \centering
        \begin{tikzpicture}[x=0.8cm,y=0.22cm]
            \useasboundingbox (-0.4,-4.6) rectangle (4.3,4.3);
            \hieraxis
            \hierbands
            \hierdots
        \end{tikzpicture}
        \caption{{Horizon-Wise Multiplicity}\\[-1pt]{\scriptsize \emph{per-horizon bands}}}
        \label{fig:hierarchy-horizon}
    \end{subfigure}
    \hfill
    \begin{subfigure}[t]{0.28\textwidth}
        \centering
        \begin{tikzpicture}[x=0.8cm,y=0.22cm]
            \useasboundingbox (-0.4,-4.6) rectangle (4.3,4.3);
            \hieraxis
            \hierbands
            \draw[niceblue,line width=0.8pt] (0.65,2) -- (1.9,-2) -- (3.15,2);
            \draw[darkred,line width=0.8pt] (0.65,-1.4) -- (1.9,1.4) -- (3.15,-1.4);
            \draw[niceblue,line width=0.8pt,dashed] (0.65,-2) -- (1.9,2) -- (3.15,-2);
            \draw[darkred,line width=0.8pt,dashed] (0.65,1.4) -- (1.9,-1.4) -- (3.15,1.4);
            \hierdots
        \end{tikzpicture}
        \caption{{Temporal Multiplicity}\\[-1pt]{\scriptsize \emph{whole trajectories}}}
        \label{fig:hierarchy-temporal}
    \end{subfigure}
    \caption{Pictorial representation of the structure that each of the loss and the three levels of the multiplicity hierarchy is aware of on a running example.}
    \label{fig:hierarchy}
\end{figure}

\textbf{A Hierarchy of Predictive Multiplicity in Forecasting.}
Having established the forecasting setup, we now introduce the three levels of our hierarchy. The \emph{pointwise} level treats every prediction in isolation, while the \emph{horizon-wise} level assesses disagreement separately at each forecast horizon. We then introduce \emph{temporal predictive multiplicity}, a trajectory-level notion that compares how competing models behave jointly across horizons. The first two levels generalize existing multiplicity measures to forecasting and compare scalar predictions through their absolute difference \citep[e.g.,][]{Watson-Daniels2023,Hibbeln2026}; temporal multiplicity instead compares complete trajectories using distances designed to capture distinct aspects of their temporal behavior.

\begin{definition}[Pointwise Predictive Multiplicity]
\label{def:pointwise-pm}
Let $\tau\geq0$.
The \emph{pointwise prediction range} $\nu$, \emph{pointwise prediction diameter} $\omega$, \emph{pointwise ambiguity} $\alpha_{\tau}$, and \emph{pointwise discrepancy} $\varsigma_{\tau}$ of $\gR_\eps$ are: \begin{equation*}
\begin{array}{@{}c@{\quad}c@{}}
    \displaystyle \nu = \mathop{\E}_{\substack{t\in\gT\\ 1\leq h\leq H}}\Big[\max_{f,g\in\gR_\eps}\big|\hat{y}^{f}_{t+h}-\hat{y}^{g}_{t+h}\big|\Big]
    &
    \displaystyle \omega = \max_{f,g\in\gR_\eps}\, \mathop{\E}_{\substack{t\in\gT\\ 1\leq h\leq H}}\Big[\big|\hat{y}^{f}_{t+h}-\hat{y}^{g}_{t+h}\big|\Big]
    \\[1ex]
    \displaystyle \alpha_{\tau} = \mathop{\E}_{\substack{t\in\gT\\ 1\leq h\leq H}}\Big[\1\Big\{\max_{f,g\in\gR_\eps}\big|\hat{y}^{f}_{t+h}-\hat{y}^{g}_{t+h}\big|>\tau\Big\}\Big]
    &
    \displaystyle \varsigma_{\tau} = \max_{f,g\in\gR_\eps}\, \mathop{\E}_{\substack{t\in\gT\\ 1\leq h\leq H}}\Big[\1\Big\{\big|\hat{y}^{f}_{t+h}-\hat{y}^{g}_{t+h}\big|>\tau\Big\}\Big]
\end{array}
\end{equation*}
\end{definition}

The range and ambiguity \emph{average} the per-instance maximal disagreement, while the diameter and discrepancy ask how far a \emph{single} pair of models drifts apart on average, hence $\omega\leq\nu$ and $\varsigma\leq\alpha$.

Intuitively, pointwise {ambiguity} ($\alpha_{\tau}$) and {discrepancy} ($\varsigma_{\tau}$) generalize \citet{Marx2020a} and \citet{Watson-Daniels2023} by maximizing over all pairs of competing models, not against $f^\star$, to capture \emph{all} the multiplicity in $\gR_\eps$, and by treating the $(t,h)$ cells as instances.
The {prediction range} ($\nu$) averages the width of the viable prediction range of \citet{Watson-Daniels2023} over the forecast matrix; the {prediction diameter} ($\omega$) is a novel metric that we propose as its complement.

Following a similar pattern, we 
formalize the per-horizon notion of multiplicity of \citet{Kale2026}.

\begin{definition}[Horizon-Wise Predictive Multiplicity]
\label{def:horizon-pm}
Let $\tau\geq0$.
For each horizon $h\in\{1,\ldots,H\}$, the \emph{horizon-wise prediction range} $\nu^{(h)}$, \emph{prediction diameter} $\omega^{(h)}$, \emph{ambiguity} $\alpha^{(h)}_{\tau}$, and \emph{discrepancy} $\varsigma^{(h)}_{\tau}$ of $\gR_\eps$, producing profiles indexed by $h$, are defined as follows:
\begin{equation*}
\begin{array}{@{}c@{\quad}c@{}}
    \displaystyle \nu^{(h)} = \mathop{\E}_{t\in\gT}\Big[\max_{f,g\in\gR_\eps}\big|\hat{y}^{f}_{t+h}-\hat{y}^{g}_{t+h}\big|\Big]
    &
    \displaystyle \omega^{(h)} = \max_{f,g\in\gR_\eps}\, \mathop{\E}_{t\in\gT}\Big[\big|\hat{y}^{f}_{t+h}-\hat{y}^{g}_{t+h}\big|\Big]
    \\[1ex]
    \displaystyle \alpha^{(h)}_{\tau} = \mathop{\E}_{t\in\gT}\Big[\1\Big\{\max_{f,g\in\gR_\eps}\big|\hat{y}^{f}_{t+h}-\hat{y}^{g}_{t+h}\big|>\tau\Big\}\Big]
    &
    \displaystyle \varsigma^{(h)}_{\tau} = \max_{f,g\in\gR_\eps}\, \mathop{\E}_{t\in\gT}\Big[\1\Big\{\big|\hat{y}^{f}_{t+h}-\hat{y}^{g}_{t+h}\big|>\tau\Big\}\Big]
\end{array}
\end{equation*}
\end{definition}

\label{sec:hierarchy}

We now introduce \emph{temporal} predictive multiplicity. Unlike the pointwise and horizon-wise levels, it compares complete forecast trajectories and can therefore capture how competing models disagree across forecast horizons or across successive revisions.

\begin{definition}[Temporal Predictive Multiplicity]
\label{def:temporal-multiplicity}
Let $\tau\geq0$\EA{, let $\vv^{f}_{t}$ denote any of the trajectories we defined earlier (forecast $\widehat{\vy}^{f}_{t}$, revision $\widehat{\vr}^{f}_{t}$, or error $\ve^{f}_{t}$) and let $\dt:\R^{H}\times\R^{H}\rightarrow[0,\infty)$ be a \emph{trajectory distance function}}.
The \emph{temporal trajectory range} $\nu_{\dt}$, \emph{trajectory diameter} $\omega_{\dt}$, \emph{ambiguity} $\alpha_{\dt,\tau}$, and \emph{discrepancy} $\varsigma_{\dt,\tau}$ of $\gR_\eps$ \EA{for such a trajectory} are defined as:
\begin{equation*}
\begin{array}{@{}c@{\quad}c@{}}
    \displaystyle \nu_{\dt} = \mathop{\E}_{t\in\gT}\Big[\max_{f,g\in\gR_\eps}\dt\big(\vv^{f}_{t},\vv^{g}_{t}\big)\Big]
    &
    \displaystyle \omega_{\dt} = \max_{f,g\in\gR_\eps}\, \mathop{\E}_{t\in\gT}\Big[\dt\big(\vv^{f}_{t},\vv^{g}_{t}\big)\Big]
    \\[1ex]
    \displaystyle \alpha_{\dt,\tau} = \mathop{\E}_{t\in\gT}\Big[\1\Big\{\max_{f,g\in\gR_\eps}\dt\big(\vv^{f}_{t},\vv^{g}_{t}\big)>\tau\Big\}\Big]
    &
    \displaystyle \varsigma_{\dt,\tau} = \max_{f,g\in\gR_\eps}\, \mathop{\E}_{t\in\gT}\Big[\1\Big\{\dt\big(\vv^{f}_{t},\vv^{g}_{t}\big)>\tau\Big\}\Big]
\end{array}
\end{equation*}
\end{definition}

The choice of $\dt$ determines which aspect of temporal disagreement is assessed. Models may differ in the magnitude, direction, volatility, shape, spectral structure, or timing of their trajectories. No single distance captures all of these forms of disagreement without conflating them. We therefore leave $\dt$ unspecified in the general definition and provide concrete instantiations in \cref{sec:temporal-trajectory-distances}.

\textbf{Information Awareness Across Multiplicity Levels.} The temporal statistics defined above summarize disagreement over complete trajectories. We next show that the three levels retain different information about the forecast matrix. We formalize this distinction through their sensitivity to rearrangements of forecast data, demonstrating that pointwise and horizon-wise statistics are blind to aspects of disagreement that depend on how predictions are coupled across horizons.

\begin{definition}[Multiplicity-, Horizon-, and Trajectory-Aware Statistics]
\label{def:awareness}
A \emph{statistic} is any function of the target matrix $\mY$ and the forecast matrices $\{\widehat{\mY}^{f}\}_{f\in\gR_\eps}$.
\EA{A statistic is:}
\begin{itemize}[topsep=0pt,itemsep=0pt]
    \item \emph{multiplicity-aware} if not invariant under \emph{per-model cell permutations} $\mE^{f}\mapsto\pi^{f}(\mE^{f})$, with $\pi^{f}$ a permutation of the cells $\gT\times\{1,\ldots,H\}$ that may depend on $f$;
    \item \emph{horizon-aware} if not invariant under \emph{cell permutations}, per-model cell permutations with a common $\pi^{f}\equiv\pi$;
    \item \emph{trajectory-aware} if not invariant under \emph{row permutations}, cell permutations $\pi(t,h)=(\pi_h(t),h)$ that preserve horizons and re-index the origins within each horizon.
\end{itemize}
\end{definition}

These statistics are nested, since row permutations are cell permutations, and cell permutations are per-model cell permutations. As 
illustrated in \cref{fig:hierarchy}, the loss retains none of this structure; pointwise statistics retain model alignment but treat all $TH$ cells as an unordered collection; and horizon-wise statistics additionally distinguish horizons, but not the ordering of origins within them. Only temporal statistics retain full trajectory structure. For example, swapping origins at one horizon flattens the trajectories in the final panel of \cref{fig:hierarchy} without changing any lower-level statistic, despite changing whether models agree on the direction of change. \EA{The proof is in \cref{app:horizon-wise-invariance}}.

\section{How Loosely Is Temporal Multiplicity Constrained?}
\label{sec:theory}

Knowing what each level captures, we now quantify how loosely information below the temporal level constrains temporal multiplicity.
\label{sec:bounds}
We consider losses of the following common form:
\begin{equation}
    \Ls(f)
    =
    \frac{1}{TH}
    \sum_{t\in\gT}
    \sum_{h=1}^{H}
    \ell\big(e^{f}_{t,h}\big),
    \label{eq:decomposable-loss}
\end{equation}
\EA{for some continuous $\ell:\R\rightarrow[0,\infty)$ with $\ell(0)=0$ that is nondecreasing in $|e|$ and unbounded ($\ell(e)\rightarrow\infty$ as $|e|\rightarrow\infty$), with inverse $\ell^{-1}(L)=\max\{e\geq 0:\ell(e)\leq L\}$.
This covers most standard losses, including MSE ($\ell(e)=e^2$), MAE ($\ell(e)=|e|$)}, as well as MASE-type losses.

We begin by quantifying how loosely the performance constraint defining the Rashomon set constrains temporal multiplicity.
The bound below is tight for sufficiently rich model classes: we call $\gF$ \emph{shell-rich} (at level $\eps$) if for every $\mE\in\R^{T\times H}$ with $\Lstar\leq\Ls(\mE)\leq\Lstar+\eps$ there exists $f\in\gF$ with $\mE^{f}=\mE$.
This is a mild condition, far weaker than universal approximation \citep{Hornik1991}: $\gF$ need only realize error patterns with loss near $\Lstar$.
It is readily satisfied by sufficiently overparameterized networks, which can interpolate arbitrary values at finitely many inputs \citep{Zhang2017}.

\begin{restatable}[Multiplicity Bounds Under Rashomon Set Membership]{theorem}{restateTemporalMultiplicity}
\label{thm:temporal-multiplicity}
For all $f,g\in\gR_\eps$:
\[
    \big|
    \hat{y}^{f}_{t+h}
    -
    \hat{y}^{g}_{t+h}
    \big|
    \;\leq\;
    2\,\ell^{-1}\!\big(TH\,(\Lstar+\eps)\big)
    \qquad\forall t\in\gT,\ h\in\{1,\ldots,H\}.
\]
If $\gF$ is also \emph{shell-rich}, the bound is attained, i.e., $\max_{f,g\in\gR_\eps}\big|\hat{y}^{f}_{t+h}-\hat{y}^{g}_{t+h}\big|=2\,\ell^{-1}\!\big(TH\,(\Lstar+\eps)\big)$.
\end{restatable}

Theorem~\ref{thm:temporal-multiplicity} shows that Rashomon set membership imposes a single scalar budget on the $TH$-dimensional error matrix, without constraining where errors fall within such a matrix.
It thus uncovers a failure mode of performance-based \EA{model selection} in temporal settings: beyond leaving open how errors distribute across instances \citep{Marx2020a}, the loss is blind to how they distribute \emph{across horizons within each instance}, i.e., to the shape of every forecast trajectory.
Moreover, this freedom \emph{scales with the horizon}: for $\ell(e)=|e|^p$, the theorem's pointwise bound grows as $H^{1/p}$ ($\sqrt{H}$ under MSE and linearly under MAE).
Longer horizons therefore weaken the control that Rashomon membership provides over individual forecasts, limiting what membership alone guarantees about the resulting trajectories.

\label{sec:enforcing} 

At the opposite extreme of our hierarchy lies horizon-wise multiplicity. We now quantify how constraining at this level affects temporal multiplicity by imposing the strongest requirement afforded by its definition, \emph{vanishing horizon-wise ambiguity}: $\alpha^{(h)}_{\tau_h}=0$ for every horizon $h$, at tolerances $\bm{\tau}=(\tau_1,\ldots,\tau_H)$. This implies $\nu^{(h)},\omega^{(h)}\leq\tau_h$ and $\varsigma^{(h)}_{\tau_h}=0$. When $\tau_h\equiv\tau$, this condition also recovers the pointwise case, $\alpha_{\tau}=0$.

\begin{restatable}[Multiplicity Bound Under Horizon-Wise Control]{theorem}{restateStructuralFreedom}
\label{thm:structural-freedom}
Fix tolerances $\bm{\tau}=(\tau_1,\ldots,\tau_H)$.
Then, every pair $f,g\in\gF$ with $\alpha^{(h)}_{\tau_h}(\{f,g\})=0$ for all $h$ satisfies:
\begin{equation*} 
    \big|
    \hat{y}^{f}_{t+h}
    -
    \hat{y}^{g}_{t+h}
    \big|
    \;\leq\;
    \tau_h
    \qquad\forall t\in\gT,\ h\in\{1,\ldots,H\}.
\end{equation*}
If, moreover, $\gF$ is \emph{shell-rich}, the bound is attained by pairs in $\gR_\eps$ whenever it is \emph{loss-feasible}, i.e., whenever the loss budget of \cref{thm:temporal-multiplicity} does not bind first: $\sum_{h}\ell(\tau_h/2)\leq TH(\Lstar+\eps)$.
\end{restatable}

Indeed, horizon-wise multiplicity imposes a stronger constraint on predictions, bounding the absolute difference between models at each horizon by $\tau_h$. Yet, except in the degenerate case where $\tau_h=0$ for every horizon and all models in the Rashomon set produce identical forecasts for every sample, this constraint leaves many temporal properties of the trajectories unrestricted. For example, it provides no control over differences in their direction, spectral behavior, or shape.

\section{Instantiations of Temporal Multiplicity}
\label{sec:temporal-trajectory-distances}

\cref{def:temporal-multiplicity} deliberately leaves the choice of \emph{trajectory distance function} unspecified. We now instantiate this general definition with a family of interpretable distances, each designed to capture a distinct aspect of temporal disagreement. As there is no unique trajectory distance appropriate for all forecasting applications, we select distances that are complementary along two criteria: each has a clear interpretation in terms of a recognizable aspect of temporal behavior, and together they distinguish differences in level, local dynamics, and global trajectory structure. 

We begin with a magnitude distance, which aggregates the pointwise gaps between two trajectories $\vu,\vv\in\R^H$ and therefore retains the least temporal structure. It is consequently the temporal distance most closely related to pointwise and horizon-wise multiplicity.
\begin{itemize}[nosep,topsep=0pt]
    \item \emph{Magnitude}: $\dt_{\mathrm{magnitude}}(\vu,\vv)=\frac{1}{\sqrt{H}}\|\vu-\vv\|_2$, the root-mean-square pointwise gap.
\end{itemize}
The next three distances operate on trajectory increments and compare patterns of change independently of the trajectories' absolute levels.
\begin{itemize}[nosep,topsep=0pt]
    \item \emph{Increment}: $\dt_{\mathrm{increment}}(\vu,\vv)=\frac{1}{\sqrt{H-1}}\|\Delta\vu-\Delta\vv\|_2$, with $\Delta\vu=(\Delta u_1,\ldots,\Delta u_{H-1})$.

    \item \emph{Volatility}: $\dt_{\mathrm{volatility}}(\vu,\vv)=\frac{1}{H-1}\big|\sum_h|\Delta u_h|-\sum_h|\Delta v_h|\big|$, the gap between the mean absolute increments.

    \item \emph{Direction}: $\dt_{\mathrm{direction}}(\vu,\vv)=\frac{1}{H-1}\big|\{h:\sign(\Delta u_h)\neq\sign(\Delta v_h)\}\big|$, the fraction of increments of opposite sign.
\end{itemize}
These final three distances are scale-free, comparing the form of trajectories rather than their absolute values.\footnote{Here, $\rho$ denotes Pearson correlation, $\digamma$ the normalized discrete Fourier transform, and $\rho_\tau$ the lag-$\tau$ cross-correlation. See \cref{app:multiplicity-metrics} for precise numerical implementations.}
\begin{itemize}[topsep=1pt]
    \item \emph{Shape}: $\dt_{\mathrm{shape}}(\vu,\vv)=\tfrac{1}{2}\big(1-\rho(\vu,\vv)\big)\in[0,1]$ using the Pearson correlation.

    \item \emph{Spectral}: $\dt_{\mathrm{spectral}}(\vu,\vv)=\tfrac{1}{\sqrt{2}}\big\|\,|\digamma(\vu)|-|\digamma(\vv)|\,\big\|_2\in[0,1]$, the magnitude spectrum gap.

    \item \emph{Timing}: $\dt_{\mathrm{timing}}(\vu,\vv)=|\tau^\star|/(H-2)\in[0,1]$ with $\tau^\star=\argmax_{|\tau|\leq H-2}\rho_\tau(\vu,\vv)$, the best-alignment lag as a fraction of the largest admissible lag.
\end{itemize}

Each distance can be applied to forecast, revision, or error trajectories. On forecast trajectories $\widehat{\vy}^{f}_{t}$, it measures disagreement about the future path of the target; on revision trajectories $\widehat{\vr}^{f}_{t}$, it measures differences in how models update forecasts of a common target as new observations become available; on error trajectories $\ve^{f}_{t}$, it measures differences in the temporal structure of model errors.

The instantiations above also make the general results of \cref{sec:theory} operational: substituting each distance into our theoretical framework yields metric-specific bounds on temporal disagreement. We defer these bounds, their proofs, and an empirical assessment of their tightness to \cref{app:bounds}, and turn next to experiments that characterize temporal multiplicity among trained forecasting models.

\section{Experiments}
\label{sec:experiments}
We evaluate temporal predictive multiplicity in trained forecasting models guided by three questions:
\begin{enumerate}[label={\bfseries Q\arabic*:}, leftmargin=*, topsep=2pt, itemsep=0pt]
    \item How much do near-optimal models differ across the temporal aspects captured by our trajectory distances?
    \item Can pointwise or horizon-wise disagreement serve as a proxy for such differences?
    \item Can temporal multiplicity alter the conclusions or decisions drawn from a forecast?
\end{enumerate}
Across the experiments, we find that near-optimal models frequently produce markedly different trajectories, that these differences are largely invisible to conventional multiplicity measures, and that they can have consequential downstream implications.

\textbf{Experimental Setup.}
We consider 11 datasets from diverse domains. For each dataset, we train models spanning 19 neural forecasting architectures, including transformer, MLP/mixer-based, and recurrent families, over 10 random seeds and several hyper-parameter configurations. The empirical Rashomon set $\egR$ contains all models within $\eps=5\%$ of the best validation MASE. We compute all multiplicity statistics from rolling-origin test forecasts of its members. Full details of the datasets, models, training procedure, and evaluation protocol appear in \cref{app:experimental-setup}.

\textbf{Q1. How large is temporal multiplicity?}
We first quantify the extent to which near-optimal models differ in the temporal aspects captured by our trajectory distances. For each distance, we compute the trajectory range and diameter of the empirical Rashomon set. \Cref{tab:temporal-main} reports the results.\footnote{\Cref{app:horizon-wise-results} reports the complete pointwise, horizon-wise, and temporal multiplicity statistics, including ambiguity and discrepancy, per target.}
 
\begin{table}[t]
    \centering
    \caption{Pointwise, horizon-wise, and temporal predictive multiplicity in empirical Rashomon sets. Each cell reports the trajectory range / diameter ($\nu_{\dt} / \omega_{\dt}$), maximized across forecast, revision, and error trajectories. For multi-target datasets, multiplicity statistics are maximized and MASE is averaged across targets. Magnitude, increment, and volatility are reported in seasonal-naive units; direction, timing, shape, and spectral distances are scale-free and reported as percentages.}
    \label{tab:temporal-main}
    \footnotesize
    \setlength{\tabcolsep}{3pt}
    \resizebox{\textwidth}{!}{\begin{tabular}{llllllll|lll|l}
\toprule
Dataset & $\dt_{\mathrm{magnitude}}$ & $\dt_{\mathrm{increment}}$ & $\dt_{\mathrm{volatility}}$ & $\dt_{\mathrm{direction}}$ & $\dt_{\mathrm{shape}}$ & $\dt_{\mathrm{spectral}}$ & $\dt_{\mathrm{timing}}$ & $\nu^{(1)}$ / $\omega^{(1)}$ & $\nu^{(H)}$ / $\omega^{(H)}$ & $\nu$ / $\omega$ & MASE$^{*}$ \\
\midrule
ATM 1D & 1.38 / 0.66 & 1.55 / 0.81 & 0.87 / 0.35 & 100\% / 59\% & 96\% / 54\% & 85\% / 36\% & 100\% / 54\% & 1.79 / 0.59 & 1.94 / 0.62 & 1.85 / 0.55 & 1.05 \\
ATM 1W & 1.46 / 0.67 & 1.38 / 0.65 & 0.94 / 0.49 & 100\% / 55\% & 99\% / 56\% & 85\% / 32\% & 100\% / 58\% & 1.75 / 0.71 & 1.67 / 0.61 & 1.68 / 0.59 & 0.93 \\
Bitcoin & 7.00 / 2.99 & 3.66 / 2.08 & 2.01 / 0.88 & 100\% / 62\% & 99\% / 61\% & 82\% / 42\% & 100\% / 69\% & 5.16 / 2.35 & 9.41 / 3.34 & 7.10 / 2.63 & 2.67 \\
Electricity D. A & 1.56 / 0.62 & 0.47 / 0.33 & 0.21 / 0.15 & 79\% / 49\% & 83\% / 39\% & 82\% / 43\% & 98\% / 33\% & 1.59 / 0.46 & 2.19 / 0.58 & 1.93 / 0.53 & 0.46 \\
Electricity D. Z & 0.80 / 0.30 & 0.20 / 0.15 & 0.07 / 0.05 & 69\% / 50\% & 79\% / 44\% & 77\% / 49\% & 99\% / 40\% & 0.62 / 0.22 & 1.17 / 0.32 & 0.97 / 0.24 & 0.24 \\
Electricity Price & 0.95 / 0.45 & 0.40 / 0.29 & 0.18 / 0.11 & 81\% / 46\% & 72\% / 35\% & 76\% / 38\% & 96\% / 30\% & 0.98 / 0.34 & 1.29 / 0.48 & 1.24 / 0.37 & 0.82 \\
ETT 1H & 1.42 / 0.56 & 0.50 / 0.33 & 0.23 / 0.14 & 83\% / 47\% & 77\% / 33\% & 79\% / 37\% & 95\% / 31\% & 1.13 / 0.44 & 1.97 / 0.60 & 1.66 / 0.47 & 0.48 \\
Exchange Rate & 5.31 / 1.78 & 5.25 / 2.31 & 3.05 / 1.32 & 100\% / 100\% & 100\% / 100\% & 97\% / 87\% & 100\% / 100\% & 4.09 / 1.59 & 6.26 / 2.02 & 5.35 / 1.36 & 2.36 \\
FRED-MD & 0.16 / 0.12 & 0.13 / 0.09 & 0.09 / 0.05 & 51\% / 35\% & 49\% / 36\% & 0\% / 0\% & 61\% / 45\% & 0.11 / 0.08 & 0.21 / 0.13 & 0.16 / 0.11 & 0.14 \\
Gasoline & 2.74 / 1.33 & 2.13 / 1.00 & 1.23 / 0.56 & 100\% / 72\% & 100\% / 66\% & 94\% / 46\% & 100\% / 79\% & 3.03 / 1.32 & 3.29 / 1.29 & 3.10 / 1.22 & 0.74 \\
Weather & 2.41 / 1.25 & 0.72 / 0.51 & 0.39 / 0.27 & 66\% / 52\% & 88\% / 50\% & 77\% / 47\% & 99\% / 48\% & 2.73 / 1.39 & 3.12 / 1.10 & 2.96 / 1.06 & 0.95 \\
\midrule
Mean & 2.29 / 0.98 & 1.49 / 0.78 & 0.84 / 0.40 & 84\% / 57\% & 86\% / 52\% & 76\% / 41\% & 95\% / 53\% & 2.09 / 0.86 & 2.96 / 1.01 & 2.55 / 0.83 & 0.99 \\
\bottomrule
\end{tabular}
} 
\end{table} 

Temporal multiplicity is substantial across all temporal aspects. At a typical origin, some pair of near-optimal models moves in \emph{opposite} directions on $84\%$ of trajectory increments, reaching $100\%$ on Gasoline; even a fixed pair does so on $57\%$ of increments.A shape range of $86\%$ corresponds to a Pearson correlation of approximately $-0.7$ between the most divergent trajectories, with one trajectory tending to rise at the time steps where the other falls. The timing range reaches $95\%$, indicating that the most divergent trajectories often require a shift of about half the forecasting horizon to align, and the spectral range of $76\%$ lies well above the midpoint of its scale. Thus, models with indistinguishable validation performance can imply sharply different trajectories.

The pointwise and horizon-wise results are also consistent with prior work: pointwise multiplicity increases with the forecast horizon \citep{Marx2020a,Watson-Daniels2023,Kale2026}. However, the temporal results are not simply a rescaling of the pointwise ones. The two measures rank datasets differently; for example, Weather has the fourth-largest pointwise range but the second-smallest directional disagreement. We examine this distinction in the next experiment.

\textbf{Q2. Can conventional frameworks capture temporal multiplicity?}
The preceding theoretical results suggest that temporal multiplicity is not simply a rescaling of pointwise or horizon-wise multiplicity. We test this directly at the level of individual model pairs and forecast origins. For each pair of models in the empirical Rashomon set and each test forecast origin, we compute the Spearman correlation between their mean pointwise gap and their trajectory distance. We consider both a pointwise reading and a horizon-normalized reading, which scales gaps by the best validation MASE at each horizon:
\begin{equation}
    r_{\dt}
    =
    \varrho\Big(\big[\bar{\delta}^{fg}_{t}\big]_{\substack{\{f,g\}\in\binom{\egR}{2}\\ t\in\gT}},\;\big[\dt(\vv^{f}_{t},\vv^{g}_{t})\big]_{\substack{\{f,g\}\in\binom{\egR}{2}\\ t\in\gT}}\Big),
    \qquad
    \bar{\delta}^{fg}_{t}=\frac{1}{H}\sum_{h=1}^{H}\frac{\big|\nv{\hat{y}}^{f}_{t+h}-\nv{\hat{y}}^{g}_{t+h}\big|}{\Ls^{\star}_h},
    \label{eq:raw-correlation}
\end{equation}
where $\binom{\egR}{2}$ is the set of unordered pairs of distinct models in $\egR$, $\nv{\hat{y}}^{f}_{t+h}$ and $\nv{\hat{y}}^{g}_{t+h}$ are the forecasts in naive units, $\Ls^{\star}_h=\min_{f\in\egR}\Ls_h(f)$ is the best validation MASE at horizon $h$, and $\varrho$ denotes Spearman rank correlation.
The horizon-normalized reading places gaps from different horizons on a common scale; the unweighted pointwise reading is obtained by setting $\Ls^\star_h\equiv1$. \Cref{tab:pointwise-n-horizon-wise-multiplicity-vs-temporal-correlation} reports the weakest and strongest correlations across trajectories and, for multi-target datasets, targets.

The results are aligned with our theory. Correlation is nearly perfect for magnitude, which directly aggregates pointwise gaps, and decreases for increment and volatility distances. In contrast, disagreement in direction, shape, spectral behavior, and timing has little or no association with pointwise or horizon-wise disagreement. On the least correlated trajectory, average correlations for these aspects range from $0.04$ to $0.16$, while even the strongest correlations do not exceed $0.47$. Thus, knowing how far apart two near-optimal forecasts are at individual horizons provides little information about whether they agree on the direction, shape, spectrum, or timing of the trajectory.

\begin{table}[t]
    \centering
    \caption{Spearman rank correlation, over all (model pair, origin) samples of the empirical Rashomon set, between the pair's \emph{mean pointwise gap} $\bar{\delta}^{fg}_{t}$ and its \emph{trajectory distance} $\dt(\vv^{f}_{t},\vv^{g}_{t})$ on the same forecast origin, read as in \eqref{eq:raw-correlation} with $\Ls^{\star}_h\equiv1$ (Pointwise) or as is (Horizon-Wise).
    Each row reports the minimum / maximum over the three trajectories and, for multi-target datasets, its targets.}
    \label{tab:pointwise-n-horizon-wise-multiplicity-vs-temporal-correlation}
    \footnotesize
    \setlength{\tabcolsep}{3pt}
    \resizebox{\textwidth}{!}{\begin{tabular}{llllllll|lllllll}
\toprule
Dataset & \multicolumn{7}{l}{Pointwise} & \multicolumn{7}{l}{Horizon-Wise} \\
\cmidrule(lr){2-8} \cmidrule(lr){9-15}
 & $\dt_{\mathrm{magnitude}}$ & $\dt_{\mathrm{increment}}$ & $\dt_{\mathrm{volatility}}$ & $\dt_{\mathrm{direction}}$ & $\dt_{\mathrm{shape}}$ & $\dt_{\mathrm{spectral}}$ & $\dt_{\mathrm{timing}}$ & $\dt_{\mathrm{magnitude}}$ & $\dt_{\mathrm{increment}}$ & $\dt_{\mathrm{volatility}}$ & $\dt_{\mathrm{direction}}$ & $\dt_{\mathrm{shape}}$ & $\dt_{\mathrm{spectral}}$ & $\dt_{\mathrm{timing}}$ \\
\midrule
ATM 1D & 0.98 / 0.98 & 0.58 / 0.75 & 0.35 / 0.37 & 0.15 / 0.28 & 0.24 / 0.61 & 0.05 / 0.42 & 0.07 / 0.24 & 0.98 / 0.98 & 0.58 / 0.75 & 0.35 / 0.37 & 0.15 / 0.28 & 0.24 / 0.61 & 0.05 / 0.42 & 0.07 / 0.24 \\
ATM 1W & 0.99 / 0.99 & 0.56 / 0.70 & 0.39 / 0.47 & 0.13 / 0.18 & 0.17 / 0.40 & 0.04 / 0.24 & 0.14 / 0.20 & 0.99 / 0.99 & 0.56 / 0.70 & 0.39 / 0.47 & 0.13 / 0.18 & 0.17 / 0.40 & 0.04 / 0.24 & 0.14 / 0.20 \\
Bitcoin & 0.99 / 1.00 & 0.61 / 0.69 & 0.43 / 0.49 & -0.02 / 0.19 & -0.01 / 0.32 & -0.05 / 0.26 & 0.02 / 0.14 & 0.96 / 0.98 & 0.60 / 0.67 & 0.41 / 0.48 & -0.03 / 0.17 & -0.03 / 0.30 & -0.05 / 0.25 & -0.01 / 0.13 \\
Electricity D. A & 0.99 / 0.99 & 0.46 / 0.66 & 0.16 / 0.36 & -0.03 / 0.26 & -0.00 / 0.64 & -0.08 / 0.49 & 0.05 / 0.36 & 0.96 / 0.97 & 0.47 / 0.65 & 0.16 / 0.36 & -0.03 / 0.26 & -0.03 / 0.63 & -0.10 / 0.48 & 0.04 / 0.35 \\
Electricity D. Z & 0.99 / 0.99 & 0.23 / 0.52 & 0.09 / 0.28 & -0.08 / 0.23 & 0.07 / 0.52 & -0.05 / 0.38 & 0.10 / 0.37 & 0.98 / 0.99 & 0.24 / 0.52 & 0.09 / 0.28 & -0.07 / 0.23 & 0.07 / 0.51 & -0.06 / 0.38 & 0.10 / 0.37 \\
Electricity Price & 0.99 / 0.99 & 0.59 / 0.64 & 0.20 / 0.31 & -0.02 / 0.06 & 0.17 / 0.35 & 0.08 / 0.23 & 0.10 / 0.14 & 0.98 / 0.98 & 0.59 / 0.64 & 0.20 / 0.31 & -0.02 / 0.06 & 0.15 / 0.34 & 0.07 / 0.22 & 0.10 / 0.14 \\
ETT 1H & 0.99 / 0.99 & 0.56 / 0.68 & 0.30 / 0.39 & 0.32 / 0.36 & 0.43 / 0.51 & 0.31 / 0.39 & 0.17 / 0.29 & 0.98 / 0.98 & 0.58 / 0.69 & 0.31 / 0.40 & 0.33 / 0.38 & 0.41 / 0.50 & 0.30 / 0.39 & 0.17 / 0.29 \\
Exchange Rate & 0.99 / 0.99 & 0.50 / 0.83 & 0.29 / 0.51 & 0.09 / 0.22 & 0.14 / 0.47 & 0.01 / 0.28 & 0.06 / 0.22 & 0.96 / 0.98 & 0.49 / 0.82 & 0.29 / 0.51 & 0.09 / 0.21 & 0.14 / 0.45 & 0.01 / 0.27 & 0.06 / 0.20 \\
FRED-MD & 0.98 / 1.00 & 0.61 / 0.69 & 0.45 / 0.64 & 0.03 / 0.40 & 0.05 / 0.42 & -0.11 / 0.04 & 0.02 / 0.40 & 0.97 / 0.99 & 0.59 / 0.67 & 0.44 / 0.62 & 0.03 / 0.39 & 0.04 / 0.42 & -0.10 / 0.04 & 0.03 / 0.39 \\
Gasoline & 0.99 / 0.99 & 0.56 / 0.61 & 0.35 / 0.45 & 0.09 / 0.22 & 0.12 / 0.43 & -0.01 / 0.29 & 0.10 / 0.23 & 0.99 / 0.99 & 0.56 / 0.61 & 0.35 / 0.45 & 0.09 / 0.22 & 0.12 / 0.43 & -0.01 / 0.29 & 0.10 / 0.23 \\
Weather & 1.00 / 1.00 & 0.69 / 0.72 & 0.51 / 0.52 & 0.35 / 0.43 & 0.34 / 0.52 & 0.24 / 0.43 & 0.27 / 0.35 & 0.99 / 0.99 & 0.70 / 0.73 & 0.52 / 0.54 & 0.37 / 0.44 & 0.34 / 0.52 & 0.24 / 0.43 & 0.27 / 0.35 \\
\midrule
Mean & 0.99 / 0.99 & 0.54 / 0.68 & 0.32 / 0.44 & 0.09 / 0.26 & 0.16 / 0.47 & 0.04 / 0.31 & 0.10 / 0.27 & 0.98 / 0.98 & 0.54 / 0.68 & 0.32 / 0.44 & 0.09 / 0.26 & 0.15 / 0.46 & 0.04 / 0.31 & 0.10 / 0.26 \\
\bottomrule
\end{tabular}
}
\end{table}

\textbf{Q3. Does temporal multiplicity affect downstream decisions?} We finally examine whether trajectory-level disagreement can lead to different downstream decisions. Such decisions often depend on properties of an entire forecast trajectory, rather than on individual forecast cells. As an illustrative case study, we consider disagreement in the value and timing of trajectory extrema and in the direction of trajectory increments. For a forecast trajectory, define its maximum $v^{f,M}_{t}$, the location of its maximum $\ell^{f,M}_{t}$, and its increments $\Delta^{f}_{t,h}$ as:
\begin{equation}
    v^{f,M}_{t}=\max_{1\leq h\leq H}\nv{\hat{y}}^{f}_{t+h},
    \qquad
    \ell^{f,M}_{t}=\frac{1}{H}\argmax_{1\leq h\leq H}\nv{\hat{y}}^{f}_{t+h},
    \qquad
    \Delta^{f}_{t,h}=\nv{\hat{y}}^{f}_{t+h}-\nv{\hat{y}}^{f}_{t+h-1}.
    \label{eq:consequences-readings}
\end{equation}
We define $v^{f,m}_{t}$ and $\ell^{f,m}_{t}$ analogously for the minimum and its location. For each origin, disagreement about an extremum is measured by its range across members of $\egR$, in both value and location:
\begin{equation}
    \mathrm{range}_{t}(\vv)=\frac{\max_{f\in\egR}v^{f}_{t}-\min_{f\in\egR}v^{f}_{t}}{\Ls(f^\star)},
    \qquad
    \mathrm{range}_{t}(\vell)=\max_{f\in\egR}\ell^{f}_{t}-\min_{f\in\egR}\ell^{f}_{t}.
    \label{eq:consequences-extremes}
\end{equation}
The value range is expressed in multiples of the best validation MASE, while the location range is expressed as a fraction of the horizon. For directional disagreement, we consider the minority share of models moving against the majority direction, the share of \emph{contested cells} on which at least one model predicts an increase and another a decrease, and the range of increment angles $\theta^{f}_{t,h}=\arctan\Delta^{f}_{t,h}$ across models. \Cref{tab:consequences} reports these quantities for each dataset.

\begin{table}[t]
    \centering
    \caption{Disagreement across empirical Rashomon-set members about the value and location of forecast minima and maxima, and about the direction of forecast increments. Value ranges are expressed in multiples of the best validation MASE; locations are fractions of the forecast horizon; and increment-angle ranges are in degrees (up to $180^{\circ}$, opposite steep moves). Ranges are reported as mean (maximum) across origins, or across forecast cells for increment angles; directional shares are averaged across cells. We report the maximum for multi-target datasets. }
    \label{tab:consequences}
    \footnotesize
    \setlength{\tabcolsep}{3pt}
    \resizebox{\textwidth}{!}{\begin{tabular}{llllllll}
\toprule
Dataset & \multicolumn{2}{l}{Extremes Value} & \multicolumn{2}{l}{Extremes Location} & \multicolumn{3}{l}{Increment Direction} \\
\cmidrule(lr){2-3} \cmidrule(lr){4-5} \cmidrule(lr){6-8}
 & $\mathrm{range}(\vv^{M})$ & $\mathrm{range}(\vv^{m})$ & $\mathrm{range}(\vell^{M})$ & $\mathrm{range}(\vell^{m})$ & Contested Cells & Minority Share & $\mathrm{range}(\vtheta)$ \\
\midrule
ATM 1D & 1.91 (6.22) & 1.43 (5.40) & 18.7\% (85.7\%) & 48.3\% (85.7\%) & 51.0\% & 9.2\% & 55.8 (159) \\
ATM 1W & 1.83 (6.82) & 1.56 (8.17) & 68.3\% (75.0\%) & 70.7\% (75.0\%) & 89.8\% & 21.8\% & 69.2 (153) \\
Bitcoin & 2.83 (7.35) & 2.48 (9.21) & 84.7\% (85.7\%) & 85.2\% (85.7\%) & 100\% & 28.7\% & 101 (164) \\
Electricity D. A & 5.44 (39.4) & 3.13 (17.1) & 43.9\% (97.9\%) & 28.2\% (97.9\%) & 60.9\% & 10.3\% & 34.6 (143) \\
Electricity D. Z & 4.18 (20.9) & 2.19 (7.52) & 30.3\% (99.0\%) & 17.7\% (99.0\%) & 70.1\% & 13.3\% & 17.2 (89.3) \\
Electricity Price & 1.45 (4.13) & 1.53 (4.59) & 45.3\% (95.8\%) & 30.8\% (95.8\%) & 66.0\% & 10.0\% & 34.4 (110) \\
ETT 1H & 3.33 (8.40) & 3.24 (8.07) & 76.6\% (95.8\%) & 85.1\% (95.8\%) & 100\% & 23.2\% & 46.2 (105) \\
Exchange Rate & 2.72 (482) & 2.28 (64.5) & 80.0\% (80.0\%) & 80.0\% (80.0\%) & 100\% & 41.3\% & 118 (177) \\
FRED-MD & 1.30 (3.12) & 1.01 (4.82) & 27.7\% (66.7\%) & 27.4\% (66.7\%) & 47.0\% & 8.2\% & 5.50 (19.0) \\
Gasoline & 3.91 (17.2) & 4.26 (21.0) & 74.4\% (75.0\%) & 74.0\% (75.0\%) & 92.7\% & 21.9\% & 83.6 (160) \\
Weather & 4.09 (15.0) & 2.11 (12.0) & 79.7\% (99.3\%) & 69.6\% (99.3\%) & 100\% & 33.4\% & 60.3 (157) \\
\midrule
Mean & 3.00 (55.5) & 2.29 (14.8) & 57.2\% (86.9\%) & 56.1\% (86.9\%) & 79.8\% & 20.1\% & 56.8 (131) \\
\bottomrule
\end{tabular}
}
\end{table}

The consequences are substantial. On average, the most and least optimistic near-optimal models place the forecast peak $3.00$ error budgets apart and disagree about its timing by $57\%$ of the horizon; at the worst origin, they locate it at opposite ends of the horizon on every dataset. Moreover, at least one model predicts an increase while another predicts a decrease on $80\%$ of forecast cells. 
At a typical cell, $20\%$ of models oppose the majority direction and the members' increment angles span $57^{\circ}$ on average, i.e., steep rises against steep falls.
Thus, performance-indistinguishable models can support different downstream decisions. Horizon-wise statistics reveal disagreement at individual steps, but not whether models agree on the direction of changes or on which peak is largest.

\Cref{fig:consequences-examples} illustrates these differences. For Bitcoin, 75 of 173 near-optimal models predict that the price will rise over the following week, while 98 predict a fall. For Spanish day-ahead electricity prices, all 208 models predict the same intraday shape, yet 98 forecast a higher price after 24 hours and 110 a lower one. The demand examples show analogous disagreement about peaks: in New South Wales, predicted peak increases range from $1.2$ to $10.9$ naive units; in Queensland, 165 of 332 models place the peak on the same evening and 167 on the following morning. Comparable predictive performance therefore leaves unresolved both the magnitude and timing of peak demand.

\begin{figure}[t]
    \centering
    \includegraphics[width=\textwidth]{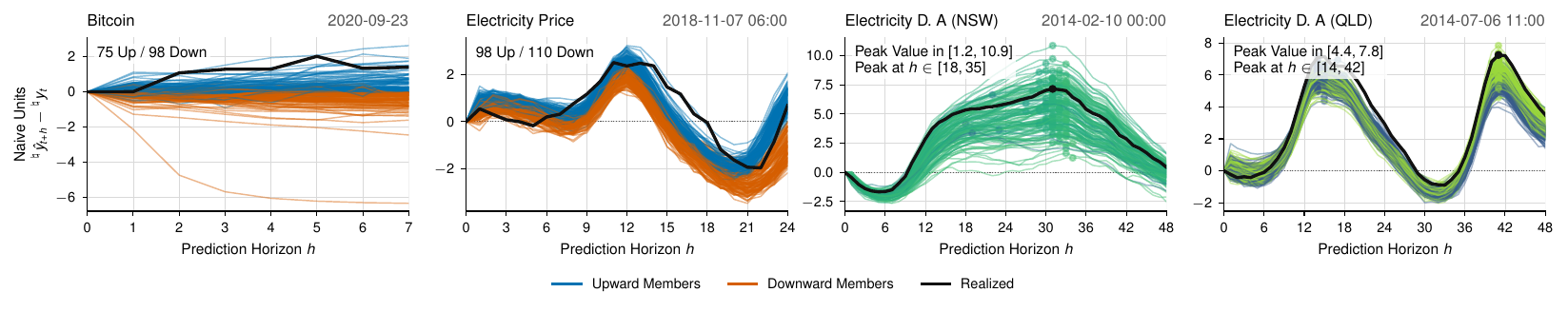}
    \caption{Examples of temporal multiplicity within empirical Rashomon sets. Each panel shows member forecast trajectories at one origin, as deviations from the realized level at that origin, together with the realized trajectory (black). Left: Bitcoin and Spanish electricity-price forecasts, colored by net direction over the horizon (blue: up; orange: down). Right: Australian electricity-demand forecasts, colored by the timing of their maximum (dots); annotations report disagreement in peak value and timing.}
    \label{fig:consequences-examples}
\end{figure}

\section{Conclusions}
\label{sec:conclusions}

We introduced a framework for characterizing predictive multiplicity in time-series forecasting at the level of complete trajectories. By distinguishing pointwise, horizon-wise, and temporal disagreement, we revealed differences in direction, shape, timing, and other temporal features that conventional multiplicity metrics cannot capture. Using this framework, we showed that near-optimal performance, and even tight horizon-wise agreement, provide limited control over the trajectories competing models may produce. Experiments across forecasting architectures and datasets confirmed that these differences arise in practice and can affect downstream decisions. Our framework therefore provides practitioners with a means to assess temporal differences among models that standard accuracy measures deem interchangeable.

Several directions follow. First, time-series attributions and counterfactuals are themselves temporal objects \cite{Tonekaboni2020,Ates2021}; applying our framework to them could reveal a temporal dimension of explanatory multiplicity \cite{Dong2020,Laberge2023}. Second, extending the framework to quantile and probabilistic forecasts would enable the study of distributional disagreement. For time-series foundation models \cite{Woo2024,Ansari2025}, competing forecasts may also result from different contexts, prompts, or covariates, raising the question of whether Rashomon sets should encompass such conditioning choices, as a temporal counterpart to prompt multiplicity in language models \cite{Sclar2023,Ganesh2026}. Finally, combining deterministic Rashomon-set thresholds with model confidence sets \cite{Hansen2011} could account for uncertainty in model performance when estimating temporal multiplicity.

\subsection*{Acknowledgements}
Toni was partially funded by J.P. Morgan and by the Royal Academy of Engineering under the Research Chairs and Senior Research Fellowships scheme. Toni was partially funded by the European Research Council (ERC) under the European Union’s Horizon 2020 Research and Innovation Programme (Grant Agreement No. 101020934). 

\subsection*{Disclaimer}
This paper was prepared for informational purposes by the Artificial Intelligence Research group of JPMorgan Chase \& Co. and its affiliates (''JP Morgan'') and is not a product of the Research Department of JP Morgan. JP Morgan makes no representation and warranty whatsoever and disclaims all liability, for the completeness, accuracy or reliability of the information contained herein. This document is not intended as investment research or investment advice, or a recommendation, offer or solicitation for the purchase or sale of any security, financial instrument, financial product or service, or to be used in any way for evaluating the merits of participating in any transaction, and shall not constitute a solicitation under any jurisdiction or to any person, if such solicitation under such jurisdiction or to such person would be unlawful.

Any views or opinions expressed herein are solely those of the authors listed.

\textcopyright{} 2026 JP Morgan Chase \& Co. All rights reserved.

\clearpage

\bibliography{zotero}

\begin{thebibliography}{76}
\providecommand{\natexlab}[1]{#1}
\providecommand{\url}[1]{\texttt{#1}}
\expandafter\ifx\csname urlstyle\endcsname\relax
  \providecommand{\doi}[1]{doi: #1}\else
  \providecommand{\doi}{doi: \begingroup \urlstyle{rm}\Url}\fi

\bibitem[AEMO()]{AEMO}
AEMO.
\newblock Aggregated price and demand data from the {{Australian Energy Market Operator}}.

\bibitem[Aksu et~al.(2024)Aksu, Woo, Liu, Liu, Liu, Savarese, Xiong, and Sahoo]{Aksu2024}
Taha Aksu, Gerald Woo, Juncheng Liu, Xu~Liu, Chenghao Liu, Silvio Savarese, Caiming Xiong, and Doyen Sahoo.
\newblock {{GIFT-Eval}}: {{A Benchmark For General Time Series Forecasting Model Evaluation}}.
\newblock In \emph{{{NeurIPS TSALM Workshop}}}, 2024.

\bibitem[Alexandrov et~al.(2020)Alexandrov, Benidis, {Bohlke-Schneider}, Flunkert, Gasthaus, Januschowski, Maddix, Rangapuram, Salinas, Schulz, Stella, T{\"u}rkmen, and Wang]{Alexandrov2020}
Alexander Alexandrov, Konstantinos Benidis, Michael {Bohlke-Schneider}, Valentin Flunkert, Jan Gasthaus, Tim Januschowski, Danielle~C. Maddix, Syama Rangapuram, David Salinas, Jasper Schulz, Lorenzo Stella, Ali~Caner T{\"u}rkmen, and Yuyang Wang.
\newblock {{GluonTS}}: {{Probabilistic}} and {{Neural Time Series Modeling}} in {{Python}}.
\newblock \emph{Journal of Machine Learning Research}, 21\penalty0 (116):\penalty0 1--6, 2020.

\bibitem[Ansari et~al.(2025)Ansari, Shchur, K{\"u}ken, Auer, Han, Mercado, Rangapuram, Shen, Stella, Zhang, Goswami, Kapoor, Maddix, Guerron, Hu, Yin, Erickson, Desai, Wang, Rangwala, Karypis, Wang, and {Bohlke-Schneider}]{Ansari2025}
Abdul~Fatir Ansari, Oleksandr Shchur, Jaris K{\"u}ken, Andreas Auer, Boran Han, Pedro Mercado, Syama~Sundar Rangapuram, Huibin Shen, Lorenzo Stella, Xiyuan Zhang, Mononito Goswami, Shubham Kapoor, Danielle~C. Maddix, Pablo Guerron, Tony Hu, Junming Yin, Nick Erickson, Prateek~Mutalik Desai, Hao Wang, Huzefa Rangwala, George Karypis, Yuyang Wang, and Michael {Bohlke-Schneider}.
\newblock Chronos-2: {{From Univariate}} to {{Universal Forecasting}}, 2025.

\bibitem[Ates et~al.(2021)Ates, Aksar, Leung, and Coskun]{Ates2021}
Emre Ates, Burak Aksar, Vitus~J. Leung, and Ayse~K. Coskun.
\newblock Counterfactual {{Explanations}} for {{Multivariate Time Series}}.
\newblock In \emph{2021 {{International Conference}} on {{Applied Artificial Intelligence}} ({{ICAPAI}})}, pp.\  1--8, 2021.

\bibitem[Bai et~al.(2018)Bai, Kolter, and Koltun]{Bai2018}
Shaojie Bai, J.~Zico Kolter, and Vladlen Koltun.
\newblock An {{Empirical Evaluation}} of {{Generic Convolutional}} and {{Recurrent Networks}} for {{Sequence Modeling}}, 2018.

\bibitem[Beck et~al.(2024)Beck, P{\"o}ppel, Spanring, Auer, Prudnikova, Kopp, Klambauer, Brandstetter, and Hochreiter]{Beck2024}
Maximilian Beck, Korbinian P{\"o}ppel, Markus Spanring, Andreas Auer, Oleksandra Prudnikova, Michael Kopp, G{\"u}nter Klambauer, Johannes Brandstetter, and Sepp Hochreiter.
\newblock {{xLSTM}}: {{Extended Long Short-Term Memory}}.
\newblock In \emph{Advances in {{Neural Information Processing Systems}}}, volume~37, pp.\  107547--107603, 2024.

\bibitem[Benidis et~al.(2022)Benidis, Rangapuram, Flunkert, Wang, Maddix, Turkmen, Gasthaus, {Bohlke-Schneider}, Salinas, Stella, Aubet, Callot, and Januschowski]{Benidis2022}
Konstantinos Benidis, Syama~Sundar Rangapuram, Valentin Flunkert, Yuyang Wang, Danielle Maddix, Caner Turkmen, Jan Gasthaus, Michael {Bohlke-Schneider}, David Salinas, Lorenzo Stella, Fran{\c c}ois-Xavier Aubet, Laurent Callot, and Tim Januschowski.
\newblock Deep {{Learning}} for {{Time Series Forecasting}}: {{Tutorial}} and {{Literature Survey}}.
\newblock \emph{ACM Computing Surveys}, 55\penalty0 (6):\penalty0 121:1--121:36, 2022.

\bibitem[Berry et~al.(2024)Berry, Cevik, and Ozyegen]{Berry2024}
Sean Berry, Mucahit Cevik, and Ozan Ozyegen.
\newblock Interpreting time series forecasting models using model classreliance.
\newblock \emph{Proceedings of the Canadian Conference on Artificial Intelligence}, 2024.

\bibitem[Black et~al.(2022)Black, Raghavan, and Barocas]{Black2022b}
Emily Black, Manish Raghavan, and Solon Barocas.
\newblock Model {{Multiplicity}}: {{Opportunities}}, {{Concerns}}, and {{Solutions}}.
\newblock In \emph{2022 {{ACM Conference}} on {{Fairness Accountability}} and {{Transparency}}}, pp.\  850--863, 2022.

\bibitem[Breiman(2001)]{Breiman2001}
Leo Breiman.
\newblock Statistical {{Modeling}}: {{The Two Cultures}}.
\newblock \emph{Statistical Science}, 16\penalty0 (3):\penalty0 199--215, 2001.

\bibitem[Brigato et~al.(2026)Brigato, Morand, Str{\o}mmen, Panagiotou, Schmidt, and Mougiakakou]{Brigato2026}
Lorenzo Brigato, Rafael Morand, Knut~Joar Str{\o}mmen, Maria Panagiotou, Markus Schmidt, and Stavroula Mougiakakou.
\newblock There are no {{Champions}} in {{Long-Term Time Series Forecasting}}.
\newblock \emph{TMLR}, 2026.

\bibitem[Brouwer \& Haemers(2012)Brouwer and Haemers]{Brouwer2012}
Andries~E. Brouwer and Willem~H. Haemers.
\newblock \emph{Spectra of {{Graphs}}}.
\newblock Universitext. 2012.

\bibitem[Challu et~al.(2023)Challu, Olivares, Oreshkin, Ramirez, Canseco, and Dubrawski]{Challu2023}
Cristian Challu, Kin~G. Olivares, Boris~N. Oreshkin, Federico~Garza Ramirez, Max~Mergenthaler Canseco, and Artur Dubrawski.
\newblock {{NHITS}}: {{Neural Hierarchical Interpolation}} for {{Time Series Forecasting}}.
\newblock \emph{Proceedings of the AAAI Conference on Artificial Intelligence}, 37\penalty0 (6):\penalty0 6989--6997, 2023.

\bibitem[Chang et~al.(2017)Chang, Zhang, Han, Yu, Guo, Tan, Cui, Witbrock, {Hasegawa-Johnson}, and Huang]{Chang2017}
Shiyu Chang, Yang Zhang, Wei Han, Mo~Yu, Xiaoxiao Guo, Wei Tan, Xiaodong Cui, Michael Witbrock, Mark~A {Hasegawa-Johnson}, and Thomas~S Huang.
\newblock Dilated {{Recurrent Neural Networks}}.
\newblock In \emph{Advances in {{Neural Information Processing Systems}}}, volume~30, 2017.

\bibitem[Chen et~al.(2023)Chen, Li, Yoder, Arik, and Pfister]{Chen2023a}
Si-An Chen, Chun-Liang Li, Nate Yoder, Sercan~O. Arik, and Tomas Pfister.
\newblock {{TSMixer}}: {{An All-MLP Architecture}} for {{Time Series Forecasting}}.
\newblock \emph{TMLR}, 2023/09, 2023.

\bibitem[Chen et~al.(2026)Chen, Jin, Huang, and Feng]{Chen2026}
Xinyang Chen, Huidong Jin, Yu~Huang, and Zaiwen Feng.
\newblock {{XLinear}}: {{A Lightweight}} and {{Accurate MLP-Based Model}} for {{Long-Term Time Series Forecasting}} with {{Exogenous Inputs}}.
\newblock \emph{Proceedings of the AAAI Conference on Artificial Intelligence}, 40\penalty0 (24):\penalty0 20325--20335, 2026.

\bibitem[Cho et~al.(2014)Cho, {van Merrienboer}, Gulcehre, Bahdanau, Bougares, Schwenk, and Bengio]{Cho2014}
Kyunghyun Cho, Bart {van Merrienboer}, Caglar Gulcehre, Dzmitry Bahdanau, Fethi Bougares, Holger Schwenk, and Yoshua Bengio.
\newblock Learning {{Phrase Representations}} using {{RNN Encoder-Decoder}} for {{Statistical Machine Translation}}.
\newblock In \emph{{{EMNLP}}}, 2014.

\bibitem[Cortesi et~al.(2026)Cortesi, Iannone, Crippa, Poggio, and Beneventano]{Cortesi2026}
Federico~Vittorio Cortesi, Giuseppe Iannone, Giulia Crippa, Tomaso Poggio, and Pierfrancesco Beneventano.
\newblock Same {{Error}}, {{Different Function}}: {{The Optimizer}} as an {{Implicit Prior}} in {{Financial Time Series}}, 2026.

\bibitem[Crone(2008)]{Crone2008}
Sven~F. Crone.
\newblock 2008 {{Time Series Forecasting Competition}} for {{Computational Intelligence}} ({{NN5}}), 2008.

\bibitem[D'Amour et~al.(2022)D'Amour, Heller, Moldovan, Adlam, Alipanahi, Beutel, Chen, Deaton, Eisenstein, Hoffman, Hormozdiari, Houlsby, Hou, Jerfel, Karthikesalingam, Lucic, Ma, McLean, Mincu, Mitani, Montanari, Nado, Natarajan, Nielson, Osborne, Raman, Ramasamy, Sayres, Schrouff, Seneviratne, Sequeira, Suresh, Veitch, Vladymyrov, Wang, Webster, Yadlowsky, Yun, Zhai, and Sculley]{DAmour2022}
Alexander D'Amour, Katherine Heller, Dan Moldovan, Ben Adlam, Babak Alipanahi, Alex Beutel, Christina Chen, Jonathan Deaton, Jacob Eisenstein, Matthew~D. Hoffman, Farhad Hormozdiari, Neil Houlsby, Shaobo Hou, Ghassen Jerfel, Alan Karthikesalingam, Mario Lucic, Yian Ma, Cory McLean, Diana Mincu, Akinori Mitani, Andrea Montanari, Zachary Nado, Vivek Natarajan, Christopher Nielson, Thomas~F. Osborne, Rajiv Raman, Kim Ramasamy, Rory Sayres, Jessica Schrouff, Martin Seneviratne, Shannon Sequeira, Harini Suresh, Victor Veitch, Max Vladymyrov, Xuezhi Wang, Kellie Webster, Steve Yadlowsky, Taedong Yun, Xiaohua Zhai, and D.~Sculley.
\newblock Underspecification {{Presents Challenges}} for {{Credibility}} in {{Modern Machine Learning}}.
\newblock \emph{Journal of Machine Learning Research}, 23\penalty0 (226):\penalty0 1--61, 2022.

\bibitem[Das et~al.(2023)Das, Kong, Leach, Mathur, Sen, and Yu]{Das2023}
Abhimanyu Das, Weihao Kong, Andrew Leach, Shaan~K. Mathur, Rajat Sen, and Rose Yu.
\newblock Long-term {{Forecasting}} with {{TiDE}}: {{Time-series Dense Encoder}}.
\newblock \emph{Transactions on Machine Learning Research}, 2023.

\bibitem[Dong \& Rudin(2020)Dong and Rudin]{Dong2020}
Jiayun Dong and Cynthia Rudin.
\newblock Exploring the cloud of variable importance for the set of all good models.
\newblock \emph{Nature Machine Intelligence}, 2\penalty0 (12):\penalty0 810--824, 2020.

\bibitem[Donnelly et~al.(2023)Donnelly, Katta, Rudin, and Browne]{Donnelly2023}
Jon Donnelly, Srikar Katta, Cynthia Rudin, and Edward~P Browne.
\newblock The {{Rashomon Importance Distribution}}: {{Getting RID}} of {{Unstable}}, {{Single Model-based Variable Importance}}.
\newblock In \emph{Advances in {{Neural Information Processing Systems}} ({{NeurIPS}} 2023)}, 2023.

\bibitem[Du et~al.(2025)Du, Ngo, and Wu]{Du2025}
Ally~Yalei Du, Dung~Daniel Ngo, and Zhiwei~Steven Wu.
\newblock Reconciling {{Model Multiplicity}} for {{Downstream Decision Making}}.
\newblock In \emph{International {{Conference}} on {{Learning Representations}} ({{ICLR}} 2025)}, 2025.

\bibitem[EIA()]{Gasoline}
U.S. EIA.
\newblock Weekly {{U}}.{{S}}. {{Product Supplied}} of {{Finished Motor Gasoline}} ({{Thousand Barrels}} per {{Day}}).

\bibitem[Fisher et~al.(2019)Fisher, Rudin, and Dominici]{Fisher2019}
Aaron Fisher, Cynthia Rudin, and Francesca Dominici.
\newblock All {{Models}} are {{Wrong}}, but {{Many}} are {{Useful}}: {{Learning}} a {{Variable}}'s {{Importance}} by {{Studying}} an {{Entire Class}} of {{Prediction Models Simultaneously}}.
\newblock \emph{Journal of Machine Learning Research}, 2019.

\bibitem[Fosten \& Gutknecht(2021)Fosten and Gutknecht]{Fosten2021}
Jack Fosten and Daniel Gutknecht.
\newblock Horizon confidence sets.
\newblock \emph{Empirical Economics}, 61\penalty0 (2):\penalty0 667--692, 2021.

\bibitem[Ganesh et~al.(2025)Ganesh, Taik, and Farnadi]{Ganesh2025}
Prakhar Ganesh, Afaf Taik, and Golnoosh Farnadi.
\newblock Systemizing {{Multiplicity}}: {{The Curious Case}} of {{Arbitrariness}} in {{Machine Learning}}.
\newblock In \emph{Proceedings of the {{AAAI}}/{{ACM Conference}} on {{AI}}, {{Ethics}}, and {{Society}} ({{AIES}} 2025)}, 2025.

\bibitem[Ganesh et~al.(2026)Ganesh, Shokri, and Farnadi]{Ganesh2026}
Prakhar Ganesh, Reza Shokri, and Golnoosh Farnadi.
\newblock Rethinking {{Hallucinations}}: {{Correctness}}, {{Consistency}}, and {{Prompt Multiplicity}}.
\newblock In \emph{Proceedings of the 19th {{Conference}} of the {{European Chapter}} of the {{Association}} for {{Computational Linguistics}} ({{Volume}} 1: {{Long Papers}})}, pp.\  6959--6978, 2026.

\bibitem[Godahewa et~al.(2021)Godahewa, Bergmeir, Webb, Hyndman, and {Montero-Manso}]{Godahewa2021}
Rakshitha Godahewa, Christoph Bergmeir, Geoffrey~I. Webb, Rob~J. Hyndman, and Pablo {Montero-Manso}.
\newblock Monash {{Time Series Forecasting Archive}}.
\newblock In \emph{{{NeurIPS}}}, 2021.

\bibitem[Godahewa et~al.(2025)Godahewa, Bergmeir, Erkin~Baz, Zhu, Song, Garc{\'i}a, and Benavides]{Godahewa2025}
Rakshitha Godahewa, Christoph Bergmeir, Zeynep Erkin~Baz, Chengjun Zhu, Zhangdi Song, Salvador Garc{\'i}a, and Dario Benavides.
\newblock On forecast stability.
\newblock \emph{International Journal of Forecasting}, 41\penalty0 (4):\penalty0 1539--1558, 2025.

\bibitem[Han et~al.(2024)Han, Chen, Ye, and Zhan]{Han2024}
Lu~Han, Xu-Yang Chen, Han-Jia Ye, and De-Chuan Zhan.
\newblock {{SOFTS}}: {{Efficient Multivariate Time Series Forecasting}} with {{Series-Core Fusion}}.
\newblock In \emph{{{NeurIPS}}}, 2024.

\bibitem[Hansen et~al.(2011)Hansen, Lunde, and Nason]{Hansen2011}
Peter~R. Hansen, Asger Lunde, and James~M. Nason.
\newblock The {{Model Confidence Set}}.
\newblock \emph{Econometrica}, 79\penalty0 (2):\penalty0 453--497, 2011.

\bibitem[Heljakka et~al.(2023)Heljakka, Trapp, Kannala, and Solin]{Heljakka2023}
Ari Heljakka, Martin Trapp, Juho Kannala, and Arno Solin.
\newblock Disentangling {{Model Multiplicity}} in {{Deep Learning}}, 2023.

\bibitem[Herzen et~al.(2022)Herzen, L{\"a}ssig, Piazzetta, Neuer, Tafti, Raille, Van~Pottelbergh, Pasieka, Skrodzki, Huguenin, Dumonal, Ko{\'s}cisz, Bader, Gusset, Benheddi, Williamson, Kosinski, Petrik, and Grosch]{Herzen2022}
Julien Herzen, Francesco L{\"a}ssig, Samuele~Giuliano Piazzetta, Thomas Neuer, L{\'e}o Tafti, Guillaume Raille, Tomas Van~Pottelbergh, Marek Pasieka, Andrzej Skrodzki, Nicolas Huguenin, Maxime Dumonal, Jan Ko{\'s}cisz, Dennis Bader, Fr{\'e}d{\'e}rick Gusset, Mounir Benheddi, Camila Williamson, Michal Kosinski, Matej Petrik, and Ga{\"e}l Grosch.
\newblock Darts: {{User-Friendly Modern Machine Learning}} for {{Time Series}}.
\newblock \emph{Journal of Machine Learning Research}, 2022.

\bibitem[Hibbeln et~al.(2026)Hibbeln, Kopp, and Urban]{Hibbeln2026}
Martin~T. Hibbeln, Raphael~M. Kopp, and Noah Urban.
\newblock Predictive multiplicity, procedural multiplicity, and heterogeneous machine learning ensembles in recovery rate forecasting.
\newblock \emph{Journal of Financial Stability}, 83:\penalty0 101510, 2026.

\bibitem[Hornik(1991)]{Hornik1991}
Kurt Hornik.
\newblock Approximation capabilities of multilayer feedforward networks.
\newblock \emph{Neural Networks}, 4\penalty0 (2):\penalty0 251--257, 1991.

\bibitem[Hsu et~al.(2026)Hsu, Chen, Zhong, and Semenova]{Hsu2026}
Ethan Hsu, Harry Chen, Chudi Zhong, and Lesia Semenova.
\newblock The {{Double-Edged Nature}} of the {{Rashomon Set}} for {{Trustworthy Machine Learning}}.
\newblock In \emph{{{ICML}}}, 2026.

\bibitem[Hyndman \& Koehler(2006)Hyndman and Koehler]{Hyndman2006}
Rob~J. Hyndman and Anne~B. Koehler.
\newblock Another look at measures of forecast accuracy.
\newblock \emph{International Journal of Forecasting}, 22\penalty0 (4):\penalty0 679--688, 2006.

\bibitem[Jhana(2019)]{Jhana2019}
Nicolas Jhana.
\newblock Hourly energy demand generation and weather, 2019.

\bibitem[Kale et~al.(2026)Kale, Vishwakarma, Diamond, Hedayatipour, and Rezaei]{Kale2026}
Gauri Kale, Rahul Vishwakarma, Holly Diamond, Ava Hedayatipour, and Amin Rezaei.
\newblock Horizon-{{Constrained Rashomon Sets}} for {{Chaotic Forecasting}}.
\newblock \emph{AIP Advances}, 16\penalty0 (4):\penalty0 045208, 2026.

\bibitem[Kolle(2025)]{Kolle2025}
Olaf Kolle.
\newblock Documentation of the {{Weather Station}} on {{Top}} of the {{Roof}} of the {{Institute Building}} of the {{Max-Planck-Institute}} for {{Biogeochemistry}}, 2025.

\bibitem[Laberge et~al.(2023)Laberge, Pequignot, Mathieu, Khomh, and Marchand]{Laberge2023}
Gabriel Laberge, Yann Pequignot, Alexandre Mathieu, Foutse Khomh, and Mario Marchand.
\newblock Partial {{Order}} in {{Chaos}}: {{Consensus}} on {{Feature Attributions}} in the {{Rashomon Set}}.
\newblock \emph{Journal of Machine Learning Research}, 24\penalty0 (364):\penalty0 1--50, 2023.

\bibitem[Lai et~al.(2018)Lai, Chang, Yang, and Liu]{Lai2018}
Guokun Lai, Wei-Cheng Chang, Yiming Yang, and Hanxiao Liu.
\newblock Modeling {{Long-}} and {{Short-Term Temporal Patterns}} with {{Deep Neural Networks}}.
\newblock In \emph{The 41st {{International ACM SIGIR Conference}} on {{Research}} \& {{Development}} in {{Information Retrieval}}}, {{SIGIR}} '18, pp.\  95--104, 2018.

\bibitem[Lim \& Zohren(2021)Lim and Zohren]{Lim2021}
Bryan Lim and Stefan Zohren.
\newblock Time {{Series Forecasting With Deep Learning}}: {{A Survey}}.
\newblock \emph{Philosophical Transactions of the Royal Society A: Mathematical, Physical and Engineering Sciences}, 379\penalty0 (2194):\penalty0 20200209, 2021.

\bibitem[Lim et~al.(2021)Lim, Ar{\i}k, Loeff, and Pfister]{Lim2021a}
Bryan Lim, Sercan~{\"O}. Ar{\i}k, Nicolas Loeff, and Tomas Pfister.
\newblock Temporal {{Fusion Transformers}} for interpretable multi-horizon time series forecasting.
\newblock \emph{International Journal of Forecasting}, 37\penalty0 (4):\penalty0 1748--1764, 2021.

\bibitem[Liu et~al.(2024)Liu, Hu, Zhang, Wu, Wang, Ma, and Long]{Liu2024a}
Yong Liu, Tengge Hu, Haoran Zhang, Haixu Wu, Shiyu Wang, Lintao Ma, and Mingsheng Long.
\newblock {{iTransformer}}: {{Inverted Transformers Are Effective}} for {{Time Series Forecasting}}.
\newblock In \emph{{{ICLR}}}, 2024.

\bibitem[Madhyastha \& Jain(2019)Madhyastha and Jain]{Madhyastha2019}
Pranava Madhyastha and Rishabh Jain.
\newblock On {{Model Stability}} as a {{Function}} of {{Random Seed}}.
\newblock In \emph{Proceedings of the 23rd {{Conference}} on {{Computational Natural Language Learning}} ({{CoNLL}})}, 2019.

\bibitem[Makridakis et~al.(2022)Makridakis, Spiliotis, and Assimakopoulos]{Makridakis2022}
Spyros Makridakis, Evangelos Spiliotis, and Vassilios Assimakopoulos.
\newblock M5 accuracy competition: {{Results}}, findings, and conclusions.
\newblock \emph{International Journal of Forecasting}, 38\penalty0 (4):\penalty0 1346--1364, 2022.

\bibitem[Marx et~al.(2020)Marx, Calmon, and Ustun]{Marx2020a}
Charles Marx, Flavio Calmon, and Berk Ustun.
\newblock Predictive {{Multiplicity}} in {{Classification}}.
\newblock In \emph{Proceedings of the 37th {{International Conference}} on {{Machine Learning}}}, pp.\  6765--6774, 2020.

\bibitem[McCracken \& Ng(2016)McCracken and Ng]{McCracken2016}
Michael~W. McCracken and Serena Ng.
\newblock {{FRED-MD}}: {{A Monthly Database}} for {{Macroeconomic Research}}.
\newblock \emph{Journal of Business \& Economic Statistics}, 34\penalty0 (4):\penalty0 574--589, 2016.

\bibitem[Monschang et~al.(2026)Monschang, Trede, and Wilfling]{Monschang2026}
Verena Monschang, Mark Trede, and Bernd Wilfling.
\newblock Multi-{{Horizon Uniform Superior Predictive Ability Revisited}}.
\newblock \emph{Journal of Business \& Economic Statistics}, 44\penalty0 (3):\penalty0 897--903, 2026.

\bibitem[Nie et~al.(2022)Nie, Nguyen, Sinthong, and Kalagnanam]{Nie2022}
Yuqi Nie, Nam~H. Nguyen, Phanwadee Sinthong, and Jayant Kalagnanam.
\newblock A {{Time Series}} is {{Worth}} 64 {{Words}}: {{Long-term Forecasting}} with {{Transformers}}.
\newblock In \emph{The {{Eleventh International Conference}} on {{Learning Representations}}}, 2022.

\bibitem[Olivares et~al.(2022)Olivares, Challu, Garza, Canseco, and Dubrawski]{Olivares2022}
Kin~G. Olivares, Cristian Challu, Azul Garza, Max~Mergenthaler Canseco, and Artur Dubrawski.
\newblock {{NeuralForecast}}: {{User}} friendly state-of-the-art neural forecasting models, 2022.

\bibitem[Olivares et~al.(2023)Olivares, Challu, Marcjasz, Weron, and Dubrawski]{Olivares2023}
Kin~G. Olivares, Cristian Challu, Grzegorz Marcjasz, Rafa{\l} Weron, and Artur Dubrawski.
\newblock Neural basis expansion analysis with exogenous variables: {{Forecasting}} electricity prices with {{NBEATSx}}.
\newblock \emph{International Journal of Forecasting}, 39\penalty0 (2):\penalty0 884--900, 2023.

\bibitem[Parikh(2026)]{Parikh2026}
Harsh Parikh.
\newblock Why are there many equally good models? {{An Anatomy}} of the {{Rashomon Effect}}, 2026.

\bibitem[Quaedvlieg(2021)]{Quaedvlieg2021}
Rogier Quaedvlieg.
\newblock Multi-{{Horizon Forecast Comparison}}.
\newblock \emph{Journal of Business \& Economic Statistics}, 39\penalty0 (1):\penalty0 40--53, 2021.

\bibitem[Rudin et~al.(2024)Rudin, Zhong, Semenova, Seltzer, Parr, Liu, Katta, Donnelly, Chen, and Boner]{Rudin2024}
Cynthia Rudin, Chudi Zhong, Lesia Semenova, Margo Seltzer, Ronald Parr, Jiachang Liu, Srikar Katta, Jon Donnelly, Harry Chen, and Zachery Boner.
\newblock Amazing {{Things Come From Having Many Good Models}}.
\newblock In \emph{Proceedings of the 41st {{International Conference}} on {{Machine Learning}} ({{ICML}} 2024)}, 2024.

\bibitem[Sclar et~al.(2023)Sclar, Choi, Tsvetkov, and Suhr]{Sclar2023}
Melanie Sclar, Yejin Choi, Yulia Tsvetkov, and Alane Suhr.
\newblock Quantifying {{Language Models}}' {{Sensitivity}} to {{Spurious Features}} in {{Prompt Design}} or: {{How I}} learned to start worrying about prompt formatting.
\newblock In \emph{The {{Twelfth International Conference}} on {{Learning Representations}}}, 2023.

\bibitem[Semenova et~al.(2022)Semenova, Rudin, and Parr]{Semenova2022}
Lesia Semenova, Cynthia Rudin, and Ronald Parr.
\newblock On the {{Existence}} of {{Simpler Machine Learning Models}}.
\newblock In \emph{2022 {{ACM Conference}} on {{Fairness}}, {{Accountability}}, and {{Transparency}}}, pp.\  1827--1858, 2022.

\bibitem[Sprangers et~al.(2023)Sprangers, Schelter, and {de Rijke}]{Sprangers2023}
Olivier Sprangers, Sebastian Schelter, and Maarten {de Rijke}.
\newblock Parameter-efficient deep probabilistic forecasting.
\newblock \emph{International Journal of Forecasting}, 39\penalty0 (1):\penalty0 332--345, 2023.

\bibitem[{Stadt-Z{\"u}rich}()]{SZurich}
{Stadt-Z{\"u}rich}.
\newblock Viertelstundenwerte des {{Stromverbrauchs}} in den {{Netzebenen}} 5 und 7 in der {{Stadt Z\"urich}}.

\bibitem[Teney et~al.(2022)Teney, Peyrard, and Abbasnejad]{Teney2022}
Damien Teney, Maxime Peyrard, and Ehsan Abbasnejad.
\newblock Predicting {{Is Not Understanding}}: {{Recognizing}} and {{Addressing Underspecification}} in {{Machine Learning}}.
\newblock In \emph{Computer {{Vision}} -- {{ECCV}} 2022}, volume 13683, pp.\  458--476. 2022.

\bibitem[Tonekaboni et~al.(2020)Tonekaboni, Joshi, Campbell, Duvenaud, and Goldenberg]{Tonekaboni2020}
Sana Tonekaboni, Shalmali Joshi, Kieran Campbell, David~K Duvenaud, and Anna Goldenberg.
\newblock What went wrong and when? {{Instance-wise}} feature importance for time-series black-box models.
\newblock In \emph{Advances in {{Neural Information Processing Systems}}}, volume~33, pp.\  799--809, 2020.

\bibitem[Toner \& Darlow(2024)Toner and Darlow]{Toner2024}
William Toner and Luke~Nicholas Darlow.
\newblock An {{Analysis}} of {{Linear Time Series Forecasting Models}}.
\newblock In \emph{Proceedings of the 41st {{International Conference}} on {{Machine Learning}}}, pp.\  48404--48427, 2024.

\bibitem[Van~Belle et~al.(2023)Van~Belle, Crevits, and Verbeke]{VanBelle2023}
Jente Van~Belle, Ruben Crevits, and Wouter Verbeke.
\newblock Improving forecast stability using deep learning.
\newblock \emph{International Journal of Forecasting}, 39\penalty0 (3):\penalty0 1333--1350, 2023.

\bibitem[Wang et~al.(2024{\natexlab{a}})Wang, Wu, Shi, Hu, Luo, Ma, Zhang, and Zhou]{Wang2024a}
Shiyu Wang, Haixu Wu, Xiaoming Shi, Tengge Hu, Huakun Luo, Lintao Ma, James~Y Zhang, and Jun Zhou.
\newblock {{TimeMixer}}: {{Decomposable Multiscale Mixing}} for {{Time Series Forecasting}}.
\newblock In \emph{{{ICLR}}}, 2024{\natexlab{a}}.

\bibitem[Wang et~al.(2024{\natexlab{b}})Wang, Wu, Dong, Qin, Zhang, Liu, Qiu, Wang, and Long]{Wang2024c}
Yuxuan Wang, Haixu Wu, Jiaxiang Dong, Guo Qin, Haoran Zhang, Yong Liu, Yunzhong Qiu, Jianmin Wang, and Mingsheng Long.
\newblock {{TimeXer}}: {{Empowering Transformers}} for {{Time Series Forecasting}} with {{Exogenous Variables}}.
\newblock In \emph{{{NeurIPS}}}, 2024{\natexlab{b}}.

\bibitem[{Watson-Daniels} et~al.(2023){Watson-Daniels}, Parkes, and Ustun]{Watson-Daniels2023}
Jamelle {Watson-Daniels}, David~C. Parkes, and Berk Ustun.
\newblock Predictive {{Multiplicity}} in {{Probabilistic Classification}}.
\newblock In \emph{Proceedings of the {{AAAI Conference}} on {{Artificial Intelligence}}}, 2023.

\bibitem[Woo et~al.(2024)Woo, Liu, Kumar, Xiong, Savarese, and Sahoo]{Woo2024}
Gerald Woo, Chenghao Liu, Akshat Kumar, Caiming Xiong, Silvio Savarese, and Doyen Sahoo.
\newblock Unified {{Training}} of {{Universal Time Series Forecasting Transformers}}.
\newblock In \emph{{{ICML}}}, 2024.

\bibitem[Wu et~al.(2021)Wu, Xu, Wang, and Long]{Wu2021b}
Haixu Wu, Jiehui Xu, Jianmin Wang, and Mingsheng Long.
\newblock Autoformer: {{Decomposition Transformers}} with {{Auto-Correlation}} for {{Long-Term Series Forecasting}}.
\newblock In \emph{Advances in {{Neural Information Processing Systems}}}, 2021.

\bibitem[Wu et~al.(2022)Wu, Hu, Liu, Zhou, Wang, and Long]{Wu2022a}
Haixu Wu, Tengge Hu, Yong Liu, Hang Zhou, Jianmin Wang, and Mingsheng Long.
\newblock {{TimesNet}}: {{Temporal 2D-Variation Modeling}} for {{General Time Series Analysis}}.
\newblock In \emph{The {{Eleventh International Conference}} on {{Learning Representations}}}, 2022.

\bibitem[Zeng et~al.(2023)Zeng, Chen, Zhang, and Xu]{Zeng2023}
Ailing Zeng, Muxi Chen, Lei Zhang, and Qiang Xu.
\newblock Are transformers effective for time series forecasting?
\newblock In \emph{Proceedings of the {{Thirty-Seventh AAAI Conference}} on {{Artificial Intelligence}} and {{Thirty-Fifth Conference}} on {{Innovative Applications}} of {{Artificial Intelligence}} and {{Thirteenth Symposium}} on {{Educational Advances}} in {{Artificial Intelligence}}}, volume~37 of \emph{{{AAAI}}'23/{{IAAI}}'23/{{EAAI}}'23}, pp.\  11121--11128, 2023.

\bibitem[Zhang et~al.(2017)Zhang, Bengio, Hardt, Recht, and Vinyals]{Zhang2017}
Chiyuan Zhang, Samy Bengio, Moritz Hardt, Benjamin Recht, and Oriol Vinyals.
\newblock Understanding deep learning requires rethinking generalization.
\newblock In \emph{International {{Conference}} on {{Learning Representations}}}, 2017.

\bibitem[Zhou et~al.(2021)Zhou, Zhang, Peng, Zhang, Li, Xiong, and Zhang]{Zhou2021b}
Haoyi Zhou, Shanghang Zhang, Jieqi Peng, Shuai Zhang, Jianxin Li, Hui Xiong, and Wancai Zhang.
\newblock Informer: {{Beyond Efficient Transformer}} for {{Long Sequence Time-Series Forecasting}}.
\newblock \emph{Proceedings of the AAAI Conference on Artificial Intelligence}, 35\penalty0 (12):\penalty0 11106--11115, 2021.

\end{thebibliography}
\bibliographystyle{iclr2027_conference}

\clearpage

\clearpage

\appendix
\crefalias{section}{appendix}
\crefalias{subsection}{subappendix}
\crefalias{subsubsection}{subsubappendix}
\crefname{appendix}{Appendix}{Appendices}
\crefname{subappendix}{Appendix}{Appendices}
\crefname{subsubappendix}{Appendix}{Appendices}

\section{Experimental Setup}
\label{app:experimental-setup}

We provide in this appendix the details of our experimental setup.

\subsection{Datasets}
\label{app:datasets}

\begin{table}[h]
    \centering
    \caption{Provenance of the datasets used in our experiments.}
    \label{tab:datasets-provenance}
    \footnotesize
    \setlength{\tabcolsep}{3pt}
    \resizebox{\textwidth}{!}{\begin{tabular}{lllll}
\toprule
Dataset & Domain & Source & License & Distribution Channel \\
\midrule
UK ATM Withdrawals - Daily & finance & \citet{Crone2008,Godahewa2021} & CC BY 4.0 & GluonTS~\citep{Alexandrov2020} \\
UK ATM Withdrawals - Weekly & finance & \citet{Crone2008,Godahewa2021} & CC BY 4.0 & GluonTS~\citep{Alexandrov2020} \\
Bitcoin Price & finance & \citet{Godahewa2021} & CC BY 4.0 & GIFT-Eval~\citep{Aksu2024} \\
Australian Electricity Demand & energy & \citet{AEMO,Godahewa2021} & CC BY 4.0 & GluonTS~\citep{Alexandrov2020} \\
Zurich Electricity Consumption & energy & \citet{SZurich} & CC0 Public Domain & Darts~\citep{Herzen2022} \\
Hourly Energy Pricing - Spain & energy & \citet{Jhana2019} & CC0 Public Domain & Darts~\citep{Herzen2022} \\
Electricity Transformer Temperature 1 & industry & \citet{Zhou2021b} & CC BY-ND 4.0 & Darts~\citep{Herzen2022} \\
Currency Exchange Rates (USD) & finance & \citet{Lai2018} & -- & GluonTS~\citep{Alexandrov2020} \\
FRED-MD (US Macroeconomic Data) & economy & \citet{McCracken2016} & CC0 Public Domain & GluonTS~\citep{Alexandrov2020} \\
US Finished Motor Gasoline Supplied & energy & \citet{Gasoline} & CC0 Public Domain & Darts~\citep{Herzen2022} \\
Jena Weather & nature & \citet{Kolle2025,Wu2021b} & CC BY 4.0 & Darts~\citep{Herzen2022} \\
\bottomrule
\end{tabular}
}
\end{table}

\begin{table}[h]
    \centering
    \caption{Properties of the datasets used in our experiments.}
    \label{tab:datasets-properties}
    \footnotesize
    \setlength{\tabcolsep}{3pt}
    \resizebox{\textwidth}{!}{\begin{tabular}{lllllllll}
\toprule
Dataset & Frequency & Variables & Targets & Timestamps ($T$) & Input Size$^{\dagger}$ & Horizon & Seasonality & Poolable$^{*}$ \\
\midrule
UK ATM Withdrawals - Daily & D & 108 & 108 & 791 & 28 (4W) / 56 (8W) & 7 (1W) & 7 (1W) & Yes \\
UK ATM Withdrawals - Weekly & W & 108 & 108 & 113 & 16 (16W) / 26 (26W) & 4 (4W) & 52 (52W) & Only \\
Bitcoin Price & D & 18 & 1 & 2629 & 28 (4W) / 56 (8W) & 7 (1W) & -- & No \\
Australian Electricity Demand & 30min & 5 & 5 & 230736 & 336 (1W) & 48 (1D) & 48 (1D), 336 (1W), 17,520 (52W) & Yes \\
Zurich Electricity Consumption & 15min & 10 & 2 & 268705 & 672 (1W) & 96 (1D) & 96 (1D), 672 (1W), 35,040 (52W) & No \\
Hourly Energy Pricing - Spain & h & 20 & 1 & 35064 & 168 (1W) & 24 (1D) & 24 (1D), 168 (1W), 8,760 (52W) & No \\
Electricity Transformer Temperature 1 & h & 7 & 1 & 17420 & 168 (1W) / 192 (1.1W) & 24 (1D) & 24 (1D), 168 (1W) & No \\
Currency Exchange Rates (USD) & B & 8 & 8 & 6221 & 20 (4W) / 40 (8W) & 5 (1W) & -- & Yes \\
FRED-MD (US Macroeconomic Data) & ME & 107 & 3 & 728 & 12 (1Y) / 24 (2Y) & 3 (3M) & 12 (1Y) & No \\
US Finished Motor Gasoline Supplied & W-FRI & 1 & 1 & 1578 & 16 (16W) / 32 (32W) & 4 (4W) & 52 (52W) & -- \\
Jena Weather & 10min & 21 & 1 & 52704 & 1,008 (1W) & 144 (1D) & 144 (1D) & No \\
\bottomrule
\end{tabular}
}
\par\smallskip{\footnotesize\noindent $^{*}$Which representations a dataset's targets admit: \emph{No}, the targets are heterogeneous, so pooling is not meaningful; \emph{Yes}, they are homogeneous, so a pooled model is available alongside the per-target and joint ones; \emph{Only}, pooling is the one feasible representation, as for mixed panels of many unrelated series. A dash (--) marks a genuinely univariate dataset: a single series, with nothing to pool. $^{\dagger}$Every input size the sweep spans, in sweep order, where a dataset is swept at more than one; the runs of one dataset differ in nothing else.}

\end{table}

We use 11 publicly available forecasting datasets spanning finance, economics, energy, industry, and nature domains.
\Cref{tab:datasets-provenance} reports their provenance; \Cref{tab:datasets-properties} reports their properties; \Cref{tab:datasets-splits} reports the train/validation/test splits.
\begin{table}[h]
    \centering
    \caption{Train/validation/test splits of the datasets used in our experiments.}
    \label{tab:datasets-splits}
    \footnotesize
    \setlength{\tabcolsep}{3pt}
    \resizebox{\textwidth}{!}{\begin{tabular}{lllllllllll}
\toprule
Dataset & \multicolumn{2}{l}{Length} & \multicolumn{3}{l}{Samples} & \multicolumn{3}{l}{Samples$^{*}$} & \multicolumn{2}{l}{Forecasts} \\
\cmidrule(lr){2-3} \cmidrule(lr){4-6} \cmidrule(lr){7-9} \cmidrule(lr){10-11}
 & \textsc{Train} & \textsc{Val \& Test} & \textsc{Train} & \textsc{Train}$^{\dagger}$ & \textsc{Val \& Test} & \textsc{Train} & \textsc{Train}$^{\dagger}$ & \textsc{Val \& Test} & \textsc{Train} & \textsc{Val \& Test} \\
\midrule
UK ATM Withdrawals - Daily & 553 & 119 & 526 / 498 & 531 / 503 & 113 & 56808 / 53784 & 57348 / 54324 & 12204 & 56808 / 53784 & 12204 \\
UK ATM Withdrawals - Weekly & 79 & 17 & -- & -- & -- & 6912 / 5832 & 7128 / 6048 & 1512 & 6912 / 5832 & 1512 \\
Bitcoin Price & 1841 & 394 & 1814 / 1786 & 1819 / 1791 & 388 & -- & -- & -- & 1814 / 1786 & 388 \\
Australian Electricity Demand & 161516 & 34610 & 161181 & 161227 & 34563 & 805905 & 806135 & 172815 & 805905 & 172815 \\
Zurich Electricity Consumption & 188093 & 40306 & 187422 & 187516 & 40211 & -- & -- & -- & 374844 & 80422 \\
Hourly Energy Pricing - Spain & 24544 & 5260 & 24377 & 24399 & 5237 & -- & -- & -- & 24377 & 5237 \\
Electricity Transformer Temperature 1 & 12194 & 2613 & 12027 / 12003 & 12049 / 12025 & 2590 & -- & -- & -- & 12027 / 12003 & 2590 \\
Currency Exchange Rates (USD) & 4355 & 933 & 4336 / 4316 & 4339 / 4319 & 929 & 34688 / 34528 & 34712 / 34552 & 7432 & 34688 / 34528 & 7432 \\
FRED-MD (US Macroeconomic Data) & 510 & 109 & 499 / 487 & 500 / 488 & 107 & -- & -- & -- & 1497 / 1461 & 321 \\
US Finished Motor Gasoline Supplied & 1104 & 237 & 1089 / 1073 & 1091 / 1075 & 234 & -- & -- & -- & 1089 / 1073 & 234 \\
Jena Weather & 36892 & 7906 & 35885 & 36027 & 7763 & -- & -- & -- & 35885 & 7763 \\
\bottomrule
\end{tabular}
}
\par\smallskip{\footnotesize\noindent $^{*}$Samples of the pooled representation, windows $\times$ targets; the unmarked samples are those of the per-target and joint representations, one per rolling-origin window. A count is empty where its representation is infeasible: the pooled one where the dataset is not poolable, the unmarked one where pooling is its one feasible representation. $^{\dagger}$Training samples of the recurrent architectures, which unfold training windows of the input size plus 2 steps rather than the input size plus the horizon; every other count is shared by all architectures. A cell listing several counts holds one per input size the sweep spans, in the order of the properties table's input sizes.}

\end{table}

\paragraph{Pre-processing.}
Placeholder codes (sentinel values, fabricated calendar days such as the 31st of a 30-day month) are converted to missing values before any other step.
Three datasets whose series are strictly positive and span orders of magnitude (the currency exchange rates, the Bitcoin indicators, and the FRED-MD macroeconomic levels) are modeled in natural-log space, as is standard for those series in such domains; the nine FRED-MD series that take non-positive values (interest-rate spreads and non-borrowed reserves) are kept in raw levels. 

\paragraph{Missing Values Handling.} Missing values are handled without future leakage: leading gaps are trimmed, interior gaps are forward-filled, and an availability mask records which positions were originally missing.
The mask follows the data through the whole pipeline: imputed target positions do not contribute to the training loss, to the seasonal-naive scale of MASE and RMSSE, or to any evaluation metric.

\paragraph{Standardization.} Every architecture standardizes each input window by its own mean and standard deviation and inverts the transformation on its outputs, so all models see the same raw series. Refer to the model training section for more details on normalization.

\paragraph{Exogenous Variables.} Non-target variables are offered to models as historical exogenous inputs whenever the architecture allows it (see \cref{app:models}).

\paragraph{Seasonality.} When a dataset declares one or more seasonal periods $s$, 4 Fourier terms of those periods are passed to the model as future-known exogenous features whenever the architecture allows it:
$$
\sin\!\left(\frac{2\pi t}{s}\right),\quad
\cos\!\left(\frac{2\pi t}{s}\right),\quad
\sin\!\left(\frac{4\pi t}{s}\right),\quad
\cos\!\left(\frac{4\pi t}{s}\right)
$$

\paragraph{Dataset Splits.} Each dataset is split chronologically into a training, a validation, and a test portion, the validation portion immediately preceding the test tail and both holding out, by default, 15\% of the calendar (see \cref{tab:datasets-splits} for more details). 
Within each held-out portion the forecasts are rolling-origin forecasts at a dense stride of one timestamp, every origin producing one $H$-step trajectory, so that the multiplicity statistics of \cref{sec:hierarchy} are computed over every origin of the test portion.

\subsection{Model Training}
\label{app:models}

\paragraph{Hardware Setup.} To train and evaluate our models we used a Ray cluster of machines with one \emph{NVIDIA RTX PRO 6000 Blackwell Server Edition} GPU, an $8$-core CPU running at $3.2$ GHz, and $128$ GB of RAM. To run our experiments on the trained models we instead used a CPU-only machine with $64$ cores, running at $3.7$ GHz, and $256$ GB of RAM.

\paragraph{Model Architectures.}
To generate the empirical Rashomon set we train models spanning the main families of neural forecasting architectures: Transformer-based, MLP/mixer-based, convolutional, recurrent, and near-linear models.
\Cref{tab:architectures} lists the 19 architectures used in our experiments together with the references of the original proposals.
We use NeuralForecast \cite{Olivares2022} to train all such model architectures.

\begin{table}[h]
    \centering
    \caption{Neural forecasting architectures used to generate the empirical Rashomon set, grouped by model family.}
    \label{tab:architectures}
    \begin{tabular}{lll}
        \toprule
        Family & Model & Reference \\
        \midrule
        \multirow{4}{*}{Transformer}
        & PatchTST & \citet{Nie2022} \\
        & iTransformer & \citet{Liu2024a} \\
        & TimeXer & \citet{Wang2024c} \\
        & TFT & \citet{Lim2021a} \\
        \midrule
        \multirow{7}{*}{MLP / Mixer}
        & NHITS & \citet{Challu2023} \\ 
        & NBEATSx & \citet{Olivares2023} \\
        & TiDE & \citet{Das2023} \\
        & TimeMixer & \citet{Wang2024a} \\
        & TSMixer & \citet{Chen2023a} \\
        & TSMixerx & \citet{Chen2023a} \\
        & SOFTS & \citet{Han2024} \\
        \midrule
        \multirow{3}{*}{Convolutional}
        & TimesNet & \citet{Wu2022a} \\
        & BiTCN & \citet{Sprangers2023} \\
        & TCN & \citet{Bai2018} \\
        \midrule
        \multirow{3}{*}{Recurrent}
        & xLSTM & \citet{Beck2024} \\
        & GRU & \citet{Cho2014} \\
        & DilatedRNN & \citet{Chang2017} \\
        \midrule
        \multirow{2}{*}{Near-linear}
        & DLinear & \citet{Zeng2023} \\
        & xLinear & \citet{Chen2026} \\
        \bottomrule
    \end{tabular}
\end{table}
 
\paragraph{\emph{Fixed} Hyper-Parameters.}
The following hyper-parameters are fixed on all trained models:
\begin{itemize}
    \item \textbf{Loss:} MSE (mean squared error).
    \item \textbf{Optimizer:} Adam with no decay schedule.
    \item \textbf{Regularization:} every dropout rate set to zero.
    \item \textbf{Input scaling:} per-window standard scaling of the inputs.
    \item \textbf{Precision:} train-time operations in \texttt{bfloat16} mixed precision whenever possible (parameters and optimizer state stay in \texttt{fp32}); every evaluation---for every model and every checkpoint---is run in \texttt{fp32} so that no precision confound enters the comparison of models.
\end{itemize}

\paragraph{Hyper-Parameters \textit{Swept} to Generate the Full Set of Models.}
For every dataset the full set of trained models is the Cartesian product of the following hyper-parameters:
\begin{itemize}
    \item \emph{Model Architectures:} the 19 architectures of \cref{tab:architectures}
    \item \emph{Random Seeds:} ten random seeds ($0$ to $9$).
    \item \emph{Learning Rate:} $10^{-3}$ and $3 \cdot 10^{-4}$
    \item \emph{Batch Size:} $32$ and $64$
    \item \emph{Input Size:} four and eight times the forecast horizon, where the dataset admits both. \cref{tab:datasets-properties} reports the corresponding input sizes, while \cref{tab:datasets-splits} lists the training samples each input size yields.
    \item \emph{Target:} \texttt{Univariate} and \texttt{Multivariate} $^{* \dagger}$
    \item \emph{Exogenous Variables:} \texttt{Disabled} and \texttt{Enabled} $^*$
    \item \emph{Pooled Targets:} \texttt{Unpooled} and \texttt{Pooled} targets $^*$ (as supported by each dataset, see \cref{tab:datasets-properties})
\end{itemize}
$^*$ Note that the combinations an architecture cannot realize (e.g.\ a purely univariate architecture in the multivariate representation) are pruned per (architecture, dataset) pair.

$\dagger$ Also note that when multi-target datasets are trained in univariate mode the training fans out into one model per target, which share one run's compute budget as described below. 

Finally, note that dataset–architecture combinations that exceeded the RAM or VRAM capacity of the hardware described above were excluded.

All other hyper-parameters that are not mentioned in these two paragraphs are left at their default value in NeuralForecast {3.1.9} \cite{Olivares2022}.

\paragraph{Randomness Control.} The random seed drives every source of randomness in a run: parameter initialization, the sampling of training origins, and any stochastic kernel. Deterministic kernels are requested wherever PyTorch offers them.

\paragraph{Compute Budget.}
The step budget of every (architecture, dataset, representation) cell is determined by the following rules.
\begin{itemize}
    \item \textbf{Wall-time budget:} a short profiling run measures the isolated training throughput (steps per second) on a dedicated \emph{NVIDIA RTX PRO 6000 Blackwell Server Edition} GPU, and the step budget is the number of steps that spends $900$ seconds ($15$ minutes) at that throughput.
    \item \textbf{Epoch cap:} the step budget is capped at 250 expected-coverage epochs, an epoch being the number of steps after which every training origin has been sampled once in expectation, i.e., $\lceil N_{\mathrm{origins}}/N_{\mathrm{batch}} \rceil$.
    \item \textbf{Per-target fan-out:} the step budget is divided by the number of per-target models a run fans out into, so that a per-target univariate run spends the same total compute as a joint or pooled run on the same dataset.
\end{itemize}

We estimate that training and evaluating our models required \(3{,}000\) GPU-hours, with an additional \(300\) hours of CPU-only compute used to run experiments.

\paragraph{Checkpoints.}
Along every run we save periodic snapshots of the weights, and the end-of-training weights as a final checkpoint.
The snapshot interval is $1000$ optimizer steps, shortened for slow models to at most $60$ seconds of isolated wall-time and to at most one expected-coverage epoch, whichever is smallest.

\paragraph{Checkpoint Selection.} Every checkpoint, including the final one, is scored with the above rolling-origin protocol. We keep the single checkpoint with the lowest validation \textbf{MASE} \cite{Hyndman2006}, aggregated over the whole horizon and over every validation origin.

\begin{table}[h]
    \centering
    \caption{Validation and test performance of the full set of trained models of each dataset: minimum and mean across all models of the horizon-averaged MAE, RMSE, MASE, and RMSSE, on the validation portion (left) and on the test portion (right).}
    \label{tab:performance-datasets}
    \footnotesize
    \setlength{\tabcolsep}{3pt}
    \resizebox{\textwidth}{!}{\begin{tabular}{llllllllllllllllll}
\toprule
Dataset & \multicolumn{1}{l}{} & \multicolumn{8}{l}{\textsc{Val}} & \multicolumn{8}{l}{\textsc{Test}} \\
\cmidrule(lr){3-10} \cmidrule(lr){11-18}
 & \multicolumn{1}{l}{} & \multicolumn{2}{l}{MAE} & \multicolumn{2}{l}{RMSE} & \multicolumn{2}{l}{MASE} & \multicolumn{2}{l}{RMSSE} & \multicolumn{2}{l}{MAE} & \multicolumn{2}{l}{RMSE} & \multicolumn{2}{l}{MASE} & \multicolumn{2}{l}{RMSSE} \\
\cmidrule(lr){3-4} \cmidrule(lr){5-6} \cmidrule(lr){7-8} \cmidrule(lr){9-10} \cmidrule(lr){11-12} \cmidrule(lr){13-14} \cmidrule(lr){15-16} \cmidrule(lr){17-18}
 & Models & Min & Mean & Min & Mean & Min & Mean & Min & Mean & Min & Mean & Min & Mean & Min & Mean & Min & Mean \\
\midrule
ATM 1D & 2080 & 4.23 & 4.68 & 6.88 & 7.56 & 1.05 & 1.16 & 1.09 & 1.20 & 3.16 & 3.45 & 4.95 & 5.23 & 0.73 & 0.80 & 0.71 & 0.76 \\
ATM 1W & 1520 & 15.4 & 18.0 & 22.3 & 25.9 & 0.93 & 1.11 & 1.05 & 1.26 & 12.9 & 13.9 & 18.2 & 19.5 & 0.75 & 0.81 & 0.80 & 0.86 \\
Bitcoin & 2800 & 0.06 & 0.07 & 0.08 & 0.10 & 2.67 & 3.26 & 2.58 & 3.03 & 0.05 & 0.06 & 0.07 & 0.09 & 2.35 & 3.01 & 2.16 & 2.76 \\
Electricity D. A (NSW) & 2160 & 171 & 206 & 274 & 323 & 0.38 & 0.45 & 0.41 & 0.48 & 171 & 206 & 274 & 320 & 0.38 & 0.46 & 0.41 & 0.48 \\
Electricity D. A (QLD) & 2160 & 100 & 118 & 153 & 176 & 0.41 & 0.48 & 0.40 & 0.46 & 114 & 141 & 174 & 205 & 0.47 & 0.58 & 0.46 & 0.54 \\
Electricity D. A (SA) & 2160 & 61.8 & 73.0 & 103 & 116 & 0.52 & 0.62 & 0.55 & 0.63 & 74.7 & 87.2 & 121 & 136 & 0.63 & 0.73 & 0.65 & 0.73 \\
Electricity D. A (TAS) & 2160 & 22.6 & 24.7 & 31.4 & 34.2 & 0.53 & 0.58 & 0.50 & 0.54 & 23.5 & 25.7 & 32.5 & 35.2 & 0.56 & 0.62 & 0.53 & 0.57 \\
Electricity D. A (VIC) & 2160 & 151 & 179 & 246 & 285 & 0.43 & 0.52 & 0.46 & 0.53 & 166 & 200 & 275 & 321 & 0.48 & 0.57 & 0.51 & 0.60 \\
Electricity D. Z (Value\_NE5) & 1680 & 524 & 864 & 745 & 1195 & 0.21 & 0.35 & 0.18 & 0.28 & 543 & 888 & 819 & 1244 & 0.23 & 0.37 & 0.20 & 0.31 \\
Electricity D. Z (Value\_NE7) & 1680 & 1019 & 1823 & 1485 & 2481 & 0.27 & 0.49 & 0.23 & 0.39 & 934 & 1734 & 1421 & 2400 & 0.26 & 0.48 & 0.23 & 0.39 \\
Electricity Price & 1400 & 4.88 & 5.65 & 6.88 & 7.90 & 0.82 & 0.95 & 0.81 & 0.93 & 2.65 & 3.12 & 3.63 & 4.17 & 0.43 & 0.51 & 0.42 & 0.48 \\
ETT 1H & 2800 & 1.14 & 1.27 & 1.50 & 1.64 & 0.48 & 0.53 & 0.46 & 0.50 & 1.27 & 1.46 & 1.75 & 1.96 & 0.57 & 0.65 & 0.57 & 0.64 \\
Exchange Rate (AUD) & 4320 & 0.01 & 0.01 & 0.02 & 0.02 & 2.78 & 3.35 & 2.86 & 3.40 & 0.01 & 0.01 & 0.01 & 0.01 & 1.66 & 2.03 & 1.50 & 1.84 \\
Exchange Rate (CAD) & 4320 & 0.01 & 0.01 & 0.01 & 0.02 & 3.61 & 4.35 & 2.89 & 3.47 & 0.01 & 0.01 & 0.01 & 0.01 & 1.70 & 2.04 & 1.39 & 1.74 \\
Exchange Rate (CHF) & 4320 & 0.01 & 0.01 & 0.01 & 0.01 & 1.56 & 1.90 & 1.51 & 1.82 & 0.01 & 0.01 & 0.01 & 0.01 & 1.51 & 1.89 & 1.55 & 1.91 \\
Exchange Rate (CNY) & 4320 & 0.00 & 0.00 & 0.00 & 0.00 & 1.83 & 3.77 & 0.21 & 0.43 & 0.00 & 0.00 & 0.00 & 0.00 & 2.52 & 4.75 & 0.28 & 0.60 \\
Exchange Rate (GBP) & 4320 & 0.01 & 0.01 & 0.01 & 0.02 & 2.21 & 2.65 & 2.16 & 2.56 & 0.01 & 0.01 & 0.01 & 0.01 & 1.31 & 1.59 & 1.13 & 1.40 \\
Exchange Rate (JPY) & 4320 & 0.01 & 0.01 & 0.01 & 0.02 & 1.90 & 2.28 & 1.78 & 2.08 & 0.01 & 0.01 & 0.01 & 0.01 & 1.22 & 1.51 & 1.11 & 1.36 \\
Exchange Rate (NZD) & 4320 & 0.01 & 0.02 & 0.02 & 0.02 & 2.98 & 3.59 & 2.90 & 3.43 & 0.01 & 0.01 & 0.01 & 0.01 & 1.75 & 2.14 & 1.55 & 1.89 \\
Exchange Rate (SGD) & 4320 & 0.00 & 0.01 & 0.01 & 0.01 & 1.99 & 2.43 & 1.62 & 1.98 & 0.00 & 0.00 & 0.01 & 0.01 & 1.77 & 2.13 & 1.34 & 1.65 \\
FRED-MD (CPIAUCSL) & 3240 & 0.00 & 0.01 & 0.01 & 0.01 & 0.09 & 0.13 & 0.12 & 0.16 & 0.00 & 0.00 & 0.00 & 0.00 & 0.06 & 0.08 & 0.06 & 0.09 \\
FRED-MD (INDPRO) & 3240 & 0.01 & 0.01 & 0.01 & 0.02 & 0.14 & 0.21 & 0.17 & 0.27 & 0.00 & 0.01 & 0.01 & 0.01 & 0.10 & 0.15 & 0.11 & 0.16 \\
FRED-MD (UNRATE) & 3240 & 0.02 & 0.03 & 0.03 & 0.04 & 0.19 & 0.28 & 0.18 & 0.28 & 0.02 & 0.03 & 0.03 & 0.04 & 0.17 & 0.22 & 0.15 & 0.22 \\
Gasoline & 1520 & 207 & 224 & 262 & 286 & 0.74 & 0.80 & 0.73 & 0.80 & 330 & 400 & 538 & 700 & 1.15 & 1.39 & 1.48 & 1.93 \\
Weather & 1240 & 8.22 & 9.14 & 12.0 & 13.3 & 0.95 & 1.06 & 0.88 & 0.98 & 8.93 & 10.4 & 12.3 & 14.3 & 1.02 & 1.19 & 0.89 & 1.03 \\
\bottomrule
\end{tabular}
}
\end{table}

\paragraph{Model Performance.}
\Cref{tab:performance-datasets} reports, for each dataset, the minimum and the mean of the validation and of the test performance across the whole set of trained models, under the four error metrics MAE, RMSE, MASE, and RMSSE \cite{Hyndman2006,Makridakis2022}, each averaged over the whole horizon and over every origin of the portion.

\subsection{Rashomon Set}
\label{app:rashomon-set}

\paragraph{Selection Rule.}
The empirical Rashomon set $\egR$ of a dataset is selected once again using the validation \emph{MASE} as our $\Ls(f)$.

\begin{table}[h]
    \centering
    \caption{Empirical Rashomon sets: forecast horizon $H$, the size of the full set of trained models $|\widehat{\gF}|$, Rashomon set size $|\egR|$, the Rashomon ratio, and the best and worst validation score of the set's members on the selection metric (MASE, in bold) and on the other reported metrics.}
    \label{tab:rashomon-set}
    \footnotesize
    \setlength{\tabcolsep}{3pt}
    \resizebox{\textwidth}{!}{\begin{tabular}{lllllllllllll}
\toprule
Dataset & \multicolumn{4}{l}{} & \multicolumn{2}{l}{\textbf{MASE}} & \multicolumn{2}{l}{MAE} & \multicolumn{2}{l}{RMSE} & \multicolumn{2}{l}{RMSSE} \\
\cmidrule(lr){6-7} \cmidrule(lr){8-9} \cmidrule(lr){10-11} \cmidrule(lr){12-13}
 & $H$ & $|\hat{\gF}|$ & $|\gR_\eps|$ & Ratio & Best & Worst & Best & Worst & Best & Worst & Best & Worst \\
\midrule
ATM 1D & 7 & 2080 & 88 & 0.04 & 1.05 & 1.10 & 4.23 & 4.44 & 6.88 & 7.51 & 1.09 & 1.19 \\
ATM 1W & 4 & 1520 & 81 & 0.05 & 0.93 & 0.98 & 15.4 & 16.2 & 22.3 & 24.5 & 1.05 & 1.16 \\
Bitcoin & 7 & 2800 & 173 & 0.06 & 2.67 & 2.81 & 0.06 & 0.06 & 0.08 & 0.09 & 2.58 & 2.81 \\
Electricity D. A (NSW) & 48 & 2160 & 135 & 0.06 & 0.38 & 0.40 & 171 & 179 & 274 & 298 & 0.41 & 0.44 \\
Electricity D. A (QLD) & 48 & 2160 & 332 & 0.15 & 0.41 & 0.43 & 100 & 105 & 153 & 168 & 0.40 & 0.44 \\
Electricity D. A (SA) & 48 & 2160 & 164 & 0.08 & 0.52 & 0.55 & 61.8 & 64.8 & 103 & 110 & 0.55 & 0.59 \\
Electricity D. A (TAS) & 48 & 2160 & 876 & 0.41 & 0.53 & 0.56 & 22.6 & 23.7 & 31.4 & 33.4 & 0.50 & 0.53 \\
Electricity D. A (VIC) & 48 & 2160 & 259 & 0.12 & 0.43 & 0.46 & 151 & 158 & 246 & 270 & 0.46 & 0.50 \\
Electricity D. Z (Value\_NE5) & 96 & 1680 & 107 & 0.06 & 0.21 & 0.22 & 524 & 550 & 745 & 827 & 0.18 & 0.20 \\
Electricity D. Z (Value\_NE7) & 96 & 1680 & 109 & 0.06 & 0.27 & 0.29 & 1019 & 1070 & 1485 & 1603 & 0.23 & 0.25 \\
Electricity Price & 24 & 1400 & 208 & 0.15 & 0.82 & 0.86 & 4.88 & 5.12 & 6.88 & 7.27 & 0.81 & 0.86 \\
ETT 1H & 24 & 2800 & 748 & 0.27 & 0.48 & 0.51 & 1.14 & 1.20 & 1.50 & 1.59 & 0.46 & 0.49 \\
Exchange Rate (AUD) & 5 & 4320 & 1435 & 0.33 & 2.78 & 2.92 & 0.01 & 0.01 & 0.02 & 0.02 & 2.86 & 3.10 \\
Exchange Rate (CAD) & 5 & 4320 & 1339 & 0.31 & 3.61 & 3.79 & 0.01 & 0.01 & 0.01 & 0.02 & 2.89 & 3.10 \\
Exchange Rate (CHF) & 5 & 4320 & 1326 & 0.31 & 1.56 & 1.64 & 0.01 & 0.01 & 0.01 & 0.01 & 1.51 & 1.60 \\
Exchange Rate (CNY) & 5 & 4320 & 100 & 0.02 & 1.83 & 1.92 & 0.00 & 0.00 & 0.00 & 0.00 & 0.21 & 0.23 \\
Exchange Rate (GBP) & 5 & 4320 & 1201 & 0.28 & 2.21 & 2.33 & 0.01 & 0.01 & 0.01 & 0.01 & 2.16 & 2.32 \\
Exchange Rate (JPY) & 5 & 4320 & 1334 & 0.31 & 1.90 & 2.00 & 0.01 & 0.01 & 0.01 & 0.01 & 1.78 & 1.90 \\
Exchange Rate (NZD) & 5 & 4320 & 1238 & 0.29 & 2.98 & 3.13 & 0.01 & 0.01 & 0.02 & 0.02 & 2.90 & 3.10 \\
Exchange Rate (SGD) & 5 & 4320 & 1208 & 0.28 & 1.99 & 2.09 & 0.00 & 0.00 & 0.01 & 0.01 & 1.62 & 1.83 \\
FRED-MD (CPIAUCSL) & 3 & 3240 & 128 & 0.04 & 0.09 & 0.09 & 0.00 & 0.00 & 0.01 & 0.01 & 0.12 & 0.13 \\
FRED-MD (INDPRO) & 3 & 3240 & 13 & 0.00 & 0.14 & 0.15 & 0.01 & 0.01 & 0.01 & 0.01 & 0.17 & 0.19 \\
FRED-MD (UNRATE) & 3 & 3240 & 1 & 0.00 & 0.19 & 0.19 & 0.02 & 0.02 & 0.03 & 0.03 & 0.18 & 0.18 \\
Gasoline & 4 & 1520 & 436 & 0.29 & 0.74 & 0.77 & 207 & 218 & 262 & 284 & 0.73 & 0.79 \\
Weather & 144 & 1240 & 153 & 0.12 & 0.95 & 1.00 & 8.22 & 8.63 & 12.0 & 13.0 & 0.88 & 0.96 \\
\midrule
Mean & -- & -- & -- & 0.14 & 0.99 & 1.03 & -- & -- & -- & -- & 0.95 & 1.03 \\
\bottomrule
\end{tabular}
}
\end{table}

Since the best achievable MASE differs from dataset to dataset, a common absolute tolerance would admit sets of very different quality across datasets; we therefore choose $\eps$ as a fixed fraction of the best validation score, so that the results remain comparable across datasets: $\eps = 5\% \cdot \Ls(f^\star)$, i.e., every model within 5\% of the best validation MASE is a member. 
For a dataset with several targets that admit per-target modeling, one set is selected and measured per target.

Every multiplicity statistic is then computed on the members' test forecasts, so that the selection and the reported disagreement never share data. 

\Cref{tab:rashomon-set} reports, for every dataset, several metrics on its Rashomon set including the size $|\egR|$, its Rashomon ratio, and the best and the worst validation score among the members of the set, on four metrics: MASE, MAE, RMSE, and RMSSE. 
The best MASE among the models of the Rashomon set is by construction the best of the full set of trained models, $\Ls(f^\star)$.

\subsection{Temporal Multiplicity}
\label{app:multiplicity-metrics}

\paragraph{Trajectory Scaling.}
Since the series differ from dataset to dataset (and, within a dataset, from series to series) in unit and in scale, a distance between two forecasts, or a threshold $\tau$ on such a distance, expressed in raw units would not be comparable across them; we therefore read every magnitude-like quantity in \emph{scale-free} units, so that the results remain comparable across series and datasets.

Unless stated otherwise, before any multiplicity statistic is computed, each prediction and each trajectory is divided by its series' in-sample seasonal-naive mean absolute error---the MASE denominator \cite{Hyndman2006}. The thresholds $\tau$ are also expressed in the same unit.

We note that ($\dt_{\mathrm{direction}}$, $\dt_{\mathrm{shape}}$, $\dt_{\mathrm{spectral}}$, $\dt_{\mathrm{timing}}$) are invariant to positive per-series rescaling and are unaffected.
For the other trajectory distance functions we note that scaling is a fixed positive constant per series, applied identically to every model and to the realized values: as such, every result of the paper holds verbatim in the rescaled space, while the statistics gain a unit-free, cross-dataset interpretation.

\paragraph{Numerical Implementation: Shape Distance.}
$\rho(\vu,\vv)$ is the Pearson correlation between the entries of $\vu$ and $\vv$, so that $\dt_{\mathrm{shape}}=\tfrac{1}{2}(1-\rho)$ ranges from $0$ (proportional trajectories) to $1$ (anti-proportional ones).
The Pearson correlation is undefined at constant trajectories and unstable near them.
Calling $\vu$ \emph{flat} if $\mathrm{sd}(\vu)\leq\zeta$ for a small threshold $\zeta\geq0$, with $\mathrm{sd}$ the population standard deviation of its entries, we adopt the convention $\rho(\vu,\vv)=1$ if both trajectories are flat and $\rho(\vu,\vv)=0$ if exactly one is; the same threshold governs the direction, spectral, and timing trajectory distance functions, as detailed below, and the increments.
We use $\zeta=10^{-8}$ (in naive units) in our experiments, so that a trajectory is flat only when it is constant up to floating-point noise.

\paragraph{Numerical Implementation: Direction Distance.}
The sign of an increment is read with the flatness threshold $\zeta$ above: an increment is \emph{flat}, $\sign(\Delta u_h)=0$, when it does not clear $\zeta$ in absolute value, $|\Delta u_h|\leq\zeta$, and it goes up or down, $\sign(\Delta u_h)=\pm1$, otherwise.
A flat increment thus counts as a third direction: two flat increments agree, whereas a flat increment against an upward or a downward one counts as a disagreement, exactly as an upward move against a downward one does.
Two flat trajectories are therefore at distance $0$ and a flat trajectory against a strictly monotone one at distance $1$, mirroring the convention adopted for $\rho$ above; the direction shares read the increments with the same rule.

\paragraph{Numerical Implementation: Spectral Distance.}
The trajectories are first standardized, $\check{\vu}=(\vu-\bar{u}\vone)/\|\vu-\bar{u}\vone\|_2$ with $\bar{u}=\frac1H\sum_h u_h$, and $\digamma$ is applied to $\check{\vu}$ and $\check{\vv}$; $\digamma$ denotes the \emph{normalized} discrete Fourier transform, $(\digamma\vu)_k=\frac{1}{\sqrt{H}}\sum_{h=1}^{H}u_h\,e^{-2\pi i k(h-1)/H}$, so that $\|\digamma\vu\|_2=\|\vu\|_2$.
The standardization is undefined at flat trajectories; with the flatness threshold $\zeta$ above, we set $\check{\vu}=\vzero$ whenever $\mathrm{sd}(\vu)\leq\zeta$.
Two unit-norm magnitude spectra are at most $\sqrt{2}$ apart, so $\dt_{\mathrm{spectral}}\in[0,1]$, and the maximum $1$ is attained exactly when the two spectra have disjoint supports.

\paragraph{Numerical Implementation: Timing Distance.}
For a lag $\tau\geq0$, the overlapping segments of $\vu$ and $\vv$ are $(u_{1+\tau},\ldots,u_H)$ and $(v_1,\ldots,v_{H-\tau})$, of length $H-\tau$.
Their means are denoted by $\bar{u}_\tau$ and $\bar{v}_\tau$, while $\bar{u}$ and $\bar{v}$ denote the means of the full trajectories.
Negative lags are handled by swapping the roles of $\vu$ and $\vv$.
The lag-$\tau$ Pearson correlation $\rho_\tau(\vu,\vv)$ is computed on the overlapping segments.
Following the shape-distance convention, a segment is flat if its population standard deviation is at most $\zeta$, computed over that segment's own length.
We set $\rho_\tau=1$ if both segments are flat and $\rho_\tau=0$ if exactly one is flat.

Pearson correlation measures agreement within the overlap alone, so a short overlap can achieve a high correlation even when the full trajectories align poorly.
We limit this effect in two ways.
First, we require the overlap to contain at least two points and cover at least half of the horizon, restricting admissible lags to $|\tau|\leq L$, where
\[
L=H-\max\{2,\lceil H/2\rceil\}.
\]
Second, we normalize the centered cross-product over the overlap by the norms of the two full centered trajectories.
For $\tau\geq0$, the resulting score is
\begin{equation*}
\tilde{\rho}_\tau(\vu,\vv)
=
\frac{
\sum_{k=1}^{H-\tau}
(u_{k+\tau}-\bar{u}_\tau)(v_k-\bar{v}_\tau)
}{
\sqrt{\sum_{i=1}^{H}(u_i-\bar{u})^2}
\sqrt{\sum_{i=1}^{H}(v_i-\bar{v})^2}
}.
\end{equation*}
The denominator is fixed across lags, so an overlap containing little of the full trajectories' variation receives a smaller score in magnitude, even if its Pearson correlation is high.
At zero lag, $\tilde{\rho}_0=\rho_0$.
If either full trajectory is flat, we use $\rho_\tau$ instead.

The best-alignment lag is $\tau^\star=\argmax_{|\tau|\leq L}\tilde{\rho}_\tau(\vu,\vv)$.
We define the timing distance as $\dt_{\mathrm{timing}}(\vu,\vv)=|\tau^\star|/L$, setting it to zero when $L=0$.
The distance therefore lies in $[0,1]$, and the thresholds $\tau_0=0.25$ and $\tau_1=0.5$ correspond to one quarter and one half of the largest admissible lag, respectively.

This implementation differs from the definition in \cref{sec:temporal-trajectory-distances} by capping the lag at $L$ rather than $H-2$, maximizing $\tilde{\rho}_\tau$ rather than $\rho_\tau$, and normalizing the selected lag by $L$.
The theoretical statements in \cref{app:bounds} apply to the definition in \cref{sec:temporal-trajectory-distances}; the modifications above concern the numerical implementation used in the experiments.

\paragraph{Error Trajectories.}
Consistent with the evaluation protocol of \cref{app:datasets}, a forecast cell whose ground truth was imputed at load has no observed error, so the error trajectories drop every forecast origin carrying at least one such cell, whereas the forecast and revision trajectories keep it (a prediction is a prediction whatever it is scored against).

\paragraph{Extremes.}
The location of the maximum and of the minimum of a forecast trajectory, $\ell^{f,M}_{t}$ and $\ell^{f,m}_{t}$ in \eqref{eq:consequences-readings}, is the first step at which the extreme is reached whenever several steps attain it.

\paragraph{Increments.} The first increment is $\Delta^{f}_{t,1}=\nv{\hat{y}}^{f}_{t+1}-\nv{y}_{t}$, i.e., its first term is the \emph{realized} level at the origin, that is, $\nv{\hat{y}}^{f}_{t}=\nv{y}_t$ in \eqref{eq:consequences-readings}.

\clearpage

\section{Additional Experimental Results}
\label{app:additional-results}

This appendix presents additional experimental results omitted from the main text due to space constraints.
Note that while the main text presents results aggregated at the dataset level, this section reports the results from \cref{sec:experiments} at a finer grain: separately for each target in multi-target datasets.

\Cref{app:horizon-wise-results} reports all pointwise and horizon-wise multiplicity metrics separately for each target.
Similarly, \Cref{app:temporal-multiplicity-experiment} reports all temporal multiplicity statistics separately for each target.

\Cref{app:pair-correlations} examines the relationship between pointwise or horizon-wise disagreement and temporal disagreement.
It reports the Spearman rank correlations separately for each target, together with their minimum and maximum across the three trajectory types, showing how the relationships summarized in the main text vary across targets and temporal aspects.

\subsection{Pointwise and Horizon-Wise Predictive Multiplicity}
\label{app:horizon-wise-results}

In this subsection we provide more details on the pointwise and horizon-wise multiplicity results, reported \emph{per target}. In particular, \cref{tab:pointwise-n-horizon-wise-multiplicity} reports \emph{range}, \emph{diameter}, \emph{ambiguity}, and \emph{discrepancy}, with horizon-wise statistics shown at the first and last horizons.
For ambiguity and discrepancy, we report results at two thresholds relative to the best validation MASE: $\tau_0=0.25\cdot\Ls(f^\star)$ and $\tau_1=0.5\cdot\Ls(f^\star)$.

\begin{table}[h]
    \centering
    \caption{Pointwise and horizon-wise predictive multiplicity statistics.
    Ambiguity $\alpha$ and discrepancy $\varsigma$ are measured at thresholds $\tau_0=0.25\cdot\Ls(f^\star)$ and $\tau_1=0.5\cdot\Ls(f^\star)$. Horizon-wise statistics (range, diameter, ambiguity, and discrepancy) are reported at the first and the last horizon only.}
    \label{tab:pointwise-n-horizon-wise-multiplicity}
    \footnotesize
    \setlength{\tabcolsep}{3pt}
    \resizebox{\textwidth}{!}{\begin{tabular}{lllllllllllllllllll}
\toprule
Dataset & \multicolumn{6}{l}{Pointwise} & \multicolumn{12}{l}{Horizon-Wise} \\
\cmidrule(lr){2-7} \cmidrule(lr){8-19}
 & $\nu$ & $\omega$ & $\alpha(\tau_{0})$ & $\varsigma(\tau_{0})$ & $\alpha(\tau_{1})$ & $\varsigma(\tau_{1})$ & $\nu^{(1)}$ & $\nu^{(H)}$ & $\omega^{(1)}$ & $\omega^{(H)}$ & $\alpha^{(1)}(\tau_{0})$ & $\alpha^{(H)}(\tau_{0})$ & $\varsigma^{(1)}(\tau_{0})$ & $\varsigma^{(H)}(\tau_{0})$ & $\alpha^{(1)}(\tau_{1})$ & $\alpha^{(H)}(\tau_{1})$ & $\varsigma^{(1)}(\tau_{1})$ & $\varsigma^{(H)}(\tau_{1})$ \\
\midrule
ATM 1D & 1.85 & 0.55 & 1.00 & 0.67 & 1.00 & 0.41 & 1.79 & 1.94 & 0.59 & 0.62 & 1.00 & 1.00 & 0.70 & 0.76 & 1.00 & 1.00 & 0.44 & 0.49 \\
ATM 1W & 1.68 & 0.59 & 0.99 & 0.72 & 0.98 & 0.48 & 1.75 & 1.67 & 0.71 & 0.61 & 0.99 & 0.99 & 0.76 & 0.75 & 0.98 & 0.98 & 0.55 & 0.51 \\
Bitcoin & 7.10 & 2.63 & 1.00 & 0.79 & 0.98 & 0.62 & 5.16 & 9.41 & 2.35 & 3.34 & 1.00 & 1.00 & 0.77 & 0.85 & 0.93 & 1.00 & 0.59 & 0.69 \\
Electricity D. A (NSW) & 1.38 & 0.35 & 1.00 & 0.78 & 1.00 & 0.58 & 0.76 & 1.64 & 0.24 & 0.42 & 1.00 & 1.00 & 0.76 & 0.83 & 1.00 & 1.00 & 0.53 & 0.67 \\
Electricity D. A (QLD) & 1.93 & 0.48 & 1.00 & 0.85 & 1.00 & 0.70 & 1.37 & 2.19 & 0.45 & 0.56 & 1.00 & 1.00 & 0.85 & 0.87 & 1.00 & 1.00 & 0.70 & 0.75 \\
Electricity D. A (SA) & 1.92 & 0.53 & 1.00 & 0.82 & 1.00 & 0.65 & 1.26 & 2.17 & 0.41 & 0.58 & 1.00 & 1.00 & 0.81 & 0.84 & 1.00 & 1.00 & 0.61 & 0.67 \\
Electricity D. A (TAS) & 1.90 & 0.44 & 1.00 & 0.79 & 1.00 & 0.60 & 1.59 & 2.04 & 0.46 & 0.48 & 1.00 & 1.00 & 0.83 & 0.83 & 1.00 & 1.00 & 0.65 & 0.66 \\
Electricity D. A (VIC) & 1.71 & 0.46 & 1.00 & 0.81 & 1.00 & 0.63 & 1.01 & 2.00 & 0.34 & 0.52 & 1.00 & 1.00 & 0.78 & 0.83 & 1.00 & 1.00 & 0.58 & 0.67 \\
Electricity D. Z (Value\_NE5) & 0.64 & 0.17 & 1.00 & 0.76 & 1.00 & 0.55 & 0.47 & 0.78 & 0.14 & 0.20 & 1.00 & 1.00 & 0.78 & 0.81 & 1.00 & 1.00 & 0.56 & 0.63 \\
Electricity D. Z (Value\_NE7) & 0.97 & 0.24 & 1.00 & 0.80 & 1.00 & 0.61 & 0.62 & 1.17 & 0.22 & 0.32 & 1.00 & 1.00 & 0.82 & 0.85 & 1.00 & 1.00 & 0.64 & 0.71 \\
Electricity Price & 1.24 & 0.37 & 1.00 & 0.63 & 1.00 & 0.35 & 0.98 & 1.29 & 0.34 & 0.48 & 1.00 & 1.00 & 0.63 & 0.72 & 1.00 & 1.00 & 0.32 & 0.47 \\
ETT 1H & 1.66 & 0.47 & 1.00 & 0.81 & 1.00 & 0.64 & 1.13 & 1.97 & 0.44 & 0.60 & 1.00 & 1.00 & 0.83 & 0.87 & 1.00 & 1.00 & 0.64 & 0.73 \\
Exchange Rate (AUD) & 5.35 & 1.19 & 1.00 & 0.57 & 0.99 & 0.29 & 4.09 & 6.26 & 1.23 & 1.74 & 1.00 & 1.00 & 0.68 & 0.86 & 0.99 & 1.00 & 0.35 & 0.54 \\
Exchange Rate (CAD) & 5.08 & 1.18 & 1.00 & 0.60 & 0.99 & 0.22 & 4.07 & 5.90 & 1.20 & 1.86 & 1.00 & 1.00 & 0.57 & 0.88 & 0.96 & 1.00 & 0.23 & 0.50 \\
Exchange Rate (CHF) & 3.67 & 0.94 & 1.00 & 0.79 & 1.00 & 0.48 & 3.07 & 4.37 & 1.05 & 1.34 & 1.00 & 1.00 & 0.92 & 0.98 & 1.00 & 1.00 & 0.55 & 0.71 \\
Exchange Rate (CNY) & 4.47 & 1.29 & 0.99 & 0.77 & 0.97 & 0.52 & 3.78 & 5.00 & 1.59 & 1.97 & 0.99 & 1.00 & 0.96 & 0.95 & 0.97 & 0.97 & 0.74 & 0.76 \\
Exchange Rate (GBP) & 3.58 & 1.02 & 1.00 & 0.72 & 1.00 & 0.37 & 3.09 & 4.18 & 1.06 & 1.43 & 1.00 & 1.00 & 0.78 & 0.91 & 1.00 & 1.00 & 0.41 & 0.54 \\
Exchange Rate (JPY) & 3.19 & 0.81 & 1.00 & 0.64 & 1.00 & 0.29 & 2.57 & 3.89 & 0.90 & 1.14 & 1.00 & 1.00 & 0.77 & 0.86 & 1.00 & 1.00 & 0.37 & 0.51 \\
Exchange Rate (NZD) & 4.85 & 1.36 & 1.00 & 0.73 & 1.00 & 0.34 & 4.03 & 5.84 & 1.43 & 2.02 & 1.00 & 1.00 & 0.79 & 0.97 & 1.00 & 1.00 & 0.42 & 0.65 \\
Exchange Rate (SGD) & 4.22 & 1.12 & 1.00 & 0.67 & 0.99 & 0.42 & 3.38 & 5.08 & 1.03 & 1.50 & 1.00 & 1.00 & 0.78 & 0.94 & 0.98 & 1.00 & 0.43 & 0.66 \\
FRED-MD (CPIAUCSL) & 0.09 & 0.04 & 1.00 & 0.82 & 0.93 & 0.33 & 0.06 & 0.12 & 0.04 & 0.05 & 1.00 & 1.00 & 0.73 & 1.00 & 0.82 & 1.00 & 0.36 & 0.69 \\
FRED-MD (INDPRO) & 0.16 & 0.11 & 0.99 & 0.80 & 0.88 & 0.58 & 0.11 & 0.21 & 0.08 & 0.13 & 0.98 & 1.00 & 0.84 & 0.91 & 0.79 & 0.94 & 0.55 & 0.68 \\
FRED-MD (UNRATE) & 0.00 & 0.00 & 0.00 & 0.00 & 0.00 & 0.00 & 0.00 & 0.00 & 0.00 & 0.00 & 0.00 & 0.00 & 0.00 & 0.00 & 0.00 & 0.00 & 0.00 & 0.00 \\
Gasoline & 3.10 & 1.22 & 1.00 & 0.89 & 1.00 & 0.74 & 3.03 & 3.29 & 1.32 & 1.29 & 1.00 & 1.00 & 0.91 & 0.94 & 1.00 & 1.00 & 0.75 & 0.78 \\
Weather & 2.96 & 1.06 & 1.00 & 0.81 & 1.00 & 0.64 & 2.73 & 3.12 & 1.39 & 1.10 & 1.00 & 1.00 & 0.88 & 0.85 & 1.00 & 1.00 & 0.76 & 0.68 \\
\midrule
Mean & 2.41 & 0.79 & 0.97 & 0.74 & 0.96 & 0.52 & 1.99 & 2.80 & 0.81 & 0.95 & 0.97 & 0.97 & 0.76 & 0.82 & 0.95 & 0.97 & 0.55 & 0.62 \\
\bottomrule
\end{tabular}
}
\end{table}

Two observations are worth noting:
\begin{enumerate}
    \item Pointwise multiplicity is pervasive: on most of the datasets the pointwise ambiguity is almost one at both thresholds, i.e., at almost every $(t,h)$ cell some pair of near-optimal models disagrees by more than half of the best MASE, and a single pair does so on close to half of the cells on average; the exceptions are datasets whose small and accurate sets fail to contain enough models to register large multiplicity.
    \item The horizon-wise statistics grow with the horizon on most datasets, in line with the findings of \citet{Kale2026}. However, there are a few exceptions, such as ATM 1W. These exceptions may reflect differences in the underlying dynamics: \citet{Kale2026} study chaotic systems, whereas not all forecasting tasks considered here necessarily exhibit chaotic behavior.
\end{enumerate}

\subsection{Temporal Predictive Multiplicity}
\label{app:temporal-multiplicity-experiment}

In this appendix we provide more details on the temporal multiplicity results. We report the temporal multiplicity statistics \emph{per target}. In particular, \cref{tab:temporal-main-pointwise} provides the equivalent per-target results for \emph{range} and \emph{diameter} in \cref{tab:temporal-main}. We also report the same results for \emph{ambiguity} and \emph{discrepancy} in \cref{tab:temporal-main-pointwise-ambiguity-tau1} using a threshold $\tau_1=0.5\cdot\Ls(f^\star)$.

\begin{table}[h]
    \centering
    \caption{Temporal multiplicity of the empirical Rashomon Sets.
    Each cell reports the \emph{range and diameter} ($\nu_{\dt} / \omega_{\dt}$), maximized over the three trajectories (forecast, revision, and error).
    $\dt_{\mathrm{magnitude}}$, $\dt_{\mathrm{increment}}$ and $\dt_{\mathrm{volatility}}$ are in seasonal-naive units while $\dt_{\mathrm{direction}}$, $\dt_{\mathrm{timing}}$, $\dt_{\mathrm{shape}}$, and $\dt_{\mathrm{spectral}}$ are scale-free on $[0,1]$ and reported as percentages to allow comparisons across different datasets.}
    \label{tab:temporal-main-pointwise}
    \footnotesize
    \setlength{\tabcolsep}{3pt}
    \resizebox{\textwidth}{!}{\begin{tabular}{llllllll}
\toprule
Dataset & $\dt_{\mathrm{magnitude}}$ & $\dt_{\mathrm{increment}}$ & $\dt_{\mathrm{volatility}}$ & $\dt_{\mathrm{direction}}$ & $\dt_{\mathrm{shape}}$ & $\dt_{\mathrm{spectral}}$ & $\dt_{\mathrm{timing}}$ \\
\midrule
ATM 1D & 1.38 / 0.66 & 1.55 / 0.81 & 0.87 / 0.35 & 99.7\% / 58.9\% & 95.9\% / 54.1\% & 85.1\% / 36.0\% & 99.9\% / 54.1\% \\
ATM 1W & 1.46 / 0.67 & 1.38 / 0.65 & 0.94 / 0.49 & 99.9\% / 55.3\% & 99.4\% / 56.0\% & 85.4\% / 31.6\% & 100\% / 57.6\% \\
Bitcoin & 7.00 / 2.99 & 3.66 / 2.08 & 2.01 / 0.88 & 99.9\% / 61.7\% & 99.4\% / 61.1\% & 81.5\% / 41.6\% & 100\% / 68.9\% \\
Electricity D. A (NSW) & 1.17 / 0.42 & 0.32 / 0.23 & 0.14 / 0.09 & 75.4\% / 48.5\% & 81.1\% / 39.2\% & 78.9\% / 42.9\% & 97.9\% / 32.7\% \\
Electricity D. A (QLD) & 1.55 / 0.57 & 0.41 / 0.28 & 0.19 / 0.13 & 76.6\% / 48.9\% & 83.3\% / 38.6\% & 81.7\% / 43.3\% & 98.3\% / 33.1\% \\
Electricity D. A (SA) & 1.56 / 0.62 & 0.47 / 0.33 & 0.21 / 0.15 & 73.4\% / 47.4\% & 75.1\% / 33.9\% & 76.8\% / 41.8\% & 94.9\% / 29.0\% \\
Electricity D. A (TAS) & 1.41 / 0.52 & 0.42 / 0.27 & 0.19 / 0.11 & 78.6\% / 49.4\% & 79.8\% / 33.8\% & 81.6\% / 40.4\% & 97.9\% / 28.5\% \\
Electricity D. A (VIC) & 1.45 / 0.56 & 0.40 / 0.26 & 0.17 / 0.10 & 76.6\% / 48.8\% & 80.5\% / 37.1\% & 80.1\% / 42.4\% & 97.8\% / 32.4\% \\
Electricity D. Z (Value\_NE5) & 0.54 / 0.20 & 0.12 / 0.08 & 0.04 / 0.02 & 65.9\% / 48.1\% & 76.7\% / 36.8\% & 75.9\% / 41.8\% & 97.8\% / 31.6\% \\
Electricity D. Z (Value\_NE7) & 0.80 / 0.30 & 0.20 / 0.15 & 0.07 / 0.05 & 68.8\% / 50.0\% & 79.4\% / 44.3\% & 77.3\% / 49.1\% & 99.0\% / 40.0\% \\
Electricity Price & 0.95 / 0.45 & 0.40 / 0.29 & 0.18 / 0.11 & 80.6\% / 45.8\% & 72.3\% / 35.0\% & 76.5\% / 38.2\% & 95.5\% / 30.2\% \\
ETT 1H & 1.42 / 0.56 & 0.50 / 0.33 & 0.23 / 0.14 & 83.4\% / 46.5\% & 76.6\% / 32.8\% & 78.7\% / 36.6\% & 95.5\% / 31.2\% \\
Exchange Rate (AUD) & 5.31 / 1.78 & 5.25 / 2.25 & 3.05 / 1.32 & 100\% / 100\% & 100\% / 99.6\% & 97.2\% / 70.9\% & 100\% / 100\% \\
Exchange Rate (CAD) & 4.79 / 1.41 & 4.46 / 1.75 & 2.61 / 1.22 & 100\% / 100\% & 100\% / 98.5\% & 97.0\% / 79.5\% & 100\% / 100\% \\
Exchange Rate (CHF) & 3.32 / 1.04 & 2.84 / 1.20 & 1.60 / 0.76 & 100\% / 100\% & 100\% / 97.2\% & 97.0\% / 76.0\% & 100\% / 100\% \\
Exchange Rate (CNY) & 3.85 / 1.70 & 4.84 / 2.31 & 2.83 / 1.05 & 100\% / 100\% & 99.7\% / 97.5\% & 89.9\% / 75.3\% & 100\% / 100\% \\
Exchange Rate (GBP) & 3.26 / 1.14 & 3.10 / 1.43 & 1.79 / 0.67 & 100\% / 100\% & 100\% / 97.2\% & 96.7\% / 71.9\% & 100\% / 100\% \\
Exchange Rate (JPY) & 2.93 / 0.95 & 2.38 / 1.00 & 1.36 / 0.61 & 100\% / 100\% & 100\% / 99.6\% & 96.9\% / 76.7\% & 100\% / 100\% \\
Exchange Rate (NZD) & 4.50 / 1.51 & 3.40 / 1.55 & 2.02 / 0.87 & 100\% / 100\% & 100\% / 98.5\% & 96.8\% / 86.7\% & 100\% / 100\% \\
Exchange Rate (SGD) & 3.90 / 1.23 & 3.18 / 1.20 & 1.83 / 0.56 & 100\% / 100\% & 100\% / 99.6\% & 96.3\% / 58.4\% & 100\% / 100\% \\
FRED-MD (CPIAUCSL) & 0.09 / 0.04 & 0.06 / 0.03 & 0.05 / 0.03 & 51.4\% / 27.1\% & 49.4\% / 24.4\% & 0.0\% / 0.0\% & 61.0\% / 23.8\% \\
FRED-MD (INDPRO) & 0.16 / 0.12 & 0.13 / 0.09 & 0.09 / 0.05 & 46.7\% / 34.6\% & 45.7\% / 35.6\% & 0.0\% / 0.0\% & 57.9\% / 44.9\% \\
FRED-MD (UNRATE) & 0.00 / 0.00 & 0.00 / 0.00 & 0.00 / 0.00 & 0.0\% / 0.0\% & 0.0\% / 0.0\% & 0.0\% / 0.0\% & 0.0\% / 0.0\% \\
Gasoline & 2.74 / 1.33 & 2.13 / 1.00 & 1.23 / 0.56 & 100\% / 71.6\% & 100\% / 65.6\% & 94.3\% / 45.8\% & 100\% / 78.6\% \\
Weather & 2.41 / 1.25 & 0.72 / 0.51 & 0.39 / 0.27 & 66.0\% / 51.9\% & 88.0\% / 49.5\% & 76.6\% / 46.9\% & 99.4\% / 48.3\% \\
\midrule
Mean & 2.14 / 0.92 & 1.33 / 0.70 & 0.75 / 0.35 & 82.3\% / 55.4\% & 83.7\% / 50.0\% & 75.5\% / 39.9\% & 93.2\% / 50.8\% \\
\bottomrule
\end{tabular}
}
\end{table}

\begin{table}[h]
    \centering
    \caption{Temporal multiplicity of the empirical Rashomon Sets.
    Each cell reports the \emph{ambiguity and discrepancy} ($\alpha_{\dt} / \varsigma_{\dt}$) with threshold $\tau_1=0.5\cdot\Ls(f^\star)$, maximized over the three trajectories (forecast, revision, and error).}
    \label{tab:temporal-main-pointwise-ambiguity-tau1}
    \footnotesize
    \setlength{\tabcolsep}{3pt}
    \resizebox{\textwidth}{!}{\begin{tabular}{llllllll}
\toprule
Dataset & $\dt_{\mathrm{magnitude}}$ & $\dt_{\mathrm{increment}}$ & $\dt_{\mathrm{volatility}}$ & $\dt_{\mathrm{direction}}$ & $\dt_{\mathrm{shape}}$ & $\dt_{\mathrm{spectral}}$ & $\dt_{\mathrm{timing}}$ \\
\midrule
ATM 1D & 100\% / 62\% & 100\% / 76\% & 90\% / 16\% & 100\% / 53\% & 100\% / 58\% & 100\% / 23\% & 100\% / 56\% \\
ATM 1W & 98\% / 63\% & 98\% / 61\% & 90\% / 44\% & 100\% / 58\% & 100\% / 58\% & 100\% / 23\% & 100\% / 41\% \\
Bitcoin & 99\% / 75\% & 96\% / 70\% & 73\% / 24\% & 100\% / 62\% & 100\% / 65\% & 100\% / 34\% & 100\% / 79\% \\
Electricity D. A (NSW) & 100\% / 95\% & 98\% / 67\% & 13\% / 2\% & 100\% / 42\% & 95\% / 27\% & 99\% / 30\% & 98\% / 29\% \\
Electricity D. A (QLD) & 100\% / 99\% & 100\% / 77\% & 31\% / 10\% & 100\% / 44\% & 96\% / 29\% & 99\% / 32\% & 98\% / 30\% \\
Electricity D. A (SA) & 100\% / 94\% & 99\% / 71\% & 20\% / 7\% & 100\% / 37\% & 88\% / 22\% & 97\% / 28\% & 95\% / 25\% \\
Electricity D. A (TAS) & 100\% / 96\% & 92\% / 48\% & 10\% / 1\% & 100\% / 47\% & 92\% / 21\% & 99\% / 25\% & 98\% / 25\% \\
Electricity D. A (VIC) & 100\% / 97\% & 98\% / 51\% & 19\% / 6\% & 100\% / 43\% & 94\% / 27\% & 99\% / 30\% & 98\% / 28\% \\
Electricity D. Z (Value\_NE5) & 100\% / 90\% & 41\% / 12\% & 3\% / 1\% & 100\% / 36\% & 93\% / 24\% & 99\% / 25\% & 98\% / 29\% \\
Electricity D. Z (Value\_NE7) & 100\% / 99\% & 95\% / 75\% & 4\% / 1\% & 100\% / 50\% & 97\% / 39\% & 99\% / 48\% & 99\% / 39\% \\
Electricity Price & 100\% / 50\% & 40\% / 13\% & 0\% / 0\% & 100\% / 34\% & 88\% / 21\% & 98\% / 18\% & 96\% / 27\% \\
ETT 1H & 100\% / 98\% & 99\% / 83\% & 37\% / 8\% & 100\% / 38\% & 89\% / 22\% & 96\% / 18\% & 96\% / 28\% \\
Exchange Rate (AUD) & 99\% / 34\% & 96\% / 45\% & 75\% / 8\% & 100\% / 100\% & 100\% / 100\% & 100\% / 100\% & 100\% / 100\% \\
Exchange Rate (CAD) & 98\% / 27\% & 95\% / 29\% & 55\% / 8\% & 100\% / 100\% & 100\% / 100\% & 100\% / 100\% & 100\% / 100\% \\
Exchange Rate (CHF) & 100\% / 60\% & 100\% / 75\% & 95\% / 35\% & 100\% / 100\% & 100\% / 100\% & 100\% / 100\% & 100\% / 100\% \\
Exchange Rate (CNY) & 97\% / 76\% & 98\% / 88\% & 92\% / 47\% & 100\% / 100\% & 100\% / 100\% & 100\% / 100\% & 100\% / 100\% \\
Exchange Rate (GBP) & 100\% / 40\% & 99\% / 46\% & 74\% / 12\% & 100\% / 100\% & 100\% / 100\% & 100\% / 100\% & 100\% / 100\% \\
Exchange Rate (JPY) & 100\% / 34\% & 100\% / 46\% & 78\% / 11\% & 100\% / 100\% & 100\% / 100\% & 100\% / 99\% & 100\% / 100\% \\
Exchange Rate (NZD) & 100\% / 38\% & 100\% / 46\% & 69\% / 10\% & 100\% / 100\% & 100\% / 100\% & 100\% / 100\% & 100\% / 100\% \\
Exchange Rate (SGD) & 98\% / 51\% & 98\% / 50\% & 84\% / 12\% & 100\% / 100\% & 100\% / 100\% & 100\% / 78\% & 100\% / 100\% \\
FRED-MD (CPIAUCSL) & 99\% / 42\% & 85\% / 21\% & 61\% / 19\% & 28\% / 10\% & 46\% / 21\% & 0\% / 0\% & 61\% / 24\% \\
FRED-MD (INDPRO) & 92\% / 66\% & 92\% / 54\% & 55\% / 23\% & 33\% / 18\% & 37\% / 29\% & 0\% / 0\% & 58\% / 45\% \\
FRED-MD (UNRATE) & 0\% / 0\% & 0\% / 0\% & 0\% / 0\% & 0\% / 0\% & 0\% / 0\% & 0\% / 0\% & 0\% / 0\% \\
Gasoline & 100\% / 92\% & 100\% / 90\% & 100\% / 61\% & 100\% / 79\% & 100\% / 74\% & 100\% / 49\% & 100\% / 70\% \\
Weather & 100\% / 89\% & 70\% / 50\% & 27\% / 9\% & 100\% / 66\% & 98\% / 51\% & 99\% / 41\% & 100\% / 49\% \\
\midrule
Mean & 96\% / 73\% & 84\% / 57\% & 51\% / 18\% & 93\% / 53\% & 90\% / 47\% & 90\% / 34\% & 93\% / 49\% \\
\bottomrule
\end{tabular}
}
\end{table}

\subsection{Temporal Multiplicity Is Uncorrelated with Other Notions of Multiplicity}
\label{app:pair-correlations}
\label{app:pair-correlations-raw}

In this appendix we report the correlation results of \cref{tab:pointwise-n-horizon-wise-multiplicity-vs-temporal-correlation} \emph{per target}.
\Cref{tab:raw-correlations-mean-per-target} reports the same Spearman rank correlations of \eqref{eq:raw-correlation}, for both the pointwise and the horizon-wise reading, but with one row per target rather than aggregated over the targets of each multi-target dataset.
The picture is the same as in the main text: the correlation is close to one for magnitude, decreases for increment and volatility, and is weak or absent for direction, shape, spectral behavior, and timing on every target.

\begin{table}[h]
    \centering
    \caption{Spearman rank correlation, over all (pair, origin) samples of the empirical Rashomon set, between the pair's \emph{mean pointwise gap} and its \emph{trajectory distance} $\dt$ on the same forecast origin, read as in \eqref{eq:raw-correlation} with $\Ls^{\star}_h\equiv1$ (Pointwise) or as is (Horizon-Wise).
    Each row reports the minimum / maximum over the three trajectories.}
    \label{tab:raw-correlations-mean-per-target}
    \footnotesize
    \setlength{\tabcolsep}{3pt}
    \resizebox{\textwidth}{!}{\begin{tabular}{llllllll|lllllll}
\toprule
Dataset & \multicolumn{7}{l}{Pointwise} & \multicolumn{7}{l}{Horizon-Wise} \\
\cmidrule(lr){2-8} \cmidrule(lr){9-15}
 & $\dt_{\mathrm{magnitude}}$ & $\dt_{\mathrm{increment}}$ & $\dt_{\mathrm{volatility}}$ & $\dt_{\mathrm{direction}}$ & $\dt_{\mathrm{shape}}$ & $\dt_{\mathrm{spectral}}$ & $\dt_{\mathrm{timing}}$ & $\dt_{\mathrm{magnitude}}$ & $\dt_{\mathrm{increment}}$ & $\dt_{\mathrm{volatility}}$ & $\dt_{\mathrm{direction}}$ & $\dt_{\mathrm{shape}}$ & $\dt_{\mathrm{spectral}}$ & $\dt_{\mathrm{timing}}$ \\
\midrule
ATM 1D & 0.98 / 0.98 & 0.58 / 0.75 & 0.35 / 0.37 & 0.15 / 0.28 & 0.24 / 0.61 & 0.05 / 0.42 & 0.07 / 0.24 & 0.98 / 0.98 & 0.58 / 0.75 & 0.35 / 0.37 & 0.15 / 0.28 & 0.24 / 0.61 & 0.05 / 0.42 & 0.07 / 0.24 \\
ATM 1W & 0.99 / 0.99 & 0.56 / 0.70 & 0.39 / 0.47 & 0.13 / 0.18 & 0.17 / 0.40 & 0.04 / 0.24 & 0.14 / 0.20 & 0.99 / 0.99 & 0.56 / 0.70 & 0.39 / 0.47 & 0.13 / 0.18 & 0.17 / 0.40 & 0.04 / 0.24 & 0.14 / 0.20 \\
Bitcoin & 0.99 / 1.00 & 0.61 / 0.69 & 0.43 / 0.49 & -0.02 / 0.19 & -0.01 / 0.32 & -0.05 / 0.26 & 0.02 / 0.14 & 0.96 / 0.98 & 0.60 / 0.67 & 0.41 / 0.48 & -0.03 / 0.17 & -0.03 / 0.30 & -0.05 / 0.25 & -0.01 / 0.13 \\
Electricity D. A (NSW) & 0.99 / 0.99 & 0.51 / 0.66 & 0.16 / 0.36 & -0.01 / 0.16 & 0.04 / 0.63 & -0.08 / 0.46 & 0.08 / 0.33 & 0.97 / 0.97 & 0.52 / 0.65 & 0.16 / 0.36 & -0.01 / 0.16 & 0.03 / 0.61 & -0.09 / 0.45 & 0.07 / 0.32 \\
Electricity D. A (QLD) & 0.99 / 0.99 & 0.53 / 0.60 & 0.19 / 0.33 & -0.01 / 0.23 & 0.04 / 0.63 & -0.05 / 0.44 & 0.08 / 0.36 & 0.97 / 0.97 & 0.54 / 0.59 & 0.19 / 0.33 & -0.01 / 0.23 & 0.03 / 0.63 & -0.06 / 0.44 & 0.07 / 0.35 \\
Electricity D. A (SA) & 0.99 / 0.99 & 0.46 / 0.61 & 0.16 / 0.29 & -0.03 / 0.26 & -0.00 / 0.61 & -0.08 / 0.45 & 0.05 / 0.34 & 0.97 / 0.97 & 0.47 / 0.58 & 0.16 / 0.30 & -0.03 / 0.26 & -0.03 / 0.59 & -0.10 / 0.43 & 0.04 / 0.33 \\
Electricity D. A (TAS) & 0.99 / 0.99 & 0.55 / 0.55 & 0.23 / 0.30 & 0.02 / 0.21 & 0.11 / 0.43 & 0.04 / 0.35 & 0.11 / 0.24 & 0.96 / 0.97 & 0.54 / 0.55 & 0.23 / 0.29 & 0.02 / 0.21 & 0.09 / 0.42 & 0.02 / 0.34 & 0.10 / 0.24 \\
Electricity D. A (VIC) & 0.99 / 0.99 & 0.53 / 0.66 & 0.21 / 0.31 & -0.02 / 0.19 & 0.04 / 0.64 & -0.07 / 0.49 & 0.08 / 0.33 & 0.96 / 0.97 & 0.55 / 0.64 & 0.21 / 0.31 & -0.01 / 0.18 & 0.02 / 0.62 & -0.08 / 0.48 & 0.07 / 0.32 \\
Electricity D. Z (Value\_NE5) & 0.99 / 0.99 & 0.33 / 0.51 & 0.19 / 0.28 & -0.08 / 0.23 & 0.08 / 0.40 & -0.05 / 0.26 & 0.12 / 0.28 & 0.98 / 0.99 & 0.33 / 0.51 & 0.19 / 0.28 & -0.07 / 0.23 & 0.08 / 0.39 & -0.06 / 0.26 & 0.11 / 0.28 \\
Electricity D. Z (Value\_NE7) & 0.99 / 0.99 & 0.23 / 0.52 & 0.09 / 0.24 & -0.01 / 0.18 & 0.07 / 0.52 & -0.03 / 0.38 & 0.10 / 0.37 & 0.98 / 0.98 & 0.24 / 0.52 & 0.09 / 0.24 & -0.00 / 0.18 & 0.07 / 0.51 & -0.03 / 0.38 & 0.10 / 0.37 \\
Electricity Price & 0.99 / 0.99 & 0.59 / 0.64 & 0.20 / 0.31 & -0.02 / 0.06 & 0.17 / 0.35 & 0.08 / 0.23 & 0.10 / 0.14 & 0.98 / 0.98 & 0.59 / 0.64 & 0.20 / 0.31 & -0.02 / 0.06 & 0.15 / 0.34 & 0.07 / 0.22 & 0.10 / 0.14 \\
ETT 1H & 0.99 / 0.99 & 0.56 / 0.68 & 0.30 / 0.39 & 0.32 / 0.36 & 0.43 / 0.51 & 0.31 / 0.39 & 0.17 / 0.29 & 0.98 / 0.98 & 0.58 / 0.69 & 0.31 / 0.40 & 0.33 / 0.38 & 0.41 / 0.50 & 0.30 / 0.39 & 0.17 / 0.29 \\
Exchange Rate (AUD) & 0.99 / 0.99 & 0.57 / 0.74 & 0.36 / 0.45 & 0.09 / 0.21 & 0.14 / 0.39 & 0.02 / 0.24 & 0.09 / 0.15 & 0.97 / 0.98 & 0.57 / 0.73 & 0.36 / 0.45 & 0.09 / 0.21 & 0.14 / 0.37 & 0.02 / 0.23 & 0.09 / 0.14 \\
Exchange Rate (CAD) & 0.99 / 0.99 & 0.56 / 0.71 & 0.34 / 0.43 & 0.10 / 0.17 & 0.16 / 0.38 & 0.01 / 0.24 & 0.07 / 0.12 & 0.97 / 0.98 & 0.55 / 0.70 & 0.34 / 0.43 & 0.09 / 0.17 & 0.14 / 0.36 & 0.01 / 0.22 & 0.07 / 0.11 \\
Exchange Rate (CHF) & 0.99 / 0.99 & 0.55 / 0.70 & 0.34 / 0.43 & 0.10 / 0.18 & 0.17 / 0.39 & 0.02 / 0.24 & 0.10 / 0.15 & 0.97 / 0.98 & 0.55 / 0.69 & 0.34 / 0.43 & 0.10 / 0.17 & 0.16 / 0.37 & 0.02 / 0.23 & 0.10 / 0.15 \\
Exchange Rate (CNY) & 0.99 / 0.99 & 0.72 / 0.83 & 0.45 / 0.51 & 0.10 / 0.20 & 0.15 / 0.39 & 0.01 / 0.26 & 0.11 / 0.22 & 0.97 / 0.98 & 0.71 / 0.82 & 0.44 / 0.51 & 0.10 / 0.19 & 0.14 / 0.38 & 0.01 / 0.25 & 0.10 / 0.20 \\
Exchange Rate (GBP) & 0.99 / 0.99 & 0.52 / 0.69 & 0.32 / 0.41 & 0.13 / 0.21 & 0.22 / 0.47 & 0.02 / 0.28 & 0.11 / 0.18 & 0.96 / 0.98 & 0.51 / 0.68 & 0.32 / 0.40 & 0.13 / 0.21 & 0.20 / 0.45 & 0.02 / 0.27 & 0.11 / 0.16 \\
Exchange Rate (JPY) & 0.99 / 0.99 & 0.50 / 0.67 & 0.29 / 0.38 & 0.11 / 0.22 & 0.17 / 0.38 & 0.01 / 0.23 & 0.12 / 0.17 & 0.96 / 0.98 & 0.49 / 0.66 & 0.29 / 0.38 & 0.10 / 0.21 & 0.16 / 0.35 & 0.02 / 0.22 & 0.12 / 0.16 \\
Exchange Rate (NZD) & 0.99 / 0.99 & 0.53 / 0.68 & 0.32 / 0.40 & 0.11 / 0.19 & 0.18 / 0.38 & 0.01 / 0.23 & 0.09 / 0.15 & 0.97 / 0.98 & 0.52 / 0.67 & 0.32 / 0.40 & 0.11 / 0.18 & 0.17 / 0.36 & 0.01 / 0.22 & 0.09 / 0.14 \\
Exchange Rate (SGD) & 0.99 / 0.99 & 0.58 / 0.74 & 0.38 / 0.47 & 0.10 / 0.13 & 0.19 / 0.32 & 0.04 / 0.19 & 0.06 / 0.16 & 0.97 / 0.98 & 0.58 / 0.73 & 0.37 / 0.47 & 0.10 / 0.12 & 0.18 / 0.30 & 0.04 / 0.19 & 0.06 / 0.15 \\
FRED-MD (CPIAUCSL) & 0.99 / 0.99 & 0.63 / 0.63 & 0.47 / 0.64 & 0.04 / 0.19 & 0.05 / 0.28 & -0.03 / 0.04 & 0.03 / 0.16 & 0.97 / 0.97 & 0.59 / 0.60 & 0.44 / 0.60 & 0.03 / 0.18 & 0.04 / 0.26 & -0.04 / 0.04 & 0.03 / 0.15 \\
FRED-MD (INDPRO) & 0.99 / 1.00 & 0.69 / 0.69 & 0.45 / 0.49 & 0.20 / 0.40 & 0.35 / 0.42 & 0.00 / 0.04 & 0.13 / 0.40 & 0.98 / 0.99 & 0.66 / 0.67 & 0.44 / 0.48 & 0.19 / 0.39 & 0.34 / 0.42 & 0.00 / 0.04 & 0.12 / 0.39 \\
FRED-MD (UNRATE) & 0.98 / 0.99 & 0.61 / 0.67 & 0.51 / 0.64 & 0.03 / 0.23 & 0.22 / 0.38 & -0.11 / -0.01 & 0.02 / 0.16 & 0.97 / 0.99 & 0.61 / 0.67 & 0.51 / 0.62 & 0.04 / 0.24 & 0.23 / 0.38 & -0.10 / 0.00 & 0.03 / 0.17 \\
Gasoline & 0.99 / 0.99 & 0.56 / 0.61 & 0.35 / 0.45 & 0.09 / 0.22 & 0.12 / 0.43 & -0.01 / 0.29 & 0.10 / 0.23 & 0.99 / 0.99 & 0.56 / 0.61 & 0.35 / 0.45 & 0.09 / 0.22 & 0.12 / 0.43 & -0.01 / 0.29 & 0.10 / 0.23 \\
Weather & 1.00 / 1.00 & 0.69 / 0.72 & 0.51 / 0.52 & 0.35 / 0.43 & 0.34 / 0.52 & 0.24 / 0.43 & 0.27 / 0.35 & 0.99 / 0.99 & 0.70 / 0.73 & 0.52 / 0.54 & 0.37 / 0.44 & 0.34 / 0.52 & 0.24 / 0.43 & 0.27 / 0.35 \\
\midrule
Mean & 0.99 / 0.99 & 0.54 / 0.68 & 0.32 / 0.44 & 0.09 / 0.26 & 0.16 / 0.47 & 0.04 / 0.31 & 0.10 / 0.27 & 0.98 / 0.98 & 0.54 / 0.68 & 0.32 / 0.44 & 0.09 / 0.26 & 0.15 / 0.46 & 0.04 / 0.31 & 0.10 / 0.26 \\
\bottomrule
\end{tabular}
}
\end{table}

\clearpage

\section{Predictive Multiplicity Hierarchy Characterization}
\label{app:proofs}
\label{app:horizon-wise-invariance}

\begin{table}
    \centering
    \caption{Classification of the loss and of the multiplicity statistics of the hierarchy.}
    \label{tab:dichotomy}
    \begin{tabular}{lccc}
        \toprule
        \textbf{Statistic} & \makecell{\textbf{Multiplicity-Aware} \\[-2pt] \tiny not invariant to \\[-4pt] \tiny per-model cell permutations} & \makecell{\textbf{Horizon-Aware} \\[-2pt] \tiny not invariant to \\[-4pt] \tiny cell permutations} & \makecell{\textbf{Trajectory-Aware} \\[-2pt] \tiny not invariant to \\[-4pt] \tiny row permutations} \\
        \midrule
        Loss $\Ls$ & $\times$ & $\times$ & $\times$ \\
        Pointwise Multiplicity & \checkmark & $\times$ & $\times$ \\
        Horizon-wise Multiplicity & \checkmark & \checkmark & $\times$ \\
        Temporal Multiplicity & \checkmark & \checkmark & \checkmark \\
        \bottomrule
    \end{tabular}
\end{table}

This appendix provides the deferred proofs and develops the theoretical results underlying our hierarchy of predictive multiplicity.
In particular, it establishes the distinction between the information retained by its three levels: pointwise and horizon-wise statistics are invariant to rearrangements that preserve comparisons within each horizon, whereas temporal statistics can detect the resulting changes in trajectory structure. \cref{tab:dichotomy} provides a summary of the theoretical results.

\EA{We begin by establishing what the loss and the pointwise and horizon-wise multiplicity statistics cannot detect.
The key observation is that predictions can be rearranged across origins within each horizon, changing the trajectories they form while preserving identical pointwise and horizon-wise multiplicity statistics.
The following proposition formalizes this limitation through the invariance notions of \cref{def:awareness}, showing why agreement at these levels does not fully characterize temporal disagreement.}

\begin{proposition}[Classification of Loss, and Pointwise and Horizon-Wise Multiplicity Statistics]
\label{prop:dichotomy}
The loss $\Ls$ is not multiplicity-aware; for every tolerance $\tau$, the horizon-wise statistics of \cref{def:horizon-pm} are not trajectory-aware, and the pointwise statistics of \cref{def:pointwise-pm} are neither horizon-aware nor trajectory-aware.
\end{proposition}
\begin{proof}
(i) The loss (\eqref{eq:decomposable-loss}) of each model is the entrywise sum over its own error matrix, which is invariant to any permutation of that matrix's cells; hence $\Ls$ is invariant to per-model cell permutations, i.e., it is not multiplicity-aware, nor, a fortiori, horizon- or trajectory-aware.

(ii) Each pointwise statistic (\cref{def:pointwise-pm}) is built from $TH$ per-cell quantities $\big(|\hat{y}^{f}_{t+h}-\hat{y}^{g}_{t+h}|\big)_{f,g\in\gR_\eps}$, each of which depends only on the forecasts at the single pair $(t,h)$:
the prediction range and ambiguity are symmetric functions of these quantities, while the prediction diameter and discrepancy are maxima over pairs $(f,g)$ of symmetric functions of them.
A cell permutation moves each such pair as a unit and therefore merely permutes these $TH$ quantities, under which a symmetric function---and hence a maximum of symmetric functions---is invariant; hence no pointwise statistic is horizon-aware, nor, a fortiori, trajectory-aware, since every row permutation is a cell permutation.

(iii) The horizon-wise statistics (\cref{def:horizon-pm}) are, for each $h$, symmetric functions of the same per-cell quantities restricted to horizon $h$; a row permutation permutes these quantities within each horizon, so no horizon-wise statistic is trajectory-aware.
\end{proof}

Before proving that the temporal statistics proposed in the main text are trajectory-aware in the sense of \cref{def:awareness}, we establish a useful reduction that the proof will later use: for a pair of models $f$ and $g$, the four statistics at each level of multiplicity reduce to two, as range coincides with diameter and ambiguity with discrepancy.

\begin{remark}[Multiplicity Statistics of a Pair of Models]
\label{rem:pair}
Let $\dt$ be symmetric and vanish on the diagonal, i.e., $\dt(\vu,\vv)=\dt(\vv,\vu)$ and $\dt(\vu,\vu)=0$ for all $\vu,\vv\in\R^{H}$.
Then, for every pair $f,g\in\gF$, every $\tau\geq0$, every horizon $h\in\{1,\ldots,H\}$, and every trajectory $\vv_t$, we have:
\begin{equation*}
\nu(\{f,g\})=\omega(\{f,g\})=\mathop{\E}_{\substack{t\in\gT\\ 1\leq h\leq H}}\Big[\big|\hat{y}^{f}_{t+h}-\hat{y}^{g}_{t+h}\big|\Big]
\end{equation*}

\begin{equation*}
\alpha_{\tau}(\{f,g\})=\varsigma_{\tau}(\{f,g\})=\mathop{\E}_{\substack{t\in\gT\\ 1\leq h\leq H}}\Big[\1\Big\{\big|\hat{y}^{f}_{t+h}-\hat{y}^{g}_{t+h}\big|>\tau\Big\}\Big]
\end{equation*}

\begin{equation*}
\nu^{(h)}(\{f,g\})=\omega^{(h)}(\{f,g\})=\mathop{\E}_{t\in\gT}\Big[\big|\hat{y}^{f}_{t+h}-\hat{y}^{g}_{t+h}\big|\Big]
\end{equation*}

\begin{equation*}
\alpha^{(h)}_{\tau}(\{f,g\})=\varsigma^{(h)}_{\tau}(\{f,g\})=\mathop{\E}_{t\in\gT}\Big[\1\Big\{\big|\hat{y}^{f}_{t+h}-\hat{y}^{g}_{t+h}\big|>\tau\Big\}\Big]
\end{equation*}

\begin{equation*}
\nu_{\dt}(\{f,g\})=\omega_{\dt}(\{f,g\})=\mathop{\E}_{t\in\gT}\Big[\dt\big(\vv^{f}_{t},\vv^{g}_{t}\big)\Big]
\end{equation*}

\begin{equation*}
\alpha_{\dt,\tau}(\{f,g\})=\varsigma_{\dt,\tau}(\{f,g\})=\mathop{\E}_{t\in\gT}\Big[\1\Big\{\dt\big(\vv^{f}_{t},\vv^{g}_{t}\big)>\tau\Big\}\Big].
\end{equation*}
\begin{proof}
On a two-element set, every maximum over $f',g'\in\{f,g\}$ ranges over the two zero distances between a model and itself and the two copies of the pair's distance. The same conclusion therefore holds whether the maximum is taken inside the expectation (range, ambiguity) or outside it (diameter, discrepancy).
The hypotheses of the remark hold for every trajectory distance function $\dt$ of \cref{sec:temporal-trajectory-distances}, each of which is symmetric and vanishes at identical trajectories. They also hold trivially for the absolute difference used at the pointwise and horizon-wise levels.
\end{proof}
\end{remark}

For a pair of models, multiplicity is therefore summarized by the expected distance between their predictions and the rate at which that distance exceeds $\tau$. Consequently, the inequalities $\omega\leq\nu$ and $\varsigma\leq\alpha$ noted after \cref{def:pointwise-pm} can be strict only for sets of three or more models.

With that in mind, we now prove that the statistics induced by the trajectory distances of \cref{sec:temporal-trajectory-distances} are trajectory-aware and hence, by the nesting of \cref{def:awareness}, also horizon- and multiplicity-aware.

\begin{proposition}[Temporal Multiplicity Statistics Are Trajectory-Aware]
\label{prop:axes-temporal}
The seven distances in \cref{sec:temporal-trajectory-distances} give trajectory-aware range and diameter for forecast, revision, and error trajectories.
Their ambiguity and discrepancy are also trajectory-aware, with the following restrictions on the tolerance:
\begin{itemize}[nosep]
    \item for magnitude, increment, and volatility, any $\tau\geq0$;
    \item for direction, shape, spectral, and timing, any $0\leq\tau<1$.
\end{itemize}
For the latter four distances, ambiguity and discrepancy are always zero when $\tau\geq1$.

Specifically, for each statistic at the stated tolerances, we can choose a horizon length and two ensembles that differ only by a row permutation, yet have different temporal multiplicity.
Their pointwise and horizon-wise multiplicity remain unchanged.
\end{proposition}
\begin{proof}
A row permutation exchanges origins within a horizon, using the same exchange for every model.
It preserves the comparisons between models at that horizon, but can change which values belong to the same trajectory.
We will show that this can change each of the temporal statistics.

It is enough to use two models, $f$ and $g$, and two trajectories per model.
By \cref{rem:pair}, their range and diameter both equal the average distance between the models' trajectories.
Their ambiguity and discrepancy both equal the fraction of trajectories for which this distance exceeds $\tau$.
We therefore only need examples where a row permutation changes this average or this fraction.

We begin with error trajectories, so each example consists of two error matrices with two rows.
Whenever we exchange entries within a column, we do so in both matrices.
We then explain how to use the same examples for forecasts and revisions, and how to include both models in a Rashomon set.

\begin{enumerate}[leftmargin=*,label=\arabic*.]
    \item \emph{Magnitude, increment, and volatility.}
    Consider a model whose errors are constant along each trajectory, and compare it with a model whose errors are all zero:
    \begin{equation*}
        \mE^f=\begin{pmatrix}1&1\\0&0\end{pmatrix},
        \qquad
        \mE^g=\begin{pmatrix}0&0\\0&0\end{pmatrix}.
    \end{equation*}
    Exchange the entries in the second column.
    The first matrix becomes
    \begin{equation*}
        \begin{pmatrix}1&0\\0&1\end{pmatrix},
    \end{equation*}
    while the second stays zero.
    Before the exchange, the two magnitude distances are $1$ and $0$; afterward, both are $1/\sqrt{2}$.
    Their average therefore changes from $1/2$ to $1/\sqrt{2}$, changing both range and diameter.

    The same exchange also changes increment and volatility distances.
    Initially, all trajectories are constant, so both distances are zero at both origins.
    Afterward, each trajectory of $f$ changes by $1$ in absolute value, while those of $g$ remain constant.
    Both distances are now $1$ at both origins, so their range and diameter change from $0$ to $1$.

    To handle ambiguity and discrepancy at any given $\tau\geq0$, replace every $1$ in this example by a number $a>\sqrt{2}\tau$.
    The magnitude distances are then $(a,0)$ before the exchange and $(a/\sqrt{2},a/\sqrt{2})$ afterward: the fraction exceeding $\tau$ changes from $1/2$ to $1$.
    For increment and volatility, this fraction changes from $0$ to $1$.

    \item \emph{Direction and shape.}
    This time, start with two models whose trajectories are constant but at opposite levels:
    \begin{equation*}
        \mE^f=\begin{pmatrix}1&1\\-1&-1\end{pmatrix},
        \qquad
        \mE^g=\begin{pmatrix}-1&-1\\1&1\end{pmatrix}.
    \end{equation*}
    At each origin, both models have zero increments, so their direction distance is zero.
    Their shape distance is also zero by the convention for two constant trajectories in \cref{app:multiplicity-metrics}.

    Exchange the entries in the second column.
    At each origin, the two models now have opposite trajectories, $(1,-1)$ and $(-1,1)$.
    One decreases while the other increases, and their correlation is $-1$.
    Thus both direction and shape distances change from $0$ to $1$ at both origins.
    Their range and diameter change from $0$ to $1$, as do their ambiguity and discrepancy for every $0\leq\tau<1$.

    \item \emph{Spectral.}
    We need four horizons for this example:
    \begin{equation*}
        \mE^f=\begin{pmatrix}1&-1&1&-1\\1&-1&1&-1\end{pmatrix},
        \qquad
        \mE^g=\begin{pmatrix}1&0&-1&0\\-1&0&1&0\end{pmatrix}.
    \end{equation*}
    The two models initially vary at different frequencies.
    More explicitly, after standardizing each trajectory, their Fourier magnitude spectra are
    \begin{equation*}
        (0,0,1,0)\quad\text{for }f,
        \qquad
        (0,1/\sqrt{2},0,1/\sqrt{2})\quad\text{for }g.
    \end{equation*}
    These spectra have no nonzero entries in common, so their spectral distance is $1$ at both origins.

    Now exchange the two entries within column $3$ and within column $4$.
    The trajectories of $f$ stay unchanged, while those of $g$ become $(1,0,1,0)$ and $(-1,0,-1,0)$.
    Subtracting their means gives $\tfrac12(1,-1,1,-1)$ and its negative.
    After standardization, both therefore have the same magnitude spectrum as the trajectories of $f$.
    The spectral distance changes from $1$ to $0$ at both origins.
    Hence range and diameter change from $1$ to $0$, as do ambiguity and discrepancy for every $0\leq\tau<1$.

    \item \emph{Timing.}
    The timing distance in \cref{sec:temporal-trajectory-distances} measures the shift that gives the largest Pearson correlation.
    When several shifts give the same correlation, we choose the one with the smallest absolute value.
    Take three horizons and let
    \begin{equation*}
        \mE^f=\begin{pmatrix}1&0&-1\\1&0&-1\end{pmatrix},
        \qquad
        \mE^g=\begin{pmatrix}1&0&-1\\-1&0&1\end{pmatrix}.
    \end{equation*}
    The possible shifts are $-1$, $0$, and $1$.
    At the first origin, both trajectories decrease linearly, so their correlation is $1$ at every shift.
    At the second origin, one decreases linearly and the other increases linearly, so their correlation is $-1$ at every shift.
    In both cases, shifting gives no improvement over zero, and the tie-breaking rule selects zero.
    Both timing distances are therefore zero.

    Exchange the entries in the third column.
    The trajectories of $f$ stay unchanged, while those of $g$ become $(1,0,1)$ and $(-1,0,-1)$.
    Without a shift, each has correlation $0$ with the corresponding trajectory of $f$.
    A one-step shift, however, can match two decreasing segments or two increasing segments, giving correlation $1$.
    Thus the best shift now has absolute value $1$ at both origins.
    Since $H-2=1$, both timing distances are $1$.
    Their range and diameter change from $0$ to $1$, as do their ambiguity and discrepancy for every $0\leq\tau<1$.
\end{enumerate}

These examples also work for forecast trajectories.
Choose all targets to be zero, so every forecast is the negative of its error.
Negating both trajectories leaves each of the seven distances unchanged.
The changes in the statistics are therefore exactly the same.

For revisions, the only difference is where we place the trajectories in the matrix: along anti-diagonals instead of rows.
Take $T=H+1$ origins, which gives exactly two complete revision trajectories.
Place the first row of each example on the first complete anti-diagonal and the second row on the next, setting the remaining cells to zero.
The $k$-th entries of these two trajectories occupy cells
\begin{equation*}
    (k,H+1-k)\qquad\text{and}\qquad(k+1,H+1-k).
\end{equation*}
They lie in the same column.
We can therefore exchange them by swapping origins $k$ and $k+1$ within that column, for both models.
This reproduces every exchange used above through a row permutation.
With zero targets, the resulting revision trajectories are again the negatives of the error sequences, so all the distance calculations still apply.

Finally, the examples satisfy the Rashomon-set requirement.
For each construction, including any choice of $a$ in the first example, take a model class containing just $f$ and $g$, and choose 
\begin{equation*}
    \eps\geq\big|\Ls(f)-\Ls(g)\big|.
\end{equation*}
The Rashomon set then contains both models.
Each exchange preserves each model's loss, so the same tolerance includes both models after the exchange as well.
By \cref{prop:dichotomy}, it also preserves all pointwise and horizon-wise multiplicity statistics.
The temporal statistics change as shown above, proving the claim for forecast, revision, and error trajectories.
\end{proof}

\clearpage

\section{Temporal Multiplicity Bounds Characterization}
\label{app:bounds}

\begin{table}
    \centering
    \caption{Summary of the bounds proved in \cref{app:bounds-afforded-by-rashomon-membership,app:bounds-afforded-by-pointwise-n-horizonwise-multiplicity}: how Rashomon set membership and additional pointwise or horizon-wise multiplicity constraints bound disagreement between forecast, revision, and error trajectories at each origin.
    The first column is reported for the MAE.\\
    {\footnotesize $^\dagger$ For forecast trajectories, the bounds are attained under conditions that rule out degenerate cases.\\
    $^\S$ For revision trajectories, the bounds hold when the trajectories are traversed in reverse order.}}
    \label{tab:levels}
    \footnotesize
    \setlength{\tabcolsep}{3.5pt}
    \begin{tabular}{lccc}
        \toprule
        \textbf{Temporal Statistic} & $\bm{\gR_\eps}$ \textbf{only} & \makecell{$\bm{+}$ \textbf{pointwise} \\ $\bm{\alpha_{\tau}=0}$} & \makecell{$\bm{+}$ \textbf{horizon-wise} \\ $\bm{\alpha^{(h)}_{\tau_h}=0}$} \\
        \midrule
        Pointwise $|\hat{y}^{f}-\hat{y}^{g}|$ & $2TH(\Lstar+\eps)$ & $\tau$ & $\tau_h$ \\
        \arrayrulecolor{black!50}\cmidrule[0.01em]{1-4}\arrayrulecolor{black}
        Magnitude $\dt_{\mathrm{magnitude}}$ & $2\sqrt{H}\,T(\Lstar+\eps)$ & $\tau$ & $\sqrt{\frac{1}{H}\sum_h\tau_h^2}$ \\
        Increment $\dt_{\mathrm{increment}}^{\S}$ & $\frac{2\sqrt{2}}{\sqrt{H-1}}\,TH(\Lstar+\eps)$ & $2\tau$ & $\sqrt{\frac{1}{H-1}\sum_{h<H}(\tau_h+\tau_{h+1})^2}$ \\
        Volatility $\dt_{\mathrm{volatility}}^{\dagger\S}$ & $\frac{2}{H-1}\,TH(\Lstar+\eps)$ & $2\tau$ & $\frac{1}{H-1}\sum_{h<H}(\tau_h+\tau_{h+1})$ \\
        Direction $\dt_{\mathrm{direction}}^{\dagger}$ & free & free & free \\
        Shape $\dt_{\mathrm{shape}}^{\dagger}$ & free & free & free \\
        Spectral $\dt_{\mathrm{spectral}}^{\dagger}$ & free & free & free \\
        Timing $\dt_{\mathrm{timing}}^{\dagger}$ & free & free & free \\
        \bottomrule
    \end{tabular}
\end{table}

This appendix provides the deferred proofs for the temporal multiplicity bounds, characterizes how these bounds constrain forecast trajectories, and tests how closely trained models approach them. It is organized in three parts.

\Cref{app:bounds-afforded-by-rashomon-membership} studies what Rashomon set membership alone implies for temporal multiplicity.
It first proves \cref{thm:temporal-multiplicity}, which bounds the gap between any two models' forecasts at a single cell.
It then translates this pointwise bound into bounds on the trajectory distances of \cref{sec:temporal-trajectory-distances}: magnitude, increment, and volatility disagreement are bounded, whereas direction, shape, spectral, and timing disagreement can remain maximal.

\Cref{app:bounds-afforded-by-pointwise-n-horizonwise-multiplicity} studies what changes when, in addition, pointwise or horizon-wise ambiguity is required to vanish.
It first proves \cref{thm:structural-freedom}, which tightens the per-cell gap to the chosen tolerances.
It then derives the corresponding trajectory-distance bounds, showing that the stronger control bounds magnitude, increment, and volatility disagreement more tightly, yet still leaves direction, shape, spectral, and timing disagreement unconstrained under the stated assumptions.

\Cref{app:bounds-results} evaluates these bounds empirically.
For each trajectory distance, it measures the largest disagreement among trained models in the empirical Rashomon set and reports what fraction of the corresponding theoretical bound it attains.

\Cref{tab:levels} summarizes the theoretical results of the first two parts, comparing the bounds under Rashomon membership alone, specialized to MAE, with those under additional pointwise or horizon-wise control.

\subsection{The Performance Constraint of the Rashomon Set Membership: Exact Characterization and Bounds for Each Temporal Aspect}
\label{app:bounds-afforded-by-rashomon-membership}

We first restate and prove \cref{thm:temporal-multiplicity}, which bounds pointwise disagreement under the performance constraint defining the Rashomon set.
We then derive bounds and attaining constructions for the seven trajectory distances introduced in \cref{sec:temporal-trajectory-distances}, identifying how much disagreement can remain between models that are members of the Rashomon set. We conclude by explaining how these results extend to error and revision trajectories.

The MAE results are summarized in the Rashomon-membership column of \cref{tab:levels}.

\restateTemporalMultiplicity*
\begin{proof}
Fix an origin $t\in\gT$ and a horizon $h\in\{1,\ldots,H\}$, and write
\[
    B:=TH(\Lstar+\eps),
    \qquad
    u:=\ell^{-1}(B).
\]
Here, $B$ is each model's total loss budget, and $u$ is the largest error magnitude that this budget allows at a single entry.

For any $f\in\gR_\eps$, nonnegativity of the pointwise losses gives
\begin{equation*}
    \ell(e^f_{t,h})
    \leq
    \sum_{t'\in\gT}\sum_{h'=1}^{H}
    \ell(e^f_{t',h'})
    =
    TH\,\Ls(f)
    \leq B.
\end{equation*}
Since $\ell$ depends only on the error magnitude, the definition of $u$ implies
\begin{equation*}
    |e^f_{t,h}|
    \leq u
    =
    \ell^{-1}\big(TH(\Lstar+\eps)\big).
\end{equation*}
Every Rashomon forecast therefore lies within distance $u$ of the same target. Since this target cancels when subtracting two forecasts,
\begin{equation*}
    \hat{y}^f_{t+h}-\hat{y}^g_{t+h}
    =
    e^g_{t,h}-e^f_{t,h},
\end{equation*}
and hence
\[
    \big|\hat{y}^f_{t+h}-\hat{y}^g_{t+h}\big|
    \leq
    |e^g_{t,h}|+|e^f_{t,h}|
    \leq 2u.
\]

To show attainment under shell-richness, we concentrate each model's entire budget at the chosen entry, with opposite error signs.
First,
\begin{equation*}
    \ell(u)=B:
\end{equation*}
otherwise, continuity would allow a slightly larger error with loss still below $B$, contradicting the definition of $u$.

Define
\[
    [\mE^\pm]_{t',h'}
    =
    \begin{cases}
        \pm u, & (t',h')=(t,h),\\
        0, & \text{otherwise}.
    \end{cases}
\]
Because $\ell(0)=0$ and $\ell(-u)=\ell(u)$, both matrices have loss
\[
    \Ls(\mE^\pm)
    =
    \frac{\ell(u)}{TH}
    =
    \Lstar+\eps.
\]
They therefore lie in the admissible loss shell
\[
    \gB_\eps
    :=
    \left\{
        \mE\in\R^{T\times H}:
        \Lstar\leq\Ls(\mE)\leq\Lstar+\eps
    \right\}.
\]
Shell-richness guarantees models $f^+,f^-\in\gR_\eps$ realizing these matrices.
Their forecasts at the chosen entry are $y_{t+h}-u$ and $y_{t+h}+u$, hence
\[
    \max_{f,g\in\gR_\eps}
    \big|\hat{y}^f_{t+h}-\hat{y}^g_{t+h}\big|
    =
    2u
    =
    2\,\ell^{-1}\big(TH(\Lstar+\eps)\big).
\]
Thus, two models with identical loss attain the bound by assigning opposite errors to a single entry while matching the targets everywhere else.
\end{proof}

\cref{thm:temporal-multiplicity} tells us how far apart two models in the Rashomon set can be at a single forecast horizon.
Under shell-richness, its upper bound can be reached.
We now ask how much these models can disagree over an entire trajectory, using the distances in \cref{sec:temporal-trajectory-distances}.

The following corollaries quantify how much forecast trajectories can differ despite identical aggregate loss. Each construction gives a lower bound on the largest disagreement within the Rashomon set; when a matching upper bound is available, it identifies the exact maximum.

Fix an origin $t\in\gT$, and write
\[
    B:=TH(\Lstar+\eps),
    \qquad
    u:=\ell^{-1}(B),
    \qquad
    \varrho:=\ell^{-1}(B/H).
\]
Here, $B$ is each model's total loss budget. An error of magnitude $u$ uses this budget at a single entry, whereas $H$ errors of magnitude $\varrho$ spread it equally across one trajectory:
\[
    \ell(u)=H\ell(\varrho)=B.
\]
These equalities follow from continuity and the definition of $\ell^{-1}$, as in the preceding proof.

Every Rashomon model satisfies the row budget
\[
    \sum_{h=1}^{H}\ell(e^f_{t,h})\leq B.
\]
In all constructions below, errors outside origin $t$ are zero and the row loss equals $B$. The resulting matrices therefore have average loss $\Lstar+\eps$, so shell-richness realizes them as models in $\gR_\eps$. This establishes both membership and equality of aggregate loss without repeating the argument for each construction.

Let $\vb^{(j)}\in\R^H$ denote the vector with a $1$ at entry $j$ and zeros elsewhere. For any pair of models, their forecast gap is
\[
    \vw
    :=
    \widehat{\vy}^{f}_{t}-\widehat{\vy}^{g}_{t}
    =
    \ve^{g}_{t}-\ve^{f}_{t}.
\]
Thus, distances that depend only on this gap, including magnitude and increment, do not depend on the target trajectory.

We begin with magnitude, which measures the root-mean-square gap between forecasts. Its maximum depends on how a model can distribute its loss budget to obtain the largest Euclidean error norm. Two allocations are relevant: concentrating the budget at one horizon, or spreading it equally across all horizons.
Which allocation attains the maximum depends on $\phi(s):=\ell(\sqrt{s})$, the loss associated with a squared error of size $s$:
\begin{itemize}
\item when $\phi$ is concave, concentrating the budget at one horizon attains the maximum;
\item when $\phi$ is convex, spreading it equally attains the maximum. 
\end{itemize}
The following corollary establishes this distinction. For power losses $\ell(e)=|e|^p$, these cases correspond to $0<p\leq2$ and $p\geq2$, respectively; at $p=2$, both allocations attain the same maximum.

\begin{corollary}[Magnitude Multiplicity Bound Under Rashomon Set Membership]
\label{cor:magnitude}
Let $\phi(s):=\ell(\sqrt{s})$ for $s\geq0$. If $\gF$ is shell-rich, then
\[
    \max_{f,g\in\gR_\eps}
    \dt_{\mathrm{magnitude}}
    \big(\widehat{\vy}^{f}_{t},\widehat{\vy}^{g}_{t}\big)
    =
    \begin{cases}
        2u/\sqrt{H}, & \text{if $\phi$ is concave},\\[1ex]
        2\varrho, & \text{if $\phi$ is convex}.
    \end{cases}
\]
For MAE, the maximum is $2\sqrt{H}\,T(\Lstar+\eps)$; for MSE, it is $2\sqrt{T(\Lstar+\eps)}$. Both maxima are attained by models with opposite error rows and identical pointwise absolute errors.
\end{corollary}

\begin{proof}
The triangle inequality gives
\[
    \dt_{\mathrm{magnitude}}
    \big(\widehat{\vy}^{f}_{t},\widehat{\vy}^{g}_{t}\big)
    =
    \frac{\|\ve^g_t-\ve^f_t\|_2}{\sqrt{H}}
    \leq
    \frac{\|\ve^f_t\|_2+\|\ve^g_t\|_2}{\sqrt{H}}.
\]
We therefore need the largest Euclidean norm that one error row can have within the budget. Write $s_h=e_h^2$, so that $\sum_h\phi(s_h)\leq B$.

If $\phi$ is concave, $\phi(0)=0$ implies subadditivity, and hence
\[
    \ell(\|\ve\|_2)
    =
    \phi\Big(\sum_hs_h\Big)
    \leq
    \sum_h\phi(s_h)
    \leq B.
\]
Thus $\|\ve\|_2\leq u$, giving the upper bound $2u/\sqrt{H}$. It is attained by concentrating the budget at one entry:
\[
    \ve^f_t=u\vb^{(j)},
    \qquad
    \ve^g_t=-u\vb^{(j)}.
\]
Their gap has Euclidean norm $2u$.

If $\phi$ is convex, Jensen's inequality instead gives
\[
    \ell\Big(\frac{\|\ve\|_2}{\sqrt{H}}\Big)
    =
    \phi\Big(\frac1H\sum_hs_h\Big)
    \leq
    \frac1H\sum_h\phi(s_h)
    \leq B/H.
\]
Thus $\|\ve\|_2\leq\sqrt{H}\,\varrho$, giving the upper bound $2\varrho$. It is attained by spreading the budget equally:
\[
    \ve^f_t=\varrho\vone,
    \qquad
    \ve^g_t=-\varrho\vone.
\]
Their gap has magnitude $2\varrho$ at every entry.

Finally, MAE has $\phi(s)=\sqrt{s}$ and $u=B$, whereas MSE has $\phi(s)=s$ and $\varrho=\sqrt{B/H}$, giving the stated expressions.
\end{proof}

We next consider increment distance, which measures disagreement in how forecasts change from one horizon to the next. Here, the arrangement of the errors matters as well as their size: an error concentrated at a horizon other than the first or last creates two abrupt changes, while alternating errors create changes throughout the trajectory. The following corollary uses these patterns to obtain lower bounds for general losses and identifies the exact maxima for MAE and MSE.

\begin{corollary}[Increment Multiplicity Bound Under Rashomon Set Membership]
\label{cor:increment}
Let $\phi(s):=\ell(\sqrt{s})$. If $\gF$ is shell-rich and $H\geq3$, then
\[
    \max_{f,g\in\gR_\eps}
    \dt_{\mathrm{increment}}
    \big(\widehat{\vy}^{f}_{t},\widehat{\vy}^{g}_{t}\big)
    \geq
    \begin{cases}
        2\sqrt{2}\,u/\sqrt{H-1},
        & \text{if $\phi$ is concave},\\[1ex]
        4\varrho,
        & \text{if $\phi$ is convex}.
    \end{cases}
\]
For MAE, the exact maximum is
\[
    \frac{2\sqrt{2}\,TH(\Lstar+\eps)}{\sqrt{H-1}}.
\]
For MSE, it is
\[
    4\cos\Big(\frac{\pi}{2H}\Big)
    \sqrt{\frac{TH(\Lstar+\eps)}{H-1}}.
\]
The displayed lower bounds and exact maxima are attained by pairs with opposite error rows and identical pointwise absolute errors.
\end{corollary}

\begin{proof}
Increment distance measures how quickly the forecast gap changes:
\[
    \dt_{\mathrm{increment}}
    \big(\widehat{\vy}^{f}_{t},\widehat{\vy}^{g}_{t}\big)
    =
    \frac{\|\Delta(\ve^g_t-\ve^f_t)\|_2}{\sqrt{H-1}}.
\]

For the first lower bound, take opposite errors concentrated at an interior entry:
\[
    \ve^f_t=u\vb^{(j)}=-\ve^g_t,
    \qquad 1<j<H.
\]
The gap changes by $2u$ when entering this entry and by $2u$ when leaving it. Its two nonzero increments therefore give
\[
    \dt_{\mathrm{increment}}
    =
    \frac{\sqrt{(2u)^2+(2u)^2}}{\sqrt{H-1}}
    =
    \frac{2\sqrt{2}\,u}{\sqrt{H-1}}.
\]

For the second lower bound, take opposite alternating errors,
\[
    e^f_{t,h}=(-1)^h\varrho=-e^g_{t,h}.
\]
Every increment of their gap has magnitude $4\varrho$, so $\dt_{\mathrm{increment}}=4\varrho$. Both constructions use the full budget.

To establish the exact MAE maximum, observe that $\|\ve\|_1\leq B=u$ and
\[
    \|\Delta\ve\|_2
    =
    \Big\|\sum_h e_h\Delta\vb^{(h)}\Big\|_2
    \leq
    \sum_h|e_h|\,\|\Delta\vb^{(h)}\|_2
    \leq
    \sqrt{2}\,\|\ve\|_1.
\]
The last inequality holds because a single entry affects at most two increments. Applying this bound to both models gives $2\sqrt{2}\,u/\sqrt{H-1}$, attained by the interior-entry construction.

For MSE, the budget is $\|\ve\|_2\leq\sqrt{B}$. The largest singular value of the increment operator is
\[
    \sigma_{\max}(\Delta)
    =
    2\cos\Big(\frac{\pi}{2H}\Big),
\]
as follows from the eigenvalues of the path Laplacian $\Delta^\top\Delta$ \cite[Section~1.4.4]{Brouwer2012}. Consequently,
\[
    \dt_{\mathrm{increment}}
    \leq
    \frac{2\sigma_{\max}(\Delta)\sqrt{B}}{\sqrt{H-1}}.
\]
Equality is obtained by placing opposite error rows along a unit right singular vector for $\sigma_{\max}(\Delta)$:
\[
    \ve^f_t=\sqrt{B}\,\vv=-\ve^g_t,
    \qquad
    v_h=
    \sqrt{\frac2H}\,
    (-1)^{h-1}
    \sin\Big(\frac{(2h-1)\pi}{2H}\Big).
\]
Both rows have squared norm $B$, and their gap is amplified by $\Delta$ by exactly $\sigma_{\max}(\Delta)$. This proves the stated MSE maximum.
\end{proof}

Magnitude and increment distances depend only on the difference between the two forecast trajectories, from which the common target cancels. 
The remaining distances compare properties of the trajectories that are \textit{not} determined by this difference alone. Their values can therefore depend on the target trajectory as well as on the errors.

For volatility and direction, this dependence is visible in the forecast increments:
\[
    \Delta\hat y^f_h
    =
    \Delta y_h-\Delta e^f_{t,h},
    \qquad
    \Delta y_h:=y_{t+h+1}-y_{t+h}.
\]
The target's own changes can reinforce or offset the changes in the errors, affecting both the magnitude and the sign of the predicted increments.

To isolate the disagreement produced by the errors, the volatility, shape, spectral, and timing results below assume a constant target $\vy_t=c\vone$. Under this assumption, each forecast trajectory is simply its error trajectory reflected and shifted by $c$. The direction result allows a nonconstant target, provided its increments are small enough for the constructed errors to reverse the predicted directions.

Volatility distance compares how much the two forecasts fluctuate, measured by their mean absolute increments. To produce a large disagreement, we pair a flat forecast with one that fluctuates strongly. The following corollary bounds this disagreement and constructs pairs that attain the stated bounds.

\begin{corollary}[Volatility Multiplicity  Bound Under Rashomon Set Membership]
\label{cor:volatility}
If $\gF$ is shell-rich, $\vy_t=c\vone$, and $H\geq3$, then
\[
    \max_{f,g\in\gR_\eps}
    \dt_{\mathrm{volatility}}
    \big(\widehat{\vy}^{f}_{t},\widehat{\vy}^{g}_{t}\big)
    \begin{cases}
        =2u/(H-1), & \text{if $\ell$ is concave on $[0,\infty)$},\\[1ex]
        \geq2\varrho, & \text{if $\ell$ is convex on $[0,\infty)$}.
    \end{cases}
\]
For MAE, the maximum is $2TH(\Lstar+\eps)/(H-1)$. For MSE, it is
\[
    \frac{\sqrt{H(4H-6)}}{H-1}
    \sqrt{T(\Lstar+\eps)}.
\]
Each displayed bound is attained by a pair with identical aggregate loss, one of which forecasts a flat trajectory.
\end{corollary}
\begin{proof}
For a constant target, forecast increments equal error increments up to sign. Thus
\[
    \dt_{\mathrm{volatility}}
    \big(\widehat{\vy}^{f}_{t},\widehat{\vy}^{g}_{t}\big)
    =
    \frac{
        \big|\|\Delta\ve^f_t\|_1-\|\Delta\ve^g_t\|_1\big|
    }{H-1}.
\]
Because both norms are nonnegative, their difference is at most the larger one. We therefore seek the largest possible total variation $\|\Delta\ve\|_1$ of a single error row.

If $\ell$ is concave, subadditivity and the row budget give
\[
    \ell(\|\ve\|_1)
    \leq
    \sum_h\ell(|e_h|)
    \leq B,
    \qquad
    \|\ve\|_1\leq u.
\]
Since each error contributes to at most two increments,
\[
    \|\Delta\ve\|_1
    \leq 2\|\ve\|_1
    \leq2u.
\]
This gives the upper bound $2u/(H-1)$.

To attain it, take the flat error row $\ve^f_t=\varrho\vone$ and the concentrated row $\ve^g_t=u\vb^{(j)}$, with $1<j<H$. Both use the full budget. The first forecast is flat, whereas the second has two increments of magnitude $u$, giving distance $2u/(H-1)$. This also proves the MAE formula.

For the convex-case lower bound, keep $\ve^f_t=\varrho\vone$ and take $e^g_{t,h}=(-1)^h\varrho$. Every increment of the second forecast has magnitude $2\varrho$, giving distance $2\varrho$.

For the exact MSE maximum, retain the endpoint contributions in the variation bound:
\[
    \|\Delta\ve\|_1
    \leq
    |e_1|+2\sum_{h=2}^{H-1}|e_h|+|e_H|
    \leq
    \sqrt{4H-6}\,\|\ve\|_2
    \leq
    \sqrt{(4H-6)B}.
\]
The middle step is Cauchy--Schwarz with weights $\va=(1,2,\ldots,2,1)$. All inequalities are attained by an alternating row whose magnitudes follow these weights:
\[
    e^g_{t,h}
    =
    (-1)^h a_h\sqrt{\frac{B}{4H-6}}.
\]
Indeed, $\|\ve^g_t\|_2^2=B$, and opposite signs at consecutive entries make every increment equal the sum of its endpoint magnitudes. Pairing this row with $\ve^f_t=\varrho\vone$ yields
\[
    \dt_{\mathrm{volatility}}
    =
    \frac{\sqrt{(4H-6)B}}{H-1},
\]
which is the stated MSE maximum.
\end{proof}

Direction distance asks whether the two forecasts rise or fall together at each step, regardless of the size of those changes. The following corollary shows that models with identical pointwise absolute errors can predict opposite directions at every step, provided the target's own changes are small enough for the errors to reverse the predicted directions.

\begin{corollary}[Direction Multiplicity Bound Under Rashomon Set Membership]
\label{cor:direction}
If $\gF$ is shell-rich, $H\geq2$, and
\[
    |\Delta y_h|<2\varrho
    \qquad\text{for all }h=1,\ldots,H-1,
\]
then
\[
    \max_{f,g\in\gR_\eps}
    \dt_{\mathrm{direction}}
    \big(\widehat{\vy}^{f}_{t},\widehat{\vy}^{g}_{t}\big)
    =1.
\]
The condition holds for a constant target whenever $\varrho>0$. The threshold $2\varrho$ is $2T(\Lstar+\eps)$ for MAE and $2\sqrt{T(\Lstar+\eps)}$ for MSE.

The maximum is attained by models with opposite error rows and identical pointwise absolute errors.
\end{corollary}

\begin{proof}
Take opposite alternating errors,
\[
    e^f_{t,h}=(-1)^h\varrho=-e^g_{t,h}.
\]
Their forecast increments are
\[
    \Delta\hat y^f_h=\Delta y_h+2(-1)^h\varrho,
    \qquad
    \Delta\hat y^g_h=\Delta y_h-2(-1)^h\varrho.
\]
The error contribution dominates the target increment because $|\Delta y_h|<2\varrho$. The two predicted increments therefore have opposite signs at every step. Their direction distance is $1$, the largest possible value.
\end{proof}

Shape distance compares the overall patterns of the trajectories after removing their levels and scales. For a constant target, opposite error trajectories produce forecasts with exactly opposite shapes. The following corollary shows that this yields the largest possible shape disagreement while preserving absolute errors at every horizon.

\begin{corollary}[Shape Multiplicity Bound Under Rashomon Set Membership]
\label{cor:shape}
If $\gF$ is shell-rich, $\vy_t=c\vone$, $H\geq2$, and $B>0$, then
\[
    \max_{f,g\in\gR_\eps}
    \dt_{\mathrm{shape}}
    \big(\widehat{\vy}^{f}_{t},\widehat{\vy}^{g}_{t}\big)
    =1.
\]
Every admissible pair with opposite, nonconstant error rows attains this maximum.
\end{corollary}

\begin{proof}
With a constant target, centering a forecast removes the target and leaves the negative centered error. Thus, opposite nonconstant error rows produce opposite nonzero centered forecasts, whose Pearson correlation is $-1$.

Such a pair exists: take
\[
    \ve^f_t=u\vb^{(j)}=-\ve^g_t.
\]
Since $B>0$, we have $u>0$, and these rows are nonconstant for $H\geq2$. Both use the full budget, and
\[
    \dt_{\mathrm{shape}}
    =
    \frac{1-\rho}{2}
    =
    \frac{1-(-1)}{2}
    =1.
\]
This attains the largest possible shape distance.
\end{proof}

Spectral distance compares how the variation in each trajectory is distributed across frequencies. The following construction assigns different frequencies to the two forecasts and adjusts their amplitudes to give both models the same loss. For a constant target, their spectral disagreement reaches its largest possible value.

\begin{corollary}[Spectral Multiplicity Bound Under Rashomon Set Membership]
\label{cor:spectral}
If $\gF$ is shell-rich, $\vy_t=c\vone$, $H\geq4$, and $B>0$, then
\[
    \max_{f,g\in\gR_\eps}
    \dt_{\mathrm{spectral}}
    \big(\widehat{\vy}^{f}_{t},\widehat{\vy}^{g}_{t}\big)
    =1.
\]
The maximum is attained by models with identical aggregate loss whose error rows contain different single frequencies.
\end{corollary}

\begin{proof}
For two nonconstant trajectories, their standardized magnitude spectra $\va$ and $\vb$ are nonnegative unit vectors. Hence
\[
    \dt_{\mathrm{spectral}}^2
    =
    \frac12\|\va-\vb\|_2^2
    =
    1-\langle\va,\vb\rangle
    \leq1,
\]
with equality when the spectra have disjoint supports. The flat-trajectory convention also preserves this upper bound.

To attain it, choose two distinct frequencies from $\{1,\ldots,\lfloor H/2\rfloor\}$, which is possible because $H\geq4$. For each chosen frequency $k$, define
\[
    c^{(k)}_h=\cos(2\pi kh/H),
    \qquad
    \ve^{f_k}_t=a_k\vc^{(k)},
\]
where $a_k>0$ is chosen so that
\[
    \sum_{h=1}^{H}\ell(a_kc^{(k)}_h)=B.
\]
Such an amplitude exists because this sum is continuous, starts at zero, and grows without bound.

Each cosine row has mean zero. Since the target is constant, centering and normalizing its forecast gives $-\vc^{(k)}/\|\vc^{(k)}\|_2$. Its Fourier transform is supported on $\{k,H-k\}$. Distinct frequencies in the chosen range have disjoint supports, so their magnitude spectra are orthogonal and their spectral distance is $1$.
\end{proof}

Timing distance measures the shift needed to best align the two forecast trajectories. The following corollary constructs equally accurate forecasts from shifted portions of the same cosine wave. Their best alignment requires the largest admissible lag, showing that equal aggregate loss can coexist with maximal timing disagreement.

\begin{corollary}[Timing Multiplicity Bound Under Rashomon Set Membership]
\label{cor:timing}
If $\gF$ is shell-rich, $\vy_t=c\vone$, $H\geq3$, and $B>0$, then
\[
    \max_{f,g\in\gR_\eps}
    \dt_{\mathrm{timing}}
    \big(\widehat{\vy}^{f}_{t},\widehat{\vy}^{g}_{t}\big)
    =1.
\]
The maximum is attained by models with identical aggregate loss whose forecast trajectories are samples of the same cosine wave shifted by $H-2$ steps.
\end{corollary}

\begin{proof}
The timing distance is at most $1$ because admissible lags satisfy $|\tau|\leq H-2$. We construct a pair whose best alignment occurs at $\tau=H-2$.

Let $\theta:=\pi/H$ and define
\[
    u_h=\cos(\theta h),
    \qquad
    v_h=\cos\big(\theta(h+H-2)\big)
        =-\cos\big(\theta(h-2)\big).
\]
These rows have the same absolute entries up to permutation:
\[
    v_1=u_{H-1},
    \qquad
    v_2=u_H,
    \qquad
    v_h=-u_{h-2}\quad(h\geq3).
\]
Choose $a>0$ so that the error row $-a\vu$ uses budget $B$; continuity and unboundedness guarantee its existence. Evenness of $\ell$ implies that $-a\vv$ has the same loss. Shell-richness therefore provides forecasts
\[
    \widehat{\vy}^f_t=c\vone+a\vu,
    \qquad
    \widehat{\vy}^g_t=c\vone+a\vv.
\]
Their lagged correlations equal those of $\vu$ and $\vv$, since Pearson correlation is unchanged by an offset and positive scale.

At lag $H-2$, the overlapping segments are
\[
    (u_{H-1},u_H)=(v_1,v_2).
\]
They are nonconstant, so their correlation is $1$.

It remains to exclude a better or equally good alignment at any other lag. The row $\vu$ is strictly decreasing, whereas
\[
    v_1=v_3>v_2,
\]
and $\vv$ is strictly increasing from its second entry onward.

For $0\leq\tau\leq H-3$, the overlap includes the first three entries of $\vv$. Their decrease followed by an increase prevents them from being a positive affine transformation of the strictly decreasing segment of $\vu$. Their correlation is therefore strictly less than $1$.

For $\tau<0$, the overlap pairs a strictly decreasing segment of $\vu$ with a strictly increasing segment of $\vv$, so its correlation is negative.

Thus the unique maximizing lag is $\tau^\star=H-2$, and
\[
    \dt_{\mathrm{timing}}
    =
    \frac{|\tau^\star|}{H-2}
    =1.
\]
\end{proof}

Together, Corollaries~\ref{cor:magnitude}--\ref{cor:timing} show that models with identical aggregate loss can disagree about forecast magnitude, increments, volatility, direction, shape, frequency content, and timing.

We close with three remarks explaining how these results relate to revision and error trajectories.

\begin{remark}[Error Multiplicity]
\label{rem:error-forecast}
Forecast and error trajectories have the same cross-model gaps up to sign:
\[
    \ve^f_t-\ve^g_t
    =
    -\big(\widehat{\vy}^f_t-\widehat{\vy}^g_t\big).
\]
Consequently, symmetric distances that depend only on these gaps, including magnitude and increment, give identical multiplicity values for forecasts and errors. Pointwise and horizon-wise multiplicity statistics also coincide.

For the other distances, the two trajectories can differ: forecast trajectories describe disagreement about the future, whereas error trajectories describe disagreement in the patterns of the models' mistakes.
\end{remark}

\begin{remark}[Error Multiplicity for Each Temporal Aspect]
\label{rem:error-axes}
The bounds and constructions above also apply to error trajectories, without requiring a constant target or small target increments. To see why, recall that for a constant target,
\[
    \widehat{\vy}^f_t=c\vone-\ve^f_t.
\]
Thus, forecasts are obtained by reflecting and shifting the error trajectories. Applying this transformation to both trajectories leaves every distance considered here unchanged.

The constructions therefore already produce the stated disagreement between error trajectories. Under shell-richness, the same error patterns can be realized for any target whenever they lie in the admissible loss shell. Only the restrictions on the target disappear; the assumptions on the loss, model class, horizon length, and budget are retained.
\end{remark}

\begin{remark}[Revision Multiplicity]
\label{rem:revision}
The bounds and constructions above also apply to any complete revision trajectory. A revision trajectory collects $H$ forecasts of the same target, issued at successive origins. We can therefore place the constructed errors along its anti-diagonal instead of a row without changing their total loss; shell-richness guarantees models realizing these patterns.

Because every entry forecasts the same target, the constant-target assumption holds automatically. The direction condition also holds whenever $\varrho>0$. All remaining assumptions on the loss, model class, horizon length, and budget remain in force.
\end{remark}

\clearpage
\subsection{The Horizon-Wise Multiplicity Constraint: Exact Characterization and Bounds for Each Temporal Aspect}
\label{app:bounds-afforded-by-pointwise-n-horizonwise-multiplicity}

We first restate and prove \cref{thm:structural-freedom}.
We then show how its pointwise bounds control magnitude, increment, and volatility distances.
In contrast, direction, shape, spectral behavior, and timing disagreement can remain maximal even under arbitrarily small positive tolerances, subject to the stated conditions on the target trajectory.
These results are summarized in \cref{tab:levels} at the start of this appendix.

We now ask what controlling disagreement at every horizon guarantees about entire trajectories. Vanishing horizon-wise ambiguity means that
\[
    \alpha^{(h)}_{\tau_h}(\{f,g\})=0
    \quad\Longleftrightarrow\quad
    |\hat y^f_{t+h}-\hat y^g_{t+h}|\leq\tau_h
    \quad\text{for every }t\in\gT.
\]
Thus, $\tau_h$ limits the forecast gap at horizon $h$. Taking $\tau_h\equiv\tau$ recovers the pointwise constraint. The resulting trajectory bounds are summarized in \cref{tab:levels}.

Write $B:=TH(\Lstar+\eps)$ for the total loss budget. For the attainment results below, assume $T\geq2$, so that errors can also be assigned outside the origin being studied.

A small distinction matters when constructing models: shell-richness guarantees realization only when their average losses lie in $[\Lstar,\Lstar+\eps]$. A construction below the upper limit may still fall below $\Lstar$. When two constructed error rows have the same total loss $S\leq B$, we can resolve this by assigning an identical additional error to both models at another origin, contributing
\[
    \max\{TH\Lstar-S,0\}
\]
to each total loss. Such an error exists by continuity and unboundedness of $\ell$. Both models then lie in the admissible loss shell, while their forecast gaps remain unchanged.

We first restate and prove the multiplicity bound under horizon-wise control.

\restateStructuralFreedom*

\begin{proof}
The upper bound follows directly from the definition of vanishing horizon-wise ambiguity.

To attain all the tolerances at a chosen origin $t$, take
\[
    e^f_{t,h}=-\tau_h/2,
    \qquad
    e^g_{t,h}=\tau_h/2.
\]
Their forecast gap is exactly $\tau_h$ at every horizon, and each row has total loss
\[
    S=\sum_{h=1}^{H}\ell(\tau_h/2)\leq B.
\]
Adding identical errors at another origin if necessary, as described above, places both matrices in the admissible shell without introducing further disagreement. Shell-richness then provides the required models.
\end{proof}

Magnitude, increment, and volatility distances admit upper bounds directly in terms of the tolerances. Magnitude depends on the sizes of the forecast gaps; increment also depends on how those gaps change between successive horizons. Volatility compares the amount of fluctuation in the two forecasts, whose difference is likewise controlled by changes in their gap.

Throughout the following results, maxima run over pairs $f,g\in\gR_\eps$ satisfying $|\hat y^f_{t'+h}-\hat y^g_{t'+h}|\leq\tau_h$ at every origin $t'$ and horizon $h$.

\begin{corollary}[Magnitude, Increment, and Volatility Multiplicity Bounds Under Horizon-Wise Control]
\label{cor:horizon-wise-magnitude}
Fix an origin $t$, let $H\geq2$, and assume $\tau_h\geq0$. If $\gF$ is shell-rich and
\[
    \sum_{h=1}^{H}\ell(\tau_h/2)\leq B,
\]
then
\begin{align*}
    \max\,
    \dt_{\mathrm{magnitude}}
    \big(\widehat{\vy}^f_t,\widehat{\vy}^g_t\big)
    &=
    \sqrt{\frac1H\sum_{h=1}^{H}\tau_h^2},
    \\
    \max\,
    \dt_{\mathrm{increment}}
    \big(\widehat{\vy}^f_t,\widehat{\vy}^g_t\big)
    &=
    \sqrt{\frac1{H-1}\sum_{h=1}^{H-1}(\tau_h+\tau_{h+1})^2}.
\end{align*}
If additionally $\vy_t=c\vone$ and $\sum_h\ell(\tau_h)\leq B$, then
\[
    \max\,
    \dt_{\mathrm{volatility}}
    \big(\widehat{\vy}^f_t,\widehat{\vy}^g_t\big)
    =
    \frac1{H-1}\sum_{h=1}^{H-1}(\tau_h+\tau_{h+1}).
\]
The magnitude and increment maxima are attained by models with identical pointwise absolute errors. The volatility maximum is attained by models with identical aggregate loss, one of which forecasts a flat trajectory at origin $t$.
\end{corollary}

\begin{proof}
Write $\vw=\widehat{\vy}^f_t-\widehat{\vy}^g_t$ for the forecast gap. Since $|w_h|\leq\tau_h$, we have
\[
    \dt_{\mathrm{magnitude}}
    =
    \frac{\|\vw\|_2}{\sqrt{H}}
    \leq
    \sqrt{\frac1H\sum_h\tau_h^2}.
\]
Each change in the gap is bounded by the two adjacent tolerances:
\[
    |\Delta w_h|
    \leq |w_{h+1}|+|w_h|
    \leq\tau_h+\tau_{h+1}.
\]
This gives the increment bound,
\[
    \dt_{\mathrm{increment}}
    =
    \frac{\|\Delta\vw\|_2}{\sqrt{H-1}}
    \leq
    \sqrt{\frac1{H-1}\sum_{h=1}^{H-1}(\tau_h+\tau_{h+1})^2}.
\]
For volatility, the reverse triangle inequality gives
\[
    \dt_{\mathrm{volatility}}
    =
    \frac{\big|\|\Delta\widehat{\vy}^f_t\|_1-\|\Delta\widehat{\vy}^g_t\|_1\big|}{H-1}
    \leq
    \frac{\|\Delta\vw\|_1}{H-1}
    \leq
    \frac1{H-1}\sum_{h=1}^{H-1}(\tau_h+\tau_{h+1}).
\]

To attain the magnitude bound, use the opposite errors $e^f_{t,h}=-\tau_h/2=-e^g_{t,h}$ from the preceding proof. Their gap has magnitude $\tau_h$ at every horizon.

To attain the increment bound, alternate the signs:
\[
    e^f_{t,h}=(-1)^h\tau_h/2=-e^g_{t,h}.
\]
Consecutive gaps now have opposite signs, so $|\Delta w_h|=\tau_h+\tau_{h+1}$. Both constructions have equal row losses within the budget and can be placed in the admissible shell by adding identical errors at another origin.

For volatility, we need one flat forecast and one alternating forecast. Let $z_h=(-1)^h\tau_h$, choose another origin $s\neq t$, and assign
\[
\begin{aligned}
    \ve^f_t&=a\vone+\vz,
    &\qquad \ve^g_t&=a\vone,\\
    \ve^f_s&=a\vone,
    &\qquad \ve^g_s&=a\vone+\vz,
\end{aligned}
\]
with errors zero elsewhere. Swapping the rows at origin $s$ gives the models identical total loss:
\[
    C(a)=\sum_{h=1}^{H}\ell(a+z_h)+H\ell(a).
\]
At $a=0$, this equals $\sum_h\ell(\tau_h)\leq B$. If it is below $TH\Lstar$, continuity and unboundedness allow us to choose $a>0$ so that $C(a)=TH\Lstar$. Otherwise, take $a=0$. Both matrices then lie in the admissible shell.

The gaps at origins $t$ and $s$ have magnitude $\tau_h$, so all horizon-wise constraints hold. At origin $t$, the forecast of $g$ is flat, while every increment of the forecast of $f$ has magnitude $\tau_h+\tau_{h+1}$. Their volatility distance therefore attains the upper bound.
\end{proof}

The preceding bounds shrink as the tolerances decrease. For a constant target, direction, shape, spectral, and timing disagreement can instead remain maximal even when all tolerances are arbitrarily small but positive. The reason is that their constructions can be reduced in amplitude while preserving the signs or normalized patterns that these distances compare. The direction result also allows a nonconstant target, provided its changes are sufficiently small.

\begin{corollary}[Direction, Shape, Spectral and Timing Multiplicity Structural Freedom Under Horizon-Wise Control]
\label{cor:horizon-wise-freedom}
Fix an origin $t$, let $H\geq2$, and assume $\gF$ is shell-rich, $B>0$, and $\tau_h>0$ for every horizon. Define
\[
    \varrho:=\ell^{-1}(B/H),
    \qquad
    \kappa_h:=\min\{\tau_h,2\varrho\},
    \qquad
    \kappa:=\min_h\kappa_h.
\]
Here, $\kappa_h$ limits the constructed forecast gap using both the horizon-wise tolerance and the available loss budget. Then
\begin{align*}
    \max\,
    \dt_{\mathrm{direction}}
    \big(\widehat{\vy}^f_t,\widehat{\vy}^g_t\big)
    &=1
    &&\text{if }|\Delta y_h|<\frac{\kappa_h+\kappa_{h+1}}2
      \text{ for every }h,
    \\
    \max\,
    \dt_{\mathrm{shape}}
    \big(\widehat{\vy}^f_t,\widehat{\vy}^g_t\big)
    &=1
    &&\text{if }\vy_t=c\vone,
    \\
    \max\,
    \dt_{\mathrm{spectral}}
    \big(\widehat{\vy}^f_t,\widehat{\vy}^g_t\big)
    &=1
    &&\text{if }\vy_t=c\vone\text{ and }H\geq4,
    \\
    \max\,
    \dt_{\mathrm{timing}}
    \big(\widehat{\vy}^f_t,\widehat{\vy}^g_t\big)
    &=1
    &&\text{if }\vy_t=c\vone\text{ and }H\geq3.
\end{align*}
All maxima are attained by models with identical aggregate loss. The direction and shape constructions also preserve pointwise absolute errors.
\end{corollary}

\begin{proof}
All four distances are at most $1$, so it suffices to construct admissible pairs attaining this value.

In each construction, the error magnitude at horizon $h$ is at most $\kappa_h/2$ for both models. Their gap is therefore at most $\kappa_h\leq\tau_h$, and each row satisfies the loss budget:
\[
    \sum_{h=1}^{H}\ell(e_{t,h})
    \leq
    H\ell(\varrho)
    =B.
\]
We choose the two rows to have equal loss, allowing identical errors at another origin to place both models in the admissible shell if necessary.

\emph{Direction.}
Take opposite alternating errors,
\[
    e^f_{t,h}=(-1)^h\kappa_h/2=-e^g_{t,h}.
\]
Their forecast increments are
\[
    \Delta\hat y^f_h
    =
    \Delta y_h+(-1)^h\frac{\kappa_h+\kappa_{h+1}}2,
    \qquad
    \Delta\hat y^g_h
    =
    \Delta y_h-(-1)^h\frac{\kappa_h+\kappa_{h+1}}2.
\]
Under the stated condition, the error contribution exceeds the magnitude of the target increment. The forecasts therefore move in opposite directions at every step, giving distance $1$.

\emph{Shape.}
For a constant target, the same construction gives opposite, nonconstant centered forecast trajectories. Their correlation is $-1$, so their shape distance is $1$.

\emph{Spectral.}
Choose two distinct frequencies $j,k\in\{1,\ldots,\lfloor H/2\rfloor\}$ and let
\[
    c^{(r)}_h=\cos(2\pi rh/H),
    \qquad
    \ve^f_t=a_j\vc^{(j)},
    \qquad
    \ve^g_t=a_k\vc^{(k)}.
\]
Choose positive amplitudes $a_j,a_k\leq\kappa/2$ giving both rows total loss $\ell(\kappa/2)$. This is possible by continuity because each cosine row contains an entry equal to $1$. If $\ell(\kappa/2)=0$, simply take both amplitudes equal to $\kappa/2$.

As in \cref{cor:spectral}, centering removes the constant target and normalization removes the amplitudes. The resulting magnitude spectra have disjoint supports, giving spectral distance $1$.

\emph{Timing.}
Use the shifted cosine patterns from \cref{cor:timing}:
\[
    u_h=\cos(\pi h/H),
    \qquad
    v_h=\cos\big(\pi(h+H-2)/H\big),
\]
and take
\[
    \ve^f_t=-(\kappa/2)\vu,
    \qquad
    \ve^g_t=-(\kappa/2)\vv.
\]
Their entries have magnitude at most $\kappa/2$, and their absolute entries are permutations of one another, so the rows have equal loss. For a constant target, the resulting forecasts have the same lagged correlations as $\vu$ and $\vv$. Their best alignment occurs at lag $H-2$, as established in \cref{cor:timing}, giving timing distance $1$.
\end{proof}

Thus, horizon-wise agreement controls the sizes of forecast gaps and the resulting magnitude, increment, and volatility disagreement. Under the stated target conditions, it can nevertheless leave direction, shape, spectral, and timing disagreement maximal.

\begin{remark}[Revision and Error Trajectories]
\label{rem:horizon-wise-revision}
The results also apply to complete revision trajectories. As forecasts of a fixed target are updated, their horizons decrease from $H$ to $1$, so the corresponding tolerances are $(\tau_H,\ldots,\tau_1)$. Reversing this order leaves the bounds unchanged, and the constructions transfer to the relevant anti-diagonal without changing their loss. Since every entry forecasts the same target, the constant-target condition holds automatically.

For error trajectories, the horizon-wise constraints are unchanged because
\[
    |e^f_{t,h}-e^g_{t,h}|
    =
    |\hat y^f_{t+h}-\hat y^g_{t+h}|.
\]
The same error patterns therefore attain the stated bounds directly, without assumptions on the target trajectory. The constant-target and small-target-increment conditions are needed only to obtain the corresponding disagreement between forecasts.

In both cases, the remaining assumptions on shell-richness, horizon length, tolerances, and the loss budget still apply.
\end{remark}

\clearpage

\subsection{How Close Does the Empirical Rashomon Set Come to the Theoretical Bounds?}
\label{app:bounds-results}

We test whether trained models in the empirical Rashomon set $\egR$ attain the theoretical temporal disagreement allowed by Rashomon membership alone and by vanishing pointwise or horizon-wise ambiguity (\cref{tab:levels}).
Since these bounds concern individual origins, we compute the largest disagreement at any origin, rather than the origin-averaged multiplicity statistics we measured in \cref{sec:experiments}.
The question is whether the extreme disagreement permitted by the theory occurs among near-optimal trained models, and which distances remain far from their theoretical maxima.

\paragraph{Rashomon Membership Alone.}
For MASE, the loss budget is $TH(\Lstar+\eps)$, where $T$ is the number of validation origins, $H$ the horizon, $\Lstar=\Ls(f^\star)$ is the best validation MASE, and $\eps=5 \% \cdot \Lstar$ (\cref{app:rashomon-set}).
For each trajectory distance $\dt$, we define the \emph{worst-origin range} and its \emph{utilization} of the membership bound $U_{\dt}$ as:
\begin{equation}
    \bar{\nu}_{\dt}
    =
    \max_{t\in\gT}\,\max_{f,g\in\egR}\dt\big(\vv^{f}_{t},\vv^{g}_{t}\big),
    \qquad
    \frac{\bar{\nu}_{\dt}}{U_{\dt}}.
    \label{eq:bounds-worst-origin}
\end{equation}
The corresponding bounds (\cref{thm:temporal-multiplicity,app:bounds-afforded-by-rashomon-membership}) are:
\begin{equation*}
    U_{\dt}=
    \begin{cases}
        2T\sqrt{H}\,(\Lstar+\eps), & \dt=\dt_{\mathrm{magnitude}},\\[6pt]
        \dfrac{2\sqrt{2}\,TH}{\sqrt{H-1}}\,(\Lstar+\eps), & \dt=\dt_{\mathrm{increment}},\\[6pt]
        \dfrac{2TH}{H-1}\,(\Lstar+\eps), & \dt=\dt_{\mathrm{volatility}},\\[6pt]
        1, & \dt\in\{\dt_{\mathrm{direction}},\dt_{\mathrm{shape}},\dt_{\mathrm{spectral}},\dt_{\mathrm{timing}}\}.
    \end{cases}
\end{equation*}
For the four scale-free distances, the bound $1$ is reported as $100\%$.
The worst-origin range selects the most separated pair at the most adverse origin and is at least the origin-averaged trajectory range $\nu_{\dt}$ of \cref{tab:temporal-main}.
Utilization expresses this observed distance as a fraction of the theoretical bound.
These maxima measure whether large disagreement occurs, rather than how typical it is across pairs or origins.

\Cref{tab:bounds-rashomon} shows that the scale-free bounds are readily attained: direction, shape, and timing reach $100\%$ on every dataset, while spectral distance reaches $99.7\%$. 
Scale-free distances compare signs or normalized patterns, so attaining their range does not require large errors: reducing the amplitude of a disagreement need not reduce its distance. Indeed, this confirms empirically what the theoretical findings of \cref{cor:direction,cor:shape,cor:timing,cor:spectral} predicted.

Magnitude, increment, and volatility use at most $6.8\%$ of their bounds. 
This is expected: the bounds for these distance functions grow with $T$; however, the attaining constructions would concentrate an entire model's error budget in a single forecast cell, an extreme pattern unlikely to result from training using gradient descent.
As such, a finite set of trained models need not contain such constructions, even when the model class can represent them. 

Although the observed distances reach only a small fraction of their theoretical bounds, they remain large in absolute terms: between $16.9$ and $21.4$ multiples of the naive error scale on average, with most datasets attaining values well above $3$ multiples of the naive error scale. Rashomon membership alone therefore leaves extremely large temporal multiplicity at some individual forecast origins, even without the extreme error patterns needed to attain the bounds.

\begin{table}[h]
    \centering
    \caption{Worst-origin range and utilization of the Rashomon-membership bound (\eqref{eq:bounds-worst-origin}).
    Magnitude, increment, and volatility are in naive units, with utilization in parentheses; scale-free distances are percentages of their bound $1$.
    Values are maxima over trajectories, averaged over targets where applicable; the last row averages over datasets.}
    \label{tab:bounds-rashomon}
    \footnotesize
    \setlength{\tabcolsep}{3pt}
    \resizebox{0.7\textwidth}{!}{\begin{tabular}{llllllll}
\toprule
Dataset & \multicolumn{7}{l}{$\gR_\eps$ only} \\
\cmidrule(lr){2-8}
 & $\dt_{\mathrm{magnitude}}$ & $\dt_{\mathrm{increment}}$ & $\dt_{\mathrm{volatility}}$ & $\dt_{\mathrm{direction}}$ & $\dt_{\mathrm{shape}}$ & $\dt_{\mathrm{spectral}}$ & $\dt_{\mathrm{timing}}$ \\
\midrule
ATM 1D & 6.54 (0.0\%) & 6.20 (0.0\%) & 3.52 (0.0\%) & 100\% & 100\% & 97.6\% & 100\% \\
ATM 1W & 6.66 (0.1\%) & 5.47 (0.1\%) & 3.62 (0.1\%) & 100\% & 100\% & 99.4\% & 100\% \\
Bitcoin & 19.7 (0.3\%) & 11.1 (0.1\%) & 7.39 (0.3\%) & 100\% & 100\% & 95.4\% & 100\% \\
Electricity D. A & 7.86 (0.0\%) & 1.66 (0.0\%) & 0.75 (0.0\%) & 90.2\% & 98.8\% & 95.0\% & 100\% \\
Electricity D. Z & 3.73 (0.0\%) & 0.72 (0.0\%) & 0.31 (0.0\%) & 83.2\% & 98.4\% & 93.2\% & 100\% \\
Electricity Price & 2.74 (0.0\%) & 0.85 (0.0\%) & 0.50 (0.0\%) & 95.7\% & 96.5\% & 92.4\% & 100\% \\
ETT 1H & 3.29 (0.0\%) & 1.21 (0.0\%) & 0.79 (0.0\%) & 100\% & 99.1\% & 96.5\% & 100\% \\
Exchange Rate & 162 (1.5\%) & 236 (1.3\%) & 163 (2.5\%) & 100\% & 100\% & 99.5\% & 100\% \\
FRED-MD & 0.36 (0.8\%) & 0.22 (0.3\%) & 0.17 (0.4\%) & 100\% & 100\% & 0.0\% & 100\% \\
Gasoline & 12.2 (1.7\%) & 7.88 (0.7\%) & 4.90 (1.0\%) & 100\% & 100\% & 99.7\% & 100\% \\
Weather & 10.6 (0.0\%) & 2.26 (0.0\%) & 1.31 (0.0\%) & 96.5\% & 99.8\% & 94.4\% & 100\% \\
\midrule
Mean & 21.4 (0.4\%) & 24.8 (0.2\%) & 16.9 (0.4\%) & 96.9\% & 99.3\% & 87.6\% & 100\% \\
\bottomrule
\end{tabular}
}
\end{table}

\paragraph{Vanishing Pointwise and Horizon-Wise Ambiguity.}
We evaluate each pair $f,g\in\egR$ at the smallest tolerances for which its ambiguity vanishes: the largest absolute forecast gap over all test cells (pointwise), or separately at each horizon (horizon-wise):
\begin{equation}
    \tau^{fg}
    =
    \max_{t\in\gT}\max_{1\leq h\leq H}\big|\nv{\hat{y}}^{f}_{t+h}-\nv{\hat{y}}^{g}_{t+h}\big|,
    \qquad
    \tau^{fg}_{h}
    =
    \max_{t\in\gT}\big|\nv{\hat{y}}^{f}_{t+h}-\nv{\hat{y}}^{g}_{t+h}\big|.
    \label{eq:bounds-pair-tolerances}
\end{equation}
By construction, $\alpha_{\tau^{fg}}(\{f,g\})=0$ and $\alpha^{(h)}_{\tau^{fg}_h}(\{f,g\})=0$ for every $h$.
Writing $\bm{\tau}^{fg}=(\tau^{fg}_1,\ldots,\tau^{fg}_H)$, let $U_{\dt}(\tau)$ and $U_{\dt}(\bm{\tau})$ denote the pointwise and horizon-wise bounds in \cref{tab:levels} (\cref{thm:structural-freedom,cor:horizon-wise-magnitude,app:bounds-afforded-by-pointwise-n-horizonwise-multiplicity}), namely:
\begin{equation*}
    U_{\dt}(\tau)=
    \begin{cases}
        \tau, & \dt=\dt_{\mathrm{magnitude}},\\[6pt]
        2\tau, & \dt=\dt_{\mathrm{increment}},\\[6pt]
        2\tau, & \dt=\dt_{\mathrm{volatility}},\\[6pt]
        1, & \dt\in\{\dt_{\mathrm{direction}},\dt_{\mathrm{shape}},\dt_{\mathrm{spectral}},\dt_{\mathrm{timing}}\},
    \end{cases}
\end{equation*}
\begin{equation*}
    U_{\dt}(\bm{\tau})=
    \begin{cases}
        \sqrt{\dfrac{1}{H}\displaystyle\sum_{h=1}^{H}\tau_h^2}, & \dt=\dt_{\mathrm{magnitude}},\\[6pt]
        \sqrt{\dfrac{1}{H-1}\displaystyle\sum_{h=1}^{H-1}(\tau_h+\tau_{h+1})^2}, & \dt=\dt_{\mathrm{increment}},\\[6pt]
        \dfrac{1}{H-1}\displaystyle\sum_{h=1}^{H-1}(\tau_h+\tau_{h+1}), & \dt=\dt_{\mathrm{volatility}},\\[6pt]
        1, & \dt\in\{\dt_{\mathrm{direction}},\dt_{\mathrm{shape}},\dt_{\mathrm{spectral}},\dt_{\mathrm{timing}}\}.
    \end{cases}
\end{equation*}
This evaluates the tightness of the bounds at the tolerances each pair actually satisfies; it does not establish attainment at a single, prespecified tolerance shared by all pairs.
The pair's utilization is its largest distance at any origin divided by the corresponding bound:
\begin{equation}
    u^{fg}_{\dt}
    =
    \frac{\max_{t\in\gT}\dt\big(\vv^{f}_{t},\vv^{g}_{t}\big)}{U_{\dt}(\tau^{fg})},
    \qquad
    u^{fg}_{\dt,\bm{\tau}}
    =
    \frac{\max_{t\in\gT}\dt\big(\vv^{f}_{t},\vv^{g}_{t}\big)}{U_{\dt}(\bm{\tau}^{fg})}.
    \label{eq:bounds-pair-utilization}
\end{equation}
For the scale-free distances, whose bound is $1$, we restrict the maximum to origins with $|\Delta\nv{y}_h|<(\tau^{fg}_h+\tau^{fg}_{h+1})/2$ at every step: the target-flatness criterion for direction in \cref{cor:horizon-wise-freedom}.
For shape, spectral, and timing distances, the corollary proves attainment for constant targets; the filter alone does not guarantee attainment.
For revision trajectories, the tolerance vector follows the reversed horizon order.

\Cref{tab:bounds-enforcement} reports the maximum utilization over pairs.
\Cref{fig:bounds-enforcement} plots all pairs and each dataset's running maximum against normalized tolerance.

Three observations are worth noting.

\paragraph{Most Distances Attain or Approach Their Bounds.}
Magnitude reaches $99.8\%$ under pointwise control and $100\%$ under horizon-wise control; direction, shape, and timing reach $100\%$, and spectral distance reaches $99.7\%$.
Increment and volatility remain further below: their dataset-averaged maxima are $65.4\%$ and $53.5\%$ pointwise, rising to $82.5\%$ and $67.9\%$ horizon-wise.

The gap is smallest for Exchange Rate, where increment and volatility reach $99.8\%$ and $95.5\%$ of their horizon-wise bounds, showing that trained models \emph{can} approach the theoretical maxima. A possible explanation is the limited predictability of exchange rates: if their dynamics are approximately those of a random walk, the loss may provide less guidance toward a common forecast structure, leaving more varied models with similar performance. More generally, as with the Rashomon-membership bounds, the remaining gaps may reflect the finite collection of models explored and the implicit restrictions imposed by training, rather than limitations on the disagreement that the full model class can express.

\paragraph{Substantial Gaps Remain for Increment and Volatility on Most Datasets.}
The difference follows from the configurations needed for equality.
Magnitude only requires the absolute gap to reach its tolerance at every horizon of one origin; a persistent offset can therefore approach the bound.
Increment requires the gap to alternate between $+\tau_h$ and $-\tau_{h+1}$ at every consecutive horizon: the forecasts must cross at every step with maximal amplitude.
Volatility additionally requires these changes to produce the largest possible difference in total variation; one attaining construction pairs a flat forecast with a full-amplitude zigzag.
These extreme configurations are unlikely for trained forecasters.
Smooth, persistent gaps can be large in magnitude while changing little across horizons, leaving increment and volatility well below their bounds.

\paragraph{Horizon-Wise Control Gives Higher Utilization.} Horizon-wise control gives higher utilization because each horizon receives its own tolerance, avoiding the slack introduced by applying the largest gap anywhere to every horizon.

\begin{table}[h]
    \centering
    \caption{Maximum pointwise / horizon-wise utilization over pairs (\eqref{eq:bounds-pair-utilization}), each evaluated at its smallest admissible tolerances (\eqref{eq:bounds-pair-tolerances}).
    Scale-free distances use only origins satisfying the target-flatness filter.
    Aggregation over trajectories, targets, and datasets follows \cref{tab:bounds-rashomon}.}
    \label{tab:bounds-enforcement}
    \footnotesize
    \setlength{\tabcolsep}{3pt}
    \resizebox{\textwidth}{!}{\begin{tabular}{llllllll}
\toprule
Dataset & $\dt_{\mathrm{magnitude}}$ & $\dt_{\mathrm{increment}}$ & $\dt_{\mathrm{volatility}}$ & $\dt_{\mathrm{direction}}$ & $\dt_{\mathrm{shape}}$ & $\dt_{\mathrm{spectral}}$ & $\dt_{\mathrm{timing}}$ \\
\midrule
ATM 1D & 90.4\% / 100\% & 49.8\% / 69.2\% & 41.5\% / 47.7\% & 100\% / 100\% & 100\% / 100\% & 97.6\% / 97.6\% & 100\% / 100\% \\
ATM 1W & 99.6\% / 100\% & 62.7\% / 74.8\% & 54.4\% / 65.1\% & 100\% / 100\% & 100\% / 100\% & 99.4\% / 99.4\% & 100\% / 100\% \\
Bitcoin & 99.3\% / 100\% & 53.4\% / 71.0\% & 32.8\% / 44.8\% & 100\% / 100\% & 100\% / 100\% & 95.4\% / 95.4\% & 100\% / 100\% \\
Electricity D. A & 85.5\% / 99.1\% & 18.1\% / 23.5\% & 7.7\% / 9.9\% & 90.2\% / 90.2\% & 98.8\% / 98.8\% & 95.0\% / 95.0\% & 100\% / 100\% \\
Electricity D. Z & 81.5\% / 95.4\% & 11.5\% / 14.0\% & 4.8\% / 6.5\% & 83.2\% / 83.2\% & 98.4\% / 98.4\% & 93.2\% / 93.2\% & 100\% / 100\% \\
Electricity Price & 84.7\% / 99.1\% & 26.3\% / 31.9\% & 11.8\% / 14.5\% & 95.7\% / 95.7\% & 96.5\% / 96.5\% & 92.4\% / 92.4\% & 100\% / 100\% \\
ETT 1H & 94.1\% / 100\% & 33.7\% / 53.6\% & 23.1\% / 31.7\% & 100\% / 100\% & 99.1\% / 99.1\% & 96.5\% / 96.5\% & 100\% / 100\% \\
Exchange Rate & 99.7\% / 100\% & 76.7\% / 97.7\% & 67.7\% / 92.4\% & 100\% / 100\% & 100\% / 100\% & 99.3\% / 99.3\% & 100\% / 100\% \\
FRED-MD & 95.0\% / 100\% & 67.0\% / 80.5\% & 49.6\% / 62.4\% & 100\% / 100\% & 100\% / 100\% & 0.0\% / 0.0\% & 100\% / 100\% \\
Gasoline & 99.8\% / 100\% & 83.3\% / 97.7\% & 69.3\% / 86.3\% & 100\% / 100\% & 100\% / 100\% & 99.7\% / 99.7\% & 100\% / 100\% \\
Weather & 95.4\% / 99.4\% & 23.4\% / 32.4\% & 9.4\% / 14.5\% & 96.5\% / 96.5\% & 99.8\% / 99.8\% & 94.4\% / 94.4\% & 100\% / 100\% \\
\midrule
Mean & 93.2\% / 99.4\% & 46.0\% / 58.8\% & 33.8\% / 43.3\% & 96.9\% / 96.9\% & 99.3\% / 99.3\% & 87.5\% / 87.5\% & 100\% / 100\% \\
\bottomrule
\end{tabular}
}
\end{table}

\begin{figure}[h]
    \centering
    \includegraphics[width=\textwidth]{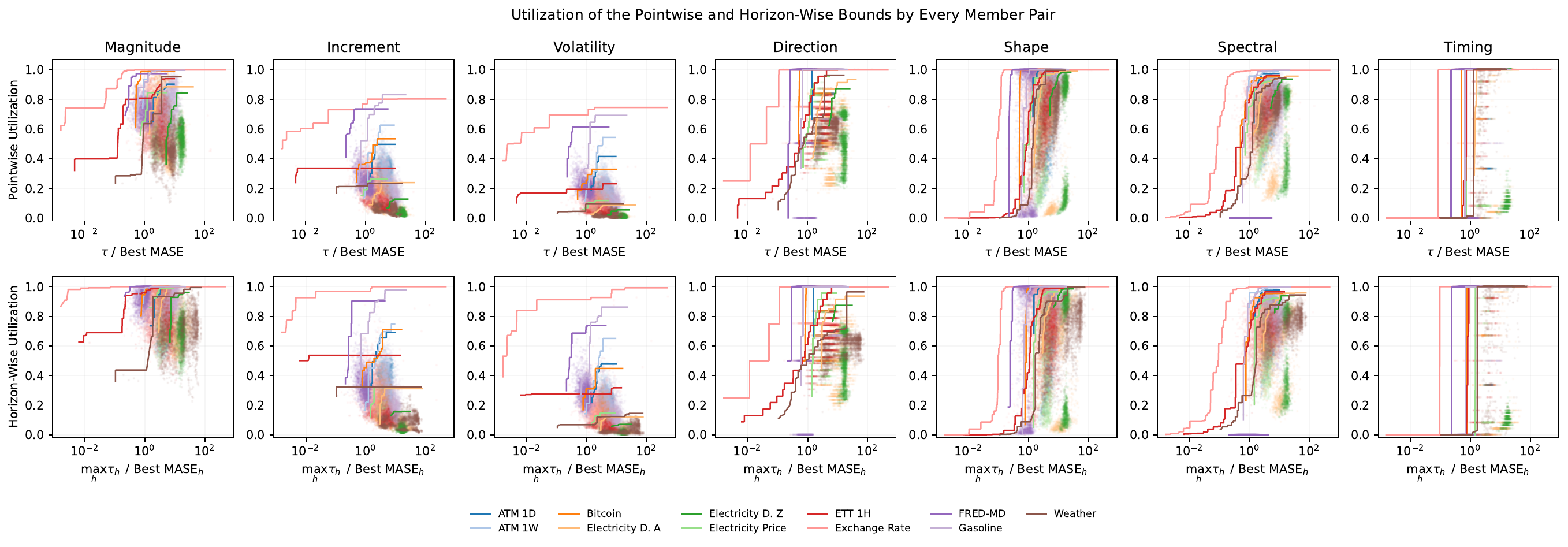}
    \caption{Pairwise bound utilization against normalized tolerance: $\tau^{fg}/\Lstar$ for pointwise control (top), and $\max_h(\tau^{fg}_h/\Ls^{\star}_h)$ for horizon-wise control (bottom).
    Each marker represents one pair; lines show each dataset's running maximum.}
    \label{fig:bounds-enforcement}
\end{figure}

\end{document}